\documentclass{article}
\usepackage{fars_preprint}

\usepackage{amsmath,amsfonts,bm}

\def\eqref#1{equation~\ref{#1}}

\def\1{\bm{1}}

\DeclareMathAlphabet{\mathsfit}{\encodingdefault}{\sfdefault}{m}{sl}
\SetMathAlphabet{\mathsfit}{bold}{\encodingdefault}{\sfdefault}{bx}{n}

\usepackage{hyperref}
\hypersetup{hypertexnames=false}
\usepackage{url}
\usepackage{graphicx}
\usepackage{amsmath,amssymb,amsthm}
\usepackage{booktabs}
\usepackage{multirow}
\usepackage{alphalph}
\usepackage{placeins,needspace}
\usepackage{array}
\usepackage{xcolor}
\usepackage{float}
\newtheorem{lemma}{Lemma}
\usepackage{xeCJK}
\usepackage{listings}
\definecolor{codebg}{HTML}{F4F6F7}
\definecolor{codeframe}{HTML}{859DAA}
\definecolor{codecomment}{HTML}{6A6A6A}
\definecolor{codekw}{HTML}{395C72}
\definecolor{codestr}{HTML}{C44E52}
\lstdefinestyle{purplebox}{
  backgroundcolor=\color{codebg},
  frame=single, rulecolor=\color{codeframe},
  basicstyle=\ttfamily\fontsize{7}{8.5}\selectfont,
  keywordstyle=\color{codekw}\bfseries,
  commentstyle=\color{codecomment}\itshape,
  stringstyle=\color{codestr},
  language=Python, showstringspaces=false,
  breaklines=true, columns=flexible, numbers=none,
  xleftmargin=2pt, xrightmargin=2pt,
  framexleftmargin=2pt, framexrightmargin=2pt,
  aboveskip=0.5em, belowskip=0.5em,
}

\newcommand{\GSMN}{100}
\newcommand{\GSMR}{3}
\newcommand{\GSMBASEQ}{90.0}
\newcommand{\GSMFARSQ}{70.0}
\newcommand{\GSMRANDQ}{90.0}
\newcommand{\GSMDROPQ}{20.0}
\newcommand{\GSMVSRANDQ}{20.0}
\newcommand{\GSMWPQ}{4.0\!\times\!10^{-5}}
\newcommand{\GSMBASEL}{86.0}
\newcommand{\GSMFARSL}{85.0}
\newcommand{\GSMRANDL}{85.5}

\newcommand{\GSMWPL}{0.44}

\title{Concept Subspaces Compute Beyond the Logit Lens:\\
A Weights-Only Test for Locating Representations\\
Upstream of Readout}
\author{Aojie Yuan\textsuperscript{1}\quad Zhiyuan Julian Su\textsuperscript{2}\quad Haiyue Zhang\textsuperscript{1}\quad Zijian Su\textsuperscript{3}\\[5pt]
{\small\textsuperscript{1}University of Southern California\quad \textsuperscript{2}Duke University}\\[2pt]
{\small\textsuperscript{3}University of Michigan}\\[4pt]
{\footnotesize\textsuperscript{1}\href{mailto:aojieyua@usc.edu}{\texttt{aojieyua@usc.edu}}\quad
\textsuperscript{2}\href{mailto:zhiyuan.j.su@duke.edu}{\texttt{zhiyuan.j.su@duke.edu}}\quad
\textsuperscript{1}\href{mailto:haiyuez@usc.edu}{\texttt{haiyuez@usc.edu}}\quad
\textsuperscript{3}\href{mailto:simoon@umich.edu}{\texttt{simoon@umich.edu}}}}

\date{September 29, 2026}
\hypersetup{pdftitle={Concept Subspaces Compute Beyond the Logit Lens: A Weights-Only Test for Locating Representations Upstream of Readout},pdfauthor={Aojie Yuan, Zhiyuan Julian Su, Haiyue Zhang, Zijian Su},pdfsubject={Concept subspaces and controlled readout geometry},colorlinks=true,linkcolor=black,citecolor=black,urlcolor=accent}

\begin{document}
\maketitle


\begin{abstract}
A concept subspace's effect on model behavior does not establish how it relates to the output readout. We introduce a two-sided geometric diagnostic that measures an extracted subspace's overlap with the dominant right-singular directions of the unembedding matrix, evaluated against output-oriented positive controls. Given an extracted basis, the raw diagnostic requires only model weights. Our testbed is the Format-Agnostic Reasoning Subspace (FARS), a ten-dimensional basis extracted from eighteen reasoning concepts expressed in six surface forms. Across nine rank-matched estimators and twenty-six models, four activation-derived concept estimators carry only 0.38--0.80\% mean energy in the top-ten readout span. Final-layer PCA carries 3.56\%, exceeding FARS in 25 of 26 models. A same-layer next-token control, evaluated using a fitted linear translator for depth matching, carries approximately thirteen times more energy than FARS, with separation in all 25 tested models. Re-extracting FARS on ten disjoint concepts yields 62--100\% cross-format retrieval across twenty-four generative models, demonstrating transfer of the extraction procedure rather than a fixed basis. A complementary four-model, three-seed intervention study finds model-dependent source-directed effects that remain well below full-vector replacement. Together, the geometry and intervention controls distinguish concept structure from dominant readout directions while limiting claims of causal sufficiency.
\par\medskip\centering\href{https://github.com/Justin0504/LLM-representation-agnostic}{\textsf{Code}}
\end{abstract}

\begin{figure}[H]
\centering
\includegraphics[width=0.96\textwidth]{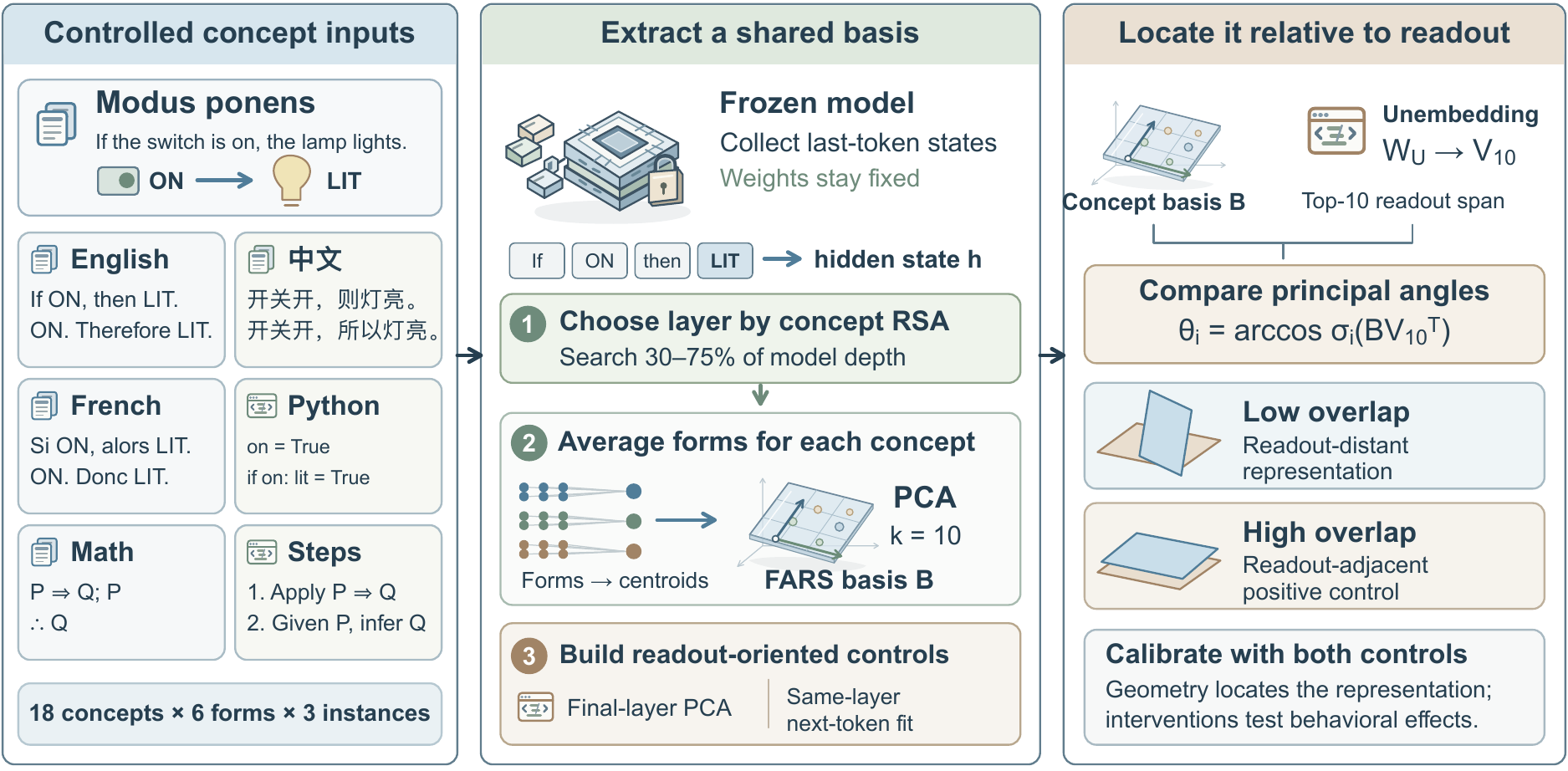}
\caption{\textbf{From controlled concepts to readout geometry.} Six schematic forms define a shared concept basis. Principal angles compare it with the dominant readout span, calibrated by two controls. Illustrations are schematic; quantitative results follow.}
\label{fig:pipeline}
\end{figure}
\clearpage

\section{Introduction}
\label{sec:intro}

Language models express reasoning concepts in prose, code, and mathematical notation, but this flexibility need not imply shared internal directions. Concept-direction research \citep{park2024linear,arditi2024refusal,turner2023activation} motivates two questions: does concept identity persist across surface forms, and how strongly does its recovered subspace overlap the model's output readout?

We extract a \textbf{Format-Agnostic Reasoning Subspace} (FARS) using concept-centroid PCA over eighteen reasoning concepts in six surface forms (Fig.~\ref{fig:pipeline}). Averaging within concept suppresses the form confound without gradient updates or supervision beyond concept labels. We use rank $k=10$, an empirical operating point where marginal gains plateau (\S\ref{sec:universality}); it is not an intrinsic-dimensionality claim, since eighteen mean-centred centroids span at most seventeen dimensions.

An ablation drop alone cannot tell whether a subspace computes an answer or carries an already-computed identity toward logits. Our main contribution is a \emph{readout-orthogonality test}: compare a probed basis with the top-$k$ right singular vectors of the model's unembedding matrix $W_U$. Given the probed basis, this geometric comparison requires weights rather than new model runs. Positive and depth-matched controls establish that the test can return both low and high readout loading.

\paragraph{Contributions.}
\begin{itemize}
\item \textbf{A calibrated geometric diagnostic.} Given an extracted basis, the raw test uses model weights to quantify its overlap with dominant readout directions. Across $26$ models, final-layer PCA exceeds FARS in $25/26$ cases; a same-layer next-token control does so in $25/25$ depth-matched cases.
\item \textbf{Controlled evidence across estimators and inventories.} Four concept estimators among nine rank-matched constructions carry $0.38$--$0.80\%$ mean readout energy. Re-extraction on ten disjoint concepts yields $62$--$100\%$ cross-format retrieval on $24$ generative models. This tests transfer of the extraction procedure, not transport of a fixed basis.
\item \textbf{An intervention-based boundary on interpretation.} A new four-model, three-seed comparison separates FARS effects from random, norm-matched random, full-vector, and no-intervention controls. Source-directed effects are model dependent and substantially weaker than full replacement. Geometric separation and intervention sensitivity therefore do not establish causal sufficiency.
\end{itemize}

\section{Related Work}
\label{sec:related}

\paragraph{Concept directions and readout.}
The Linear Representation Hypothesis treats concepts as activation-space directions \citep{park2024linear,elhage2022superposition}. TCAV estimates labelled concept vectors \citep{kim2018tcav}; linear probes recover truth directions \citep{marks2024geometry}; refusal and activation-steering studies intervene on such directions \citep{arditi2024refusal,turner2023activation}. Distributed Alignment Search learns subspace interventions \citep{geiger2024das}, LEACE erases linearly encoded concepts \citep{belrose2023leace}, and sparse autoencoders extract overcomplete feature dictionaries \citep{bricken2023sae,templeton2024scaling}. We use centroid PCA as a lightweight estimator and evaluate its readout geometry alongside alternative constructions (App.~\ref{app:method_comparison}).

Tuned lenses study intermediate readout \citep{belrose2023tunedlens}. Closest to our question, \citet{nadaf2026steerable} find that function vectors \citep{todd2024function} can steer where logit-lens decoders cannot read them, while \citet{billa2026predicting} use logit-lens accessibility to predict steering success. Patchscopes decodes hidden states through later model layers \citep{ghandeharioun2024patchscopes}. Our diagnostic adds positive controls and a depth-matched linear reference. Low overlap with the specified span does not exclude decoding through lower singular directions, a learned decoder, or subsequent nonlinear processing.

\paragraph{Cross-model / cross-architecture representation.}
The Platonic Representation Hypothesis \citep{huh2024platonic} conjectures cross-architecture convergence to a shared representational space; \citet{kornblith2019cka} and \citet{morcos2018cca} give correlation-based measures. Our X-FARS analysis tests this question using aligned concept centroids and their relations on a fixed inventory (App.~\ref{app:universal_frame}).

\paragraph{Cross-lingual and cross-format transfer.}
Studies of multilingual representations address several distinct questions: language-specific neurons \citep{tang2024language,zhao2024multilingual}, shared circuits and concept latents \citep{ferrando2024similarity,dumas2025separating,chen2025abstract}, alignment across languages \citep{conneau2020emerging,liu2025middle}, latent language \citep{wendler2024llamas}, and transfer of knowledge-free reasoning \citep{hu2025cross}. These findings motivate a controlled comparison across surface forms but do not, collectively, establish invariance across prose, code, and mathematical notation. OctoPack studies code instruction tuning \citep{muennighoff2024octopack}; our benchmarks instead hold concept identity fixed while varying these forms. We use them to test readout geometry and extraction transfer, without claiming an exhaustive priority result over all cross-format datasets.

\paragraph{Causal interpretability.}
Activation patching \citep{vig2020investigating,meng2022locating,wang2022interpretability} and its follow-ups \citep{anthropic2025biology,conmy2023automated} localise circuits by intervention, and causal-abstraction accounts \citep{geiger2021causal} formalise what such interventions license. We use subspace-level patching (\S\ref{sec:extraction}) with ablation and amplification as a bidirectional causal test; the interpretability illusions demonstrated by \citet{makelov2024subspace} caution against equating an intervention effect with recovery of the naturally used mechanism. Our geometric control addresses readout adjacency; it does not by itself resolve that broader causal-identification problem. Relatedly, \citet{yuan2026hiddenerror} report predictive hidden-state error signals without successful correction under their tested interventions. This motivates distinguishing decodability from intervention efficacy; our experiments address concept subspaces rather than error detectors.

\section{Experimental Setup}\label{sec:expsetup}
\label{sec:method}

\paragraph{TriForm Benchmark.}
A controlled stimulus set where reasoning content is held exactly constant while surface form varies along three axes: \textbf{linguistic} (English, Chinese, French prose), \textbf{symbolic} (Python code, mathematical notation), and \textbf{structural} (step-by-step). $18$ concepts spanning arithmetic, logic, relational, causal, and spatial domains $\times$ $3$ instances $\times$ $6$ forms $=$ $324$ stimuli, programmatically generated; full inventory in App.~\ref{app:benchmark}, example in App.~\ref{app:stimuli}.

\paragraph{Models.} Twenty-six LLMs from four architecture families plus reasoning-tuned variants ($0.4$B--$70$B): sixteen dense decoders (GPT-2~XL, Qwen2.5-3B-Inst, Qwen2.5-7B, Qwen3-4B/8B, Phi-3.5-mini-Inst, Phi-4, Mistral-7B base/-Inst, Llama-3.1-8B/70B-Inst, OLMo-2-7B-Inst, Granite-3.3-8B-Inst, Falcon3-7B-Inst, Yi-1.5-9B-Chat, SmolLM3-3B), three MoE (Mixtral-8x7B-Inst, DeepSeek-V2-Lite-Chat, gpt-oss-20B), two state-space (Mamba-2.8B, Falcon-Mamba-7B), four reasoning-tuned (R1-Distill-Llama-8B, R1-Distill-Qwen-7B/14B, QwQ-32B), and one encoder MLM (DeBERTa-v3-large). RoBERTa-large is excluded from the readout analyses because no layer yields a concept-dominant subspace (form-RSA exceeds concept-RSA throughout), so the test's prerequisite is unmet. The X-FARS canonical frame (App.~\ref{app:universal_frame}) includes twenty-seven extracted models including RoBERTa. Its original alignment analyses use a fifteen-model pool; expanded algebra summaries use twenty-seven models, as identified in each table and caption.

\paragraph{Methods.} Last-token hidden states at every layer ($\mathbf{A}\in\mathbb{R}^{324 \times L \times D}$). \textbf{Permutation RSA} \citep{kriegeskorte2008rsa,diedrichsen2017rsa} between the empirical cosine RDM and binary ideal RDMs ($0$ for a same-concept pair, $1$ otherwise; likewise for form; $1{,}000$ perms, FDR-corrected). \textbf{Cross-form probing} (ridge $\alpha{=}0.1$ on source form, eval on target). \textbf{FARS extraction}: average each concept's activations over forms and instances to get $18$ centroids in $\mathbb{R}^D$, PCA $\to$ top-$k$ FARS basis $\mathbf{B}_\ell \in \mathbb{R}^{k \times D}$. All estimators in \S\ref{sec:readout_control} are rank-matched at $k{=}10$ so that Grassmann distances are comparable. This matters where the natural rank is lower: \emph{FormPCA} is the $5$ directions of the six mean-centred form centroids, completed to rank $10$ by the leading directions of the form-projected-out residual, and \emph{ProbeW} is the top-$10$ PCA of the $18$ one-vs-rest probe weight vectors. \textbf{Causal interventions}: subspace patching $\mathbf{h}_\text{patched} = \mathbf{h}_\text{tgt} + \mathbf{B}^\top \mathbf{B}(\mathbf{h}_\text{src}-\mathbf{h}_\text{tgt})$ \eqref{eq:subspace_patch}, ablation $(\mathbf{I}-\mathbf{B}^\top\mathbf{B})\mathbf{h}$, and amplification $\mathbf{h} + c\,\mathbf{B}^\top\mathbf{B}\mathbf{h}$ (so $c{=}0$ is the identity and $c^\star$ denotes the gain used), compared against random orthonormal, Full-PCA, and full replacement; concept-block bootstrap ($5{,}000$).
\begin{equation}
\mathbf{h}_{\text{patched}} = \mathbf{h}_{\text{tgt}} + \mathbf{B}^\top \mathbf{B}\, (\mathbf{h}_{\text{src}} - \mathbf{h}_{\text{tgt}})
\label{eq:subspace_patch}
\end{equation}

\section{Extracting FARS}
\label{sec:extraction}

Given a set of $C$ concepts each rendered in $F$ surface forms with $I$ instances per pair ($N=CFI$ stimuli), FARS is the top-$k$ principal subspace of the mean-centred concept-centroid matrix at the concept-RSA peak within the $30$--$75\%$ depth band, as in the readout analyses. The reasoning-position analysis instead selects by centroid variance (App.~\ref{app:reasoning_full}); this is a distinct extraction setting. A \emph{shuffled-centroid null} trails FARS by $17$--$29\sigma$ in concept-RSA and \emph{multi-class LDA} matches it within $\Delta{\leq}0.02$, so the subspace is concept-carrying rather than a variance artefact (App.~\ref{app:extraction_baselines}).

\begin{table}[t]
\centering
\footnotesize
\setlength{\tabcolsep}{5pt}
\begin{tabular}{@{}lccc|cc@{}}
\toprule
\textbf{Model} & \multicolumn{3}{c|}{\textbf{Argmax preserved (\%)$\uparrow$}} & \multicolumn{2}{c}{\textbf{Behavioural (drop pp)$\downarrow$}} \\
$L{=}$ best FARS layer & FARS-10 & PCA-10 & Full & GSM8K & MATH \\
\midrule
\raisebox{-0.18em}{\includegraphics[height=0.95em]{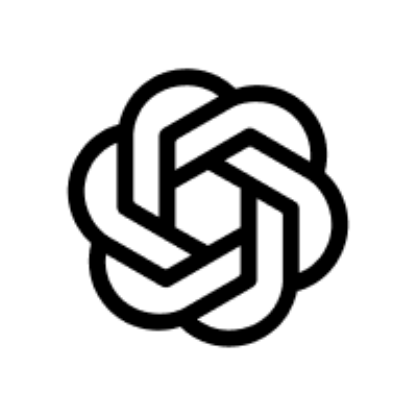}}~GPT-2 XL           & \textbf{93.7} & 87.7 & 81.0 & --- & --- \\
\raisebox{-0.18em}{\includegraphics[height=0.95em]{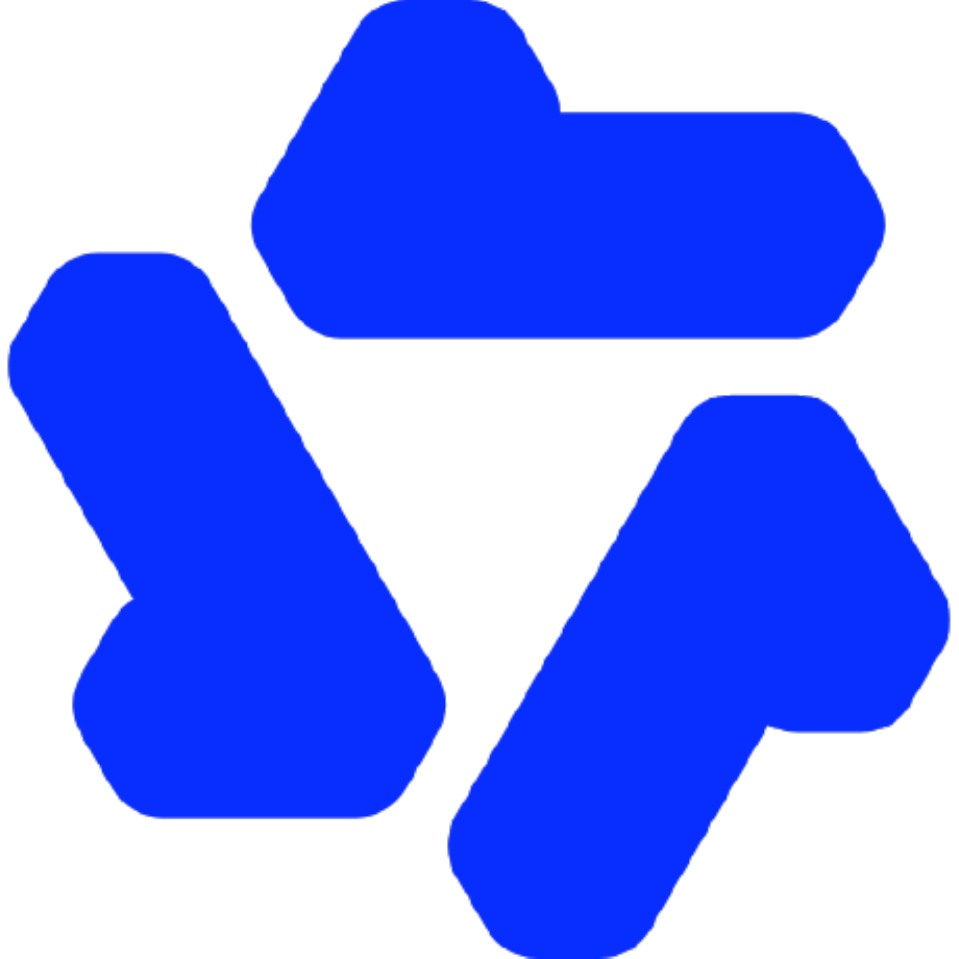}}~Qwen-2.5-3B          & \textbf{89.7} & 37.3 & 22.0 & --- & --- \\
\raisebox{-0.18em}{\includegraphics[height=0.95em]{figures/logos/qwen.pdf}}~Qwen-2.5-7B-Inst.\   & \textbf{94.7} & 59.7 & 26.3 & $\mathbf{-20.0}$\,pp & $\mathbf{-18.4}$\,pp \\
\raisebox{-0.18em}{\includegraphics[height=0.95em]{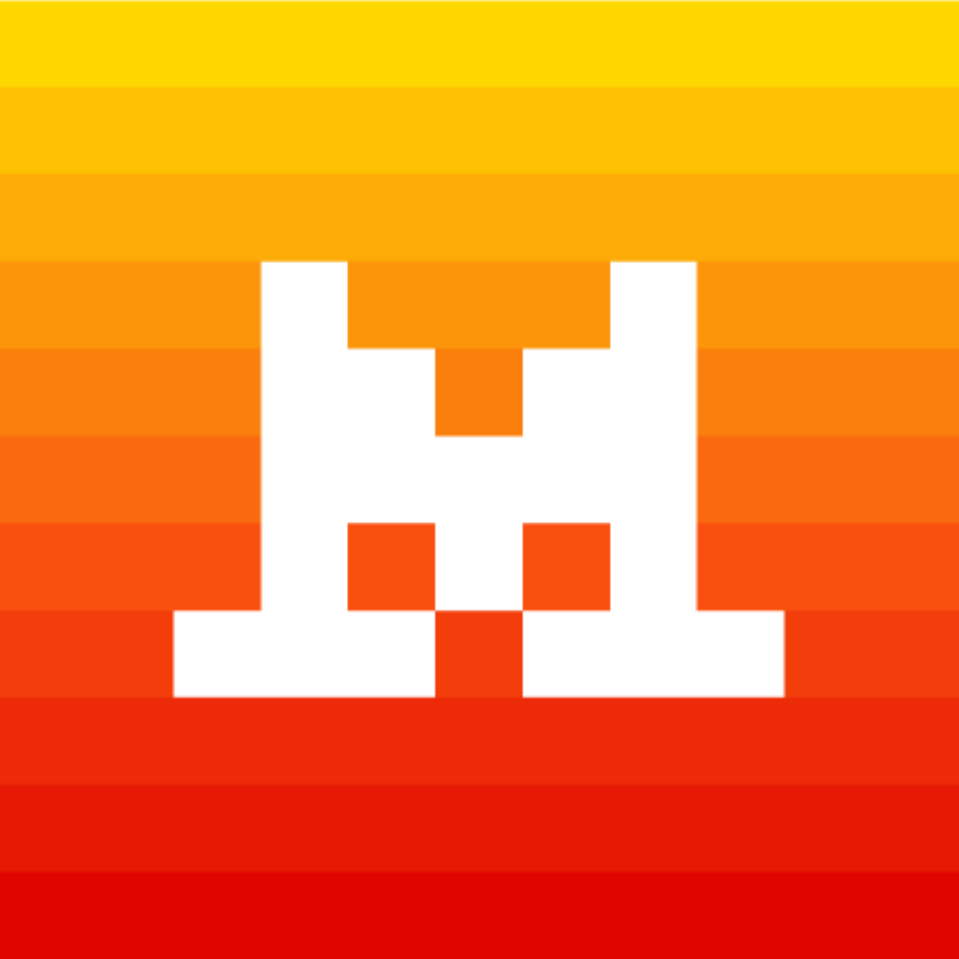}}~Mistral-7B        & \textbf{97.3} & 85.3 & 83.3 & --- & --- \\
\raisebox{-0.18em}{\includegraphics[height=0.95em]{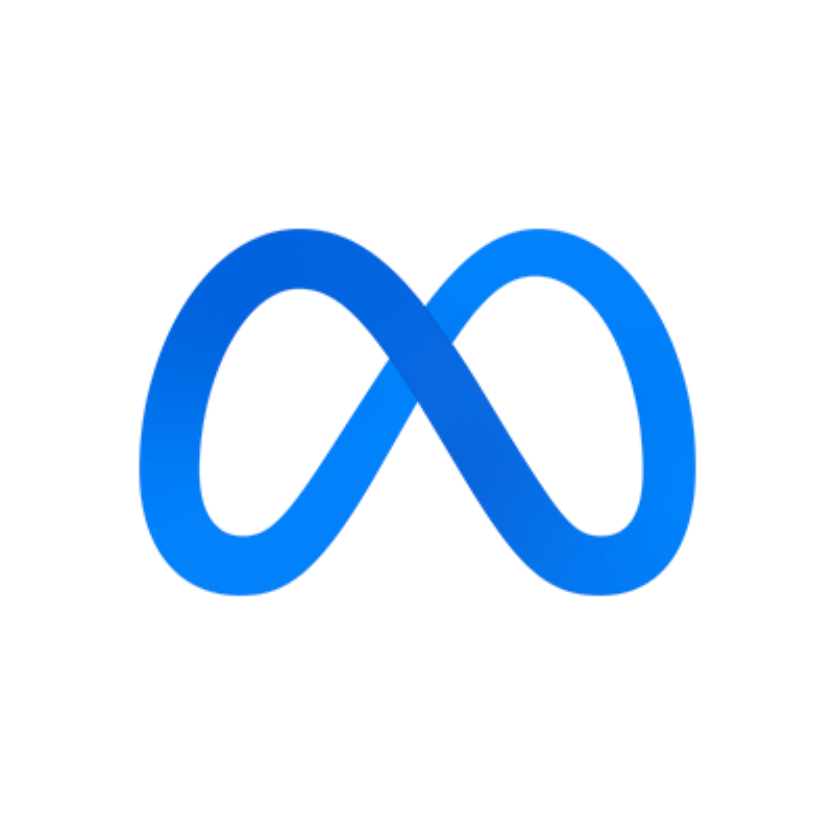}}~Llama-3.1-8B         & \textbf{94.0} & 56.7 & 39.7 & $+5$ (n.s.); $-22$ at $c^\star{=}2$ & --- \\
\midrule
Matched-rank random & $99.0 \pm 0.2$ & --- & --- & $0.0$ & $-1.7$ \\
\bottomrule
\end{tabular}
\caption{Argmax-preservation under $10$-d cross-form patching (paired Wilcoxon $p{<}10^{-22}$ vs.\ Full-PCA on every model). Behavioural: FARS ablation at the FARS layer costs $-20$ (Qwen-$7$B-Inst), $-28$ (Phi-$3.5$), $-17$ (Qwen-$3$B-Inst), and $-10$\,pp (Mistral-$7$B-Inst) on GSM8K relative to matched-rank random (paired Wilcoxon $p\le 0.03$); Llama-$8$B needs amplification $c^\star{=}2$ ($-22$\,pp, $p{=}3{\times}10^{-3}$); Llama-$70$B is inert at $c\le 2$ (App.~\ref{app:gsm8k}).}
\label{tab:causal_body}
\end{table}

Intervening on FARS affects reasoning performance in several models. Argmax preservation alone is less informative: random-subspace patches also preserve it. We therefore test readout overlap with positive controls (\S\ref{sec:readout_control}), then examine transfer across inventories (\S\ref{sec:novelty_transfer}), rank and layer sensitivity (\S\ref{sec:universality}), and the boundary of last-input-token analysis in reasoning models (\S\ref{sec:reasoning}). Cross-model alignment provides complementary evidence in App.~\ref{app:universal_frame}.

\subsection{What a concept-subspace intervention can change}
\label{sec:controlled_interventions}
Readout geometry and behavioural intervention answer different questions. We therefore extend cross-concept patching to GPT-2 XL, Qwen2.5-3B-Instruct, SmolLM3-3B, and Mamba-2.8B. Each model is evaluated with three seeds and $120$ directed concept pairs per seed, using the first stimulus instance in a shared surface form. We retain the earlier prefix-position protocol for comparison and add a last-token-to-last-token protocol with five conditions: FARS, a Haar-random rank-$10$ basis, the same random intervention rescaled to the FARS intervention norm, full-vector replacement, and no intervention. Bases and model layers are fixed across seeds.

For source $A$, target $B$, and patched output $P$, the recorded event is $D_{\mathrm{KL}}(p_A\Vert p_P)<D_{\mathrm{KL}}(p_B\Vert p_P)$. This is a measure of source-directed output change, not semantic accuracy. In the last-token protocol, the seed-mean FARS rates are $10.8\%$, $3.3\%$, $0.8\%$, and $0.8\%$, respectively (Fig.~\ref{fig:intervention_controls}). Full replacement yields $19.7$--$47.8\%$ across these models; both random controls and no intervention yield zero events in the sampled trials. The effects support a limited, model-dependent influence of the extracted subspace, not reliable substitution for the full state. Seeds reuse concept pairs; we report seed variation rather than treating every record as independent. Protocol details and both sets of results appear in App.~\ref{app:intervention_extension}.

\begin{figure}[t]
\centering
\includegraphics[width=\linewidth]{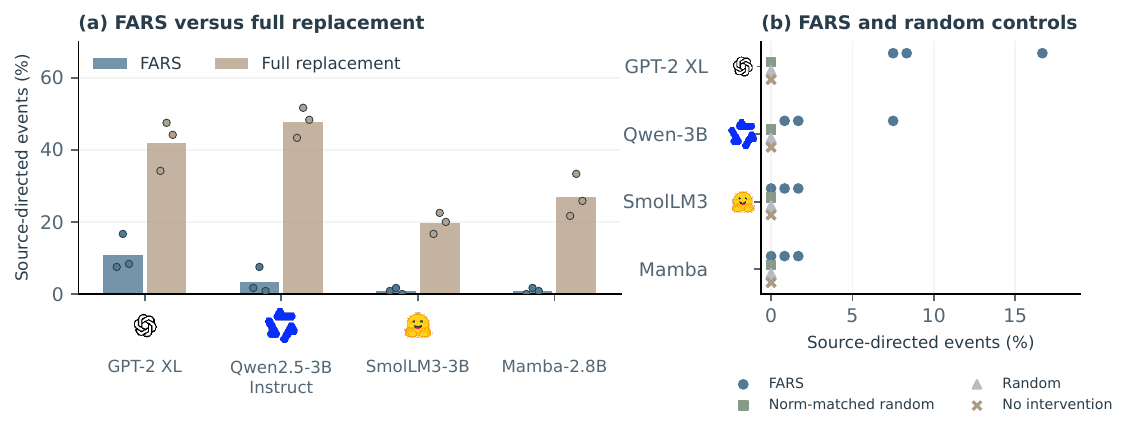}
\caption{\textbf{Concept-subspace interventions have limited source-directed effects.} Last-token replacement across four models; each dot is one seed with $120$ directed concept pairs. Bars show seed means, with no inferred confidence interval. Both random controls and the no-intervention control produce zero events. Full replacement is a reference intervention, not a guarantee of task success.}
\label{fig:intervention_controls}
\end{figure}

\section{Readout-Orthogonality Test and Method Validation}\label{sec:readout_control}


\paragraph{Per-model readout orthogonality.}
For each model we extract FARS at its best concept-RSA layer within the
$30$--$75\%$ depth band, take the top-$10$ right singular vectors of that model's
own unembedding matrix $W_U$ (the raw test uses $W_U$ alone; the
final-norm gain $\boldsymbol\gamma$ is folded in, with the layer translator, by the
depth-matched test of App.~\ref{app:depth_matched}, which returns the same verdict), and measure the Grassmann geodesic distance
between them, the smallest principal angle, and the fraction of FARS energy
lying inside the readout span.

Per-model values are tabulated in App.~\ref{app:depth_matched} (Table~\ref{tab:readout_disruption}).


\paragraph{Positive controls and four concept estimators give different readout loadings.}
A diagnostic that can only return one answer is not a diagnostic. We therefore
run the test on matched-rank subspaces of the same activation space across all
26 models: four estimators from the published literature, four structural
controls, and \emph{positive controls} designed to be readout-adjacent.

\begin{table}[t]
\centering
\small
\setlength{\tabcolsep}{5pt}
\begin{tabular}{@{}llcccc@{}}
\toprule
 & & \multicolumn{2}{c}{vs.\ $W_U$} & vs.\ $W_U\!\to\!\ell$ & \\
\cmidrule(lr){3-4} \cmidrule(lr){5-5}
Rank-$10$ subspace & role & $d_G$ & energy \% & energy \% & $\theta_{\min}$ \\
\midrule
\textbf{Final-layer PCA} & \textbf{positive ctrl} & 4.556 & 3.56 & 0.86 & 65.5 \\
FormPCA & control & 4.738 & 1.26 & 0.90 & 75.8 \\
FullPCA & control & 4.751 & 1.04 & 0.83 & 77.5 \\
Shuffled-FARS & control & 4.762 & 0.90 & 0.62 & 78.9 \\
Diff-of-means & published & 4.774 & 0.80 & 0.46 & 79.5 \\
FARS \emph{(ours)} & published & 4.774 & 0.80 & 0.46 & 79.5 \\
LDA & published & 4.805 & 0.41 & 0.39 & 83.5 \\
Probe weights & published & 4.811 & 0.38 & 0.33 & 83.6 \\
Haar random & control & 4.823 & 0.32 & 0.30 & 84.3 \\
\textbf{Next-token PCA at $\ell$} & \textbf{positive ctrl} & 4.665 & 2.34 & $\mathbf{5.88}$ & 69.6 \\
\bottomrule
\end{tabular}
\caption{Readout-orthogonality of nine matched-rank subspace estimators, averaged over 26 models spanning four architecture families plus reasoning-tuned variants ($0.4$B--$70$B). $d_G$ is the Grassmann geodesic distance to the top-$10$ right singular vectors of $W_U$ (ceiling $\sqrt{10}\pi/2 = 4.97$); energy \% is the fraction of subspace energy inside the top-$10$ readout span, measured against $W_U$ directly and against the readout channels pulled back to the FARS layer $\ell$ by a tuned-lens translator ($W_U\!\to\!\ell$; App.~\ref{app:depth_matched}, 25 models); $\theta_{\min}$ is the smallest principal angle. Rows are ordered by readout-adjacency. \emph{Published} estimators are the constructions used in the concept-direction literature; \emph{control} rows probe rank and structure; the \emph{positive controls} are the top-$10$ PCA of the \emph{final} layer, whose residual stream is the immediate input to $W_U$, and next-token-identity centroid PCA at the FARS layer itself --- the same estimator as FARS with the grouping variable changed from concept to emitted token.}
\label{tab:method_8way}
\end{table}

\textbf{The test is two-sided.} The final-layer positive control carries
$3.56\%$ of its energy in the top-$10$ readout span against $0.80\%$ for FARS
(a $4.5\times$ separation), with the minimum principal angle falling from
$79^\circ$ to $65^\circ$; per-model it is strictly more readout-adjacent than
\emph{all four} published estimators in 25/26 models by energy and 25/26 by
minimum angle. The near-orthogonal readings elsewhere are therefore substantive
rather than an artefact of the statistic. The single exception is R1-Distill-Qwen-7B, whose last-input-token FARS loads $5.78\%$ against a cross-model mean of $0.80\%$; it is the model whose concept computation \S\ref{sec:reasoning} shows moving off the last-input-token position, and the depth-matched test below resolves it.

\textbf{{The verdict survives matching depth.}} A mid-layer subspace could be far from $W_U$ merely because many layers separate it from the unembedding. We therefore pull the readout back to the FARS layer $\ell$ with a tuned-lens-style translator $A_\ell$ (ridge from block-$\ell$ to final-block states, held-out $R^2$ 0.23--0.63) and take the top-$k$ right singular vectors of $W_U\,\mathrm{{diag}}(\boldsymbol\gamma)A_\ell$ ($\boldsymbol\gamma$ the final-norm gain) as the readout channels \emph{{at layer $\ell$}}; as a same-layer positive control we build FARS's own estimator with the grouping variable changed from concept to emitted token (next-token centroid PCA). Against the pulled-back readout the concept estimators stay at the Haar floor (FARS $0.46\%$, LDA $0.39\%$, random $0.30\%$) while the same-layer next-token subspace loads $5.88\%$ --- a 13$\times$ separation that exceeds all four published estimators in 25/25 models (App.~\ref{app:depth_matched}). With layer, rank, estimator, and corpus held fixed, grouping by output token yields a readout channel and grouping by concept does not: this comparison separates grouping variables at matched depth under the fitted linear reference. It also resolves the exception above: R1-Distill-Qwen-7B's FARS loads $5.78\%$ on the raw $W_U$ channels but only $0.69\%$ on the layer-$\ell$ channels; its relative loading depends on the choice of readout reference.

\textbf{The verdict is not specific to our estimator.} Multi-class LDA,
difference-of-means directions, and one-vs-rest probe weights --- the constructions
the concept-direction literature actually uses --- all have low top-$10$ readout loading, with
LDA and probe weights \emph{further} from the readout than FARS. Sorting by
readout-adjacency recovers the form-vs-concept axis: form-dominant subspaces (FormPCA
$1.26\%$, FullPCA $1.04\%$) sit nearer the readout than concept-dominant ones
(FARS $0.80\%$, LDA $0.41\%$, probe weights $0.38\%$). This ordering associates surface-form grouping with greater readout loading in this dataset; it does not by itself identify where concept computation occurs.

\begin{lemma}[Difference-of-means directions span FARS]
\label{lem:diffmeans}
Let $\mu_c$ be the class-conditional mean activation of concept $c$ over a
balanced stimulus set and $\bar\mu$ the grand mean. Then
$\mathrm{span}\{\mu_A - \mu_B\}_{A,B} = \mathrm{span}\{\mu_c - \bar\mu\}_c$,
so the top-$k$ principal subspace of the pairwise difference matrix is exactly
the concept-centroid PCA subspace (FARS).
\end{lemma}

\begin{proof}[Proof sketch]
The two spans coincide by writing each generator in terms of the other, and the
pairwise-difference scatter equals $2n$ times the concept-centroid scatter, so the
two share eigenvectors and hence their top-$k$ subspaces are identical, not merely
co-spanning. Full derivation in App.~\ref{app:lemma_proof}.
\end{proof}

Lemma~\ref{lem:diffmeans} is why the \emph{Diff-of-means} and \emph{FARS} rows of
Table~\ref{tab:method_8way} coincide to numerical precision (verified at
$d_G < 10^{-7}$ on every model). The equivalence applies to the same class means, corpus, layer, and balanced pairwise construction. It does not imply that arbitrary steering directions extracted from other concepts or prompts lie in this FARS basis. Moreover, low average loading of a rank-$10$ subspace need not give the same loading to each individual direction. The published steering methods \citep{turner2023activation,arditi2024refusal} motivate this comparison, rather than inheriting its empirical result automatically.

\paragraph{Sensitivity to readout rank.} The rank sweep in App.~\ref{app:rank_sweep} compares loading against the exact Haar expectation $k/D$. Final-layer PCA exceeds FARS at the three reported ranks in $24/26$ models; the separation generally narrows as rank grows. This is evidence at the tested ranks, not an order-of-magnitude separation at every rank.

\textbf{Published direction constructions show the same depth sensitivity.} Rebuilding the refusal direction of \citet{arditi2024refusal} and the truth direction of \citet{marks2024geometry} at the FARS block on 22 models (held-out AUC $\ge 0.9$ in 22 and 20), both have low loading relative to the Haar reference against the pulled-back readout (medians 0.44\% and 0.20\% vs.\ $k/D\approx0.26\%$; low-loading criterion satisfied in 16/22 and 19/20), whereas the same directions built at the final block load 1.98\% and 4.09\% --- a rise in 35/39 paired comparisons that change nothing but the extraction depth (App.~\ref{app:published}). The test applies unchanged to the directions safety and steering work manipulates.

\textbf{The verdict is not unanimous: it tracks where the direction is defined.} Building the same $18$ concepts from the \emph{vocabulary} instead --- tf-idf-diagnostic tokens per concept, averaged over rows of $W_U$, PCA'd to rank $10$, the construction behind token-anchored steering --- yields a subspace that loads $8.43\%$ of its energy in the readout span against $0.80\%$ for FARS (mean per-model ratio $16.9\times$, same direction in 26/26 models; random-token control $5.20\%$), while overlapping FARS by only $0.56\%$ (App.~\ref{app:vocab_anchored}). Activation-derived and vocabulary-derived constructions therefore differ in their overlap with the chosen readout span, even when they use the same concept labels.

It also covers the estimator the field reaches for first: concept features of a pretrained residual-stream sparse autoencoder \citep{bricken2023sae} sit at $1.06\%$, indistinguishable from FARS at the same layer ($0.96\%$), while that SAE's full dictionary sits at $3.93\%$ (App.~\ref{app:sae}; one model).

\begin{figure}[t]
\centering
\includegraphics[width=\linewidth]{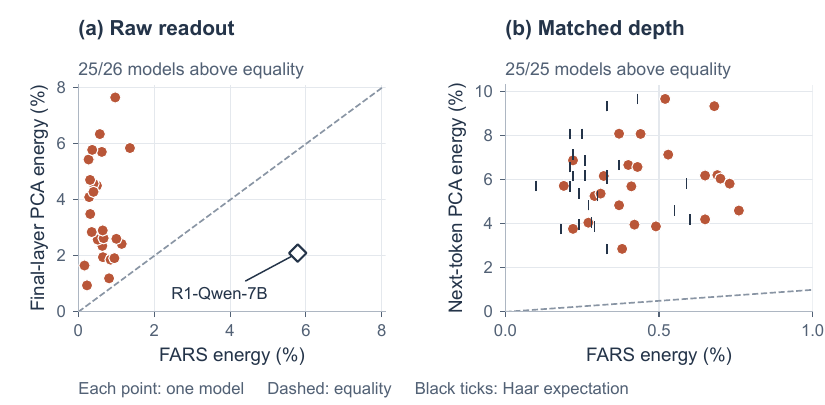}
\caption{\textbf{Positive controls distinguish readout loading.} Each point is a model. (a) Final-layer PCA exceeds FARS in $25/26$ models; the diamond marks R1-Distill-Qwen-7B. (b) Same-layer next-token PCA exceeds FARS in all $25$ depth-matched models. Black ticks mark Haar expectations at the corresponding control energy. Dashed lines indicate equal loading; axis ranges differ. Panel (b) uses table values rounded to $0.01$ percentage point.}
\label{fig:readout_two_sided}
\end{figure}

\section{Cross-Inventory Transfer: Novelty Benchmark}

The preceding comparison fixes the concept inventory. We next ask whether the extraction procedure recovers cross-format structure when that inventory changes.

\label{sec:novelty_transfer}

A natural worry is that FARS geometry is an artefact of the eighteen
hand-crafted TriForm concepts rather than a cross-format signal in the models.
We test this by extracting FARS on a disjoint inventory --- the
\emph{Novelty benchmark} introduced here --- and asking whether the method
reproduces its defining property, cross-format concept retrieval.

\paragraph{Novelty benchmark.} Ten concepts from five domains absent
from TriForm (temporal, probabilistic, graph, data-structure, quantifier),

each with three instances rendered in the same six surface forms as
TriForm, giving 180 stimuli (inventory in Table~\ref{tab:novelty_inventory}, App.~\ref{app:benchmark}). We keep $k{=}10$ for comparability,
which makes Novelty-FARS full-rank in its own concept space (ten centroids span
at most nine dimensions), so it is read against the matched-rank random baseline,
not against TriForm.

Per-model numbers are in Table~\ref{tab:novelty_retrieval} (App.~\ref{app:holdout_concepts}): Novelty-FARS reaches 62--100\% top-$1$ cross-format retrieval on the 24 generative models (32\% on the encoder DeBERTa-v3), 6--57\,pp above a matched-rank Haar baseline on every model, with a cross-to-within-format retrieval ratio $X/W \ge 0.94$ in 17 of 25 ($X$ scores a stimulus against centroids from the \emph{other} five forms and $W$ against its own, so $X/W \to 1$ means retrieval ignores surface form). \emph{Scope:} this analysis covers 25 of the twenty-six models. R1-Distill-Qwen-$14$B is the one checkpoint for which we did not run a Novelty extraction, so it is absent from this table only; it is present in every analysis that uses the TriForm inventory.

\paragraph{Fixed basis versus re-extraction.} A checkpoint-matched audit freezes the saved TriForm basis and layer, then evaluates Novelty cross-form retrieval against full-space and same-layer random controls (App.~\ref{app:frozen_transfer}). Across GPT-2 XL, Qwen2.5-7B and Mistral-7B-v0.3, frozen-basis accuracy is $22.8$--$40.0\%$, above the sampled random means ($20.2$--$35.3\%$) but below full-space retrieval ($48.3$--$90.6\%$). This is partial transfer, not uniformly near-chance performance or evidence for a universal fixed basis.

\paragraph{Reading.} Re-extraction reproduces cross-format structure on a second, disjoint inventory. Fixed-basis transfer is substantially weaker in the audited subset. The evidence supports transfer of the extraction procedure across these two inventories, not a claim about every possible concept set.


\section{Dimension Choice and Layer Localization}\label{sec:universality}

Transfer across inventories leaves two choices to check: the rank of the basis and the layer at which it is extracted.

\paragraph{$k=10$ is an empirical operating point.} Sweeping $k\in\{1,\ldots,17\}$ on TriForm, $k=10$ captures $\sim 90\%$ of both the full-rank ($k=17$) centroid variance and the full-rank cross-format retrieval, with marginal gains $\le 10$\,pp beyond it, uniformly across dense, state-space, and encoder-MLM models (App.~\ref{app:k_sweep}, Tables~\ref{tab:k_sweep_triform}--\ref{tab:k_sweep_per_model}).


\paragraph{FARS is layer-local and its peak sits above the readout layer.}
For each of the 26 models we recompute concept-centroid PCA at every
layer and record two quantities: top-$1$ cross-format concept retrieval
Agn\% (concept-invariance) and the Grassmann distance $d_G$ from the
layer-$L$ FARS to the top-$10$ readout channels of $W_U$.
Table~\ref{tab:layer_loc} reports the Agn\% peak along depth with its
$d_G$, against the last layer (Fig.~\ref{fig:layer_loc} and Table~\ref{tab:layer_loc}, App.~\ref{app:correlational_localization}).

\paragraph{Reading.} Retrieval peaks at different depths across models (reported depth fractions $0.29$--$0.97$), so the result is not a universal middle-layer location. Peak-versus-final comparisons show lower retrieval at the final layer and generally smaller distance to the readout span. Peak selection itself favours higher retrieval; it is the accompanying geometry comparison that supplies the additional observation. These measurements constrain a layer-local interpretation but do not establish the complete causal pathway.

\paragraph{Cross-architecture algebra (App.~\ref{app:universal_frame}).} The per-model subspaces of all $27$ extracted models align into one shared canonical frame (X-FARS, $97.6\%$ of cross-model centroid variance, $36.8\sigma$ above a random-projection control at $61.7\pm1.0\%$) that supports word2vec-style analogies, preserves pairwise differences, and enables training-free cross-architecture retrieval.


\section{Two Concept Subspaces in Reasoning-Tuned Models}
\label{sec:reasoning}

A last-input-token pipeline applied to reasoning-tuned models raises a scoping
question: models that ``think'' in a chain of thought (CoT) may deposit concept
content at positions a pre-decode pipeline never inspects. We read the \emph{CoT-tail}: the last position
generated under a $512$-token budget. The original length audit reports budget-hit rates of $64$--$88\%$ for three distilled models, a different cohort from the geometry comparison (App.~\ref{app:cot_truncation}) --- a claim about the reasoning stream, not its end. We answer it on the two
open-weight R1-Distill models (Llama-8B, Qwen-7B) with their base counterparts,
and on the native $32$B reasoner QwQ.

\begin{table}[t]
\centering
\footnotesize
\setlength{\tabcolsep}{3pt}
\begin{tabular}{@{}lcccc@{}}
\toprule
Probe & Llama pair & Qwen pair & QwQ-32B & Baseline \\
\midrule
$d_G$(base last, distill last)      & $\mathbf{3.06}$ & $\mathbf{3.48}$ & --- & $4.97$ (rank-$10$ ceiling) \\
$d_G$(last, CoT)                    & $\mathbf{4.53}$ & $\mathbf{4.26}$ & $\mathbf{4.33}$ & $4.97$ (rank-$10$ ceiling) \\
CoT-tail FARS \% cross-format Agn.\ & $\mathbf{60.8}$ & $\mathbf{58.0}$ & $\mathbf{59.0}$ & $21.6$--$24.4$ (Haar rand.) \\
X-FARS pool $+$Llama-Distill & \multicolumn{2}{c}{$\mathbf{55.6\%}$ ($+1.7$\,pp)} & --- & $53.8\%$ (15 models) \\
\bottomrule
\end{tabular}
\caption{Reasoning-tuned CoT departure at a glance. \emph{Row 1} (Q1):
base-to-distilled distances are nonzero and below the rank-$10$ ceiling; this is a geometric comparison in paired hidden coordinate spaces. \emph{Row 2} (Q2): within
each reasoning model, the last-token and CoT-tail subspaces sit
at large aggregate Grassmann distances,
including the natively trained QwQ-32B.
\emph{Row 3}: the CoT-tail subspace itself carries cross-format
concept identity at $58$--$61\%$ top-$1$ retrieval, $34$--$39$\,pp
above a matched-rank Haar-random baseline. \emph{Row 4} (Q3):
absorbing R1-Distill-Llama-8B into the shared X-FARS frame improves
absolute alignment ($53.8\to 55.6\%$) with the Haar-random-null gap
$27.2$--$29.6$\,pp across the three pools. Detailed tables and run distinctions appear in App.~\ref{app:reasoning_full}.}
\label{tab:reasoning_summary}
\end{table}

A format-collapse control tempers this. Regenerating the traces and keeping their text shows that 56--57\% of non-English prompts yield an English trace, so grouping by the prompt's form overstates how cross-format the position is. Regrouping the \emph{cross-format} score $X$ --- not the agnostic top-$1$ rate above, but the score that pits a stimulus against centroids from \emph{other} forms --- by the form the model actually wrote in lowers it from 30.6--44.4\% to 26.9--28.1\%, still 3.3--5.1$\times$ a matched-rank Haar control measured the same way (2 models, a fresh set of generations; App.~\ref{app:cot_format}).

\paragraph{Reading.} Under reasoning fine-tuning the picture multiplies rather than breaks. The paired and pooled comparisons quantify how the recovered geometry changes (Q1, Q3); Q2 identifies a different subspace at a generated-token position and separately selected layer. QwQ-32B also shows a large aggregate distance ($d_G{=}4.33$) and $59.0\%$ retrieval. Position, layer choice, and output format remain potential contributors; these observations do not isolate the computation responsible for reasoning.

\section{Discussion}
\label{sec:discussion}

FARS captures cross-format concept structure with low overlap with the dominant linear readout span in the tested settings. Final-layer and same-layer next-token controls show that this is a comparative result: output-oriented subspaces carry more readout energy. Together with the intervention results, this is consistent with concept representations contributing before final output readout. It does not identify the complete computation or prove that the subspace alone is sufficient to perform it.

The raw diagnostic is inexpensive once a basis is available, but the full pipeline is not compute-free: extraction requires activations, and depth matching requires fitting a translator. The latter explains only part of final-layer variation (held-out $R^2=0.23$--$0.63$), so it is an approximate linear reference. Rank sweeps test sensitivity to a chosen readout span, not an absence of information in all output directions; for a complete orthonormal basis at $k=D$, every subspace has unit loading.

The scope is also empirical. Labels supervise extraction. Fixed-basis transfer is partial and task dependent: on Novelty, three checkpoint-matched models reach $22.8$--$40.0\%$ retrieval, below full-space retrieval at $48.3$--$90.6\%$ (App.~\ref{app:frozen_transfer}); on out-of-distribution MBPP, the earlier retrieval analysis reaches only $1$--$6\%$ versus $6.5$--$17\%$ for full activations (App.~\ref{app:universal_frame}). Stronger re-extraction results test the procedure, not a universal fixed basis. The shared frame is fitted to a fixed inventory, rather than establishing universal concept coordinates. CoT-tail analyses sample a token-budget-limited reasoning stream, and an output-facingness prediction did not replicate (App.~\ref{app:concept_families}). These boundaries make the diagnostic most useful as a control alongside targeted interventions and decoding tests.

\section*{Ethics Statement}

We analyse only publicly released open-weight LLMs. No human-subjects data are collected or released, and the TriForm and Novelty stimuli comprise canonical reasoning patterns with no harmful content. One appendix experiment (App.~\ref{app:advbench}) does use published AdvBench prompts \citep{zou2023universal} to test whether a prose-vs-code representational asymmetry produces a safety-refusal gap; it does not, and we release only aggregate refusal counts, no model completions and no new attack strings. Had the gap been real we would have deferred to responsible disclosure before release.

\section*{Reproducibility Statement}

The paper specifies stimulus inventories, concept-centroid extraction, readout diagnostics, controls, and intervention procedures. Per-experiment cohorts, selected layers, and available summary values appear in the appendix. Reproducing the full study additionally requires the exact stimulus files, activation caches, extraction and alignment implementation, and checkpoint revisions. A public code and research snapshot is available at \url{https://github.com/Justin0504/LLM-representation-agnostic} (commit \texttt{4d74d85}). It includes the benchmark and core analysis code, but does not yet contain the September 28--29 supplementary runs; it is not a complete archive of every experiment reported here. The raw diagnostic requires only an extracted basis and unembedding weights; the depth-matched control additionally requires fitted activation-based translators. The reported compute description in App.~\ref{app:code} concerns the original experiments and has not been independently remeasured.

\section*{AI Use Statement}
LLM-based assistants (Claude Code and OpenAI Codex) assisted with code development, literature retrieval, experimental planning, prose revision, figure design, and figure-data provenance checks. Codex also prepared editable vector and PowerPoint assets and screened references against primary-source records. Assistant-assisted scripts also executed supplementary model experiments; the reported measurements come from recorded inference runs rather than generated illustrative values. The authors are responsible for the experimental design, verification, interpretation, and final submitted content.

\appendix
\renewcommand{\thesection}{\AlphAlph{\value{section}}}
\clearpage
\section*{Appendix: Evidence and Reproduction}
The appendix follows the evidence chain from controlled inputs to calibration, intervention, transfer, and exploratory alignment. Each section opens with its question and an interpretation boundary. Model counts refer to the stated analysis, not a single shared cohort.
\begin{center}\small
\setlength{\tabcolsep}{5pt}
\begin{tabular}{@{}p{.28\linewidth}p{.64\linewidth}@{}}\toprule
\textbf{Reading route} & \textbf{Evidence and its scope} \\ \midrule
Reproduction & Code, stimuli, inventories and extraction (Apps.~\ref{app:code}--\ref{app:method_comparison}). The raw diagnostic assumes an already extracted basis. \\
Readout geometry & Depth, rank and direction controls (Apps.~\ref{app:depth_matched}--\ref{app:concept_families}). Low loading refers to specified linear readout spans. \\
Interventions & New multi-seed controls and historical behavioural tests (Apps.~\ref{app:intervention_extension}--\ref{app:layer_writers}). Perturbation effects do not establish sufficiency. \\
Generalization & Disjoint concepts and capability associations (Apps.~\ref{app:holdout_concepts}--\ref{app:capability_prediction}). Procedure transfer differs from fixed-basis transfer. \\
Reasoning positions & CoT extraction, format and budget controls (App.~\ref{app:reasoning_full}). Original and regenerated runs remain separate. \\
Exploratory alignment & Cross-model frames, axes, composition and prediction (Apps.~\ref{app:universal_frame}--\ref{app:cross_model_predict}). Historical fits do not validate the X-FARS objective or unseen-concept generalization. \\
\bottomrule\end{tabular}\end{center}
\paragraph{Evidence status.} The new intervention extension contains $24$ completed runs and saved per-pair records. Older figures and tables retain their reported provenance; where a figure is reconstructed from a plotting-script summary rather than raw trial records, its caption says so. Negative and non-replicating results are retained. No result is promoted to a stronger claim merely because its visual presentation has changed.

\FloatBarrier
\Needspace{10\baselineskip}
\section*{1. Reproduction and controlled inputs}
\section{Implementation}\label{app:code}
\paragraph{Question and scope.} How is the raw diagnostic reproduced? The code operates on a supplied basis and readout weights; it does not reproduce extraction or validate an X-FARS optimizer.

\noindent\begin{minipage}{\linewidth}
\noindent\textit{FARS extraction:}
\begin{lstlisting}[style=purplebox]
def extract_fars(activations, c_labels, k=10):
  N, L, D = activations.shape
  basis = np.zeros((L, k, D))
  for layer in range(L):
    X = activations[:, layer, :]
    centroids = np.stack([
      X[c_labels == c].mean(0)
      for c in np.unique(c_labels)])
    pca = PCA(n_components=k).fit(centroids)
    basis[layer] = pca.components_
  return basis
\end{lstlisting}
\end{minipage}

\noindent\begin{minipage}{\linewidth}
\noindent\textit{Subspace patching} (Eq.~\ref{eq:subspace_patch}):
\begin{lstlisting}[style=purplebox]
def subspace_patch(h_src, h_tgt, B):
  # h_*: (..., D) row vectors; B: (k, D) orthonormal rows.
  diff = h_src - h_tgt
  return h_tgt + (diff @ B.T) @ B   # batches over leading dims
\end{lstlisting}
\end{minipage}

\paragraph{Compute.} ${\sim}$80 GPU-hours total on NVIDIA RTX 6000 Ada ($48$\,GB) and NVIDIA A100 ($40$\,GB), including a single multi-GPU sharded forward-pass run on Llama-3.1-70B-Instruct ($8\times$ RTX 6000 Ada) for the frontier-scale FARS extraction and compositionality replication.

\paragraph{Code and data availability.} A public code and data archive is not specified in this version. The mathematical definitions and pseudocode above describe the extraction and patching operations; they do not replace the complete experimental implementation or cached activation data.


\section{Stimulus Examples}\label{app:stimuli}
\paragraph{Question and scope.} What changes across surface forms? The examples illustrate the controlled inventory; they do not establish coverage of natural reasoning tasks.

\begin{table}[H]
\centering
\caption{Example: modus ponens in 6 surface forms.}
\footnotesize
\begin{tabular}{@{}lp{6cm}@{}}
\toprule
\textbf{Form} & \textbf{Text} \\
\midrule
English & If it rains then the ground is wet. It rains. Therefore the ground is wet. \\
Chinese & 如果下雨，那么地面是湿的。下雨了，所以地面是湿的。 \\
French & S'il pleut alors le sol est mouill\'{e}. Il pleut. Donc le sol est mouill\'{e}. \\
Python & \texttt{def mp(p, q): return q if p else None} \\
Math & $P \to Q,\; P \;\vdash\; Q$ \\
Struct. & \texttt{P1: P->Q | P2: P | Rule: MP | Q} \\
\bottomrule
\end{tabular}
\end{table}

\section{TriForm Benchmark Concept Inventory}\label{app:benchmark}
\paragraph{Question and scope.} Which concepts are included? TriForm and Novelty are separate inventories. Their model-specific evaluation pools must not be combined.

\begin{table}[H]
\centering
\caption{18 concepts across 5 reasoning domains.}
\footnotesize
\begin{tabular}{@{}llp{3.5cm}@{}}
\toprule
\textbf{Domain} & \textbf{Concept} & \textbf{Example} \\
\midrule
Arithmetic & Multi-step eval & $(7{+}3){\times}4{-}8/2$ \\
 & Modular & $47 \bmod 7$ \\
 & Proportional & 3 cost \$12; 7 cost? \\
 & GCD & $\gcd(48, 18)$ \\
\midrule
Logic & Syllogism & All A are B; all C are A \\
 & Modus ponens & P$\to$Q; P; therefore Q \\
 & Contrapositive & P$\to$Q; $\neg$Q; $\neg$P \\
 & De Morgan & $\neg(A \wedge B)$ \\
\midrule
Relational & Transitive & A$>$B, B$>$C \\
 & Set intersection & $A \cap B$ \\
 & Set difference & $A \setminus B$ \\
 & Composition & $f(g(x))$ \\
\midrule
Causal & Chain & A$\to$B$\to$C \\
 & Confounding & A$\leftarrow$C$\to$B \\
 & Interventional & $P(Y|\text{do}(X))$ \\
\midrule
Spatial & Direction & N of B, B E of C \\
 & Containment & A in B, B in C \\
 & Rotation & 90\textdegree\ rotation \\
\bottomrule
\end{tabular}
\end{table}

\begin{table}[H]
\centering
\caption{The disjoint \emph{Novelty} inventory: $10$ concepts across $5$ domains, none of
them in TriForm. Rendered in the same six surface forms, $3$ instances each, giving
$180$ stimuli. Examples are the mathematical-notation rendering.}
\label{tab:novelty_inventory}
\footnotesize
\begin{tabular}{@{}llp{4.6cm}@{}}
\toprule
\textbf{Domain} & \textbf{Concept} & \textbf{Example (math form)} \\
\midrule
Temporal & Event ordering & $t_A = 15{:}00$, $t_B = t_A + 30$\,min, $t_C \in (t_B, 16{:}00)$ \\
 & Duration arithmetic & $t_{\text{end}} = t_{\text{start}} + \Delta t$ \\
\midrule
Probabilistic & Conditional probability & $P(B_2 {=} \text{red} \mid B_1 {=} \text{red}) = 2/4$ \\
 & Expected value & $E[X] = \sum_{i=1}^{6} i \cdot (1/6) = 3.5$ \\
\midrule
Graph & Shortest path & $d(A,C) = \min(d(A,C),\, d(A,B){+}d(B,C))$ \\
 & Connectivity & $E = \{(A,B),(C,D)\}$; $A \nsim C$ \\
\midrule
Data structure & Stack (LIFO) & $\text{push}(1,2,3) \Rightarrow S=(1,2,3)$; $\text{pop} \to 3$ \\
 & Tree lookup & BST invariant $L(v) < v < R(v)$ \\
\midrule
Quantifier & Universal check & $\forall a \in \text{Basket},\ \text{red}(a)$ \\
 & Existential check & $\exists\, b \in \text{Shelf} : \text{subject}(b) = \text{physics}$ \\
\bottomrule
\end{tabular}
\end{table}

\section{Subspace Extraction Details}\label{app:subspace}
\paragraph{Question and scope.} How are bases, ranks, and layers chosen? Layer selection rules differ in the historical CoT analysis; this section does not override those protocol-specific choices.

\begin{table}[H]
\centering
\caption{Cluster purity at the best FARS layer. $k$-means accuracy (Hungarian-matched). Chance: $5.6\%$ concept ($18$-way), $16.7\%$ form ($6$-way). FARS and the form-control subspace are quantitatively orthogonal information carriers.}\label{tab:purity_app}
\footnotesize
\setlength{\tabcolsep}{3pt}
\begin{tabular}{@{}llcccc@{}}
\toprule
\textbf{Model} & \textbf{Space} & \textbf{Dim} & \textbf{Concept-ARI} & \textbf{Concept\%} & \textbf{Form\%} \\
\midrule
\multirow{3}{*}{\raisebox{-0.18em}{\includegraphics[height=0.95em]{figures/logos/openai.pdf}}~GPT-2 XL} & Full & 1600 & .142 & 28.7 & 46.0 \\
 & FARS & 10 & .300 & \textbf{42.6} & 25.9 \\
 & Form ctrl & 5 & .038 & 17.6 & \textbf{92.6} \\
\midrule
\multirow{3}{*}{\raisebox{-0.18em}{\includegraphics[height=0.95em]{figures/logos/qwen.pdf}}~Qwen-7B} & Full & 3584 & .284 & 42.3 & 25.9 \\
 & FARS & 10 & .619 & \textbf{66.7} & 23.8 \\
 & Form ctrl & 5 & .036 & 16.7 & \textbf{95.4} \\
\midrule
\multirow{3}{*}{\raisebox{-0.18em}{\includegraphics[height=0.95em]{figures/logos/mistral.pdf}}~Mistral-7B} & Full & 4096 & .257 & 40.4 & 32.4 \\
 & FARS & 10 & .647 & \textbf{69.4} & 23.1 \\
 & Form ctrl & 5 & .041 & 16.7 & \textbf{96.0} \\
\midrule
\multirow{3}{*}{\raisebox{-0.18em}{\includegraphics[height=0.95em]{figures/logos/qwen.pdf}}~Qwen-7B\textsection} & Full & 3584 & .292 & 41.7 & 24.4 \\
 & FARS & 10 & .623 & \textbf{68.8} & 24.7 \\
 & Form ctrl & 5 & .042 & 16.7 & \textbf{98.1} \\
\midrule
\multirow{3}{*}{\raisebox{-0.18em}{\includegraphics[height=0.95em]{figures/logos/meta.pdf}}~Llama-8B\textsection} & Full & 4096 & .311 & 42.0 & 33.0 \\
 & FARS & 10 & .640 & \textbf{67.6} & \emph{17.9} \\
 & Form ctrl & 5 & .042 & 16.7 & \textbf{95.7} \\
\bottomrule
\end{tabular}
\end{table}

\begin{figure*}[t]
\centering
\includegraphics[width=\linewidth]{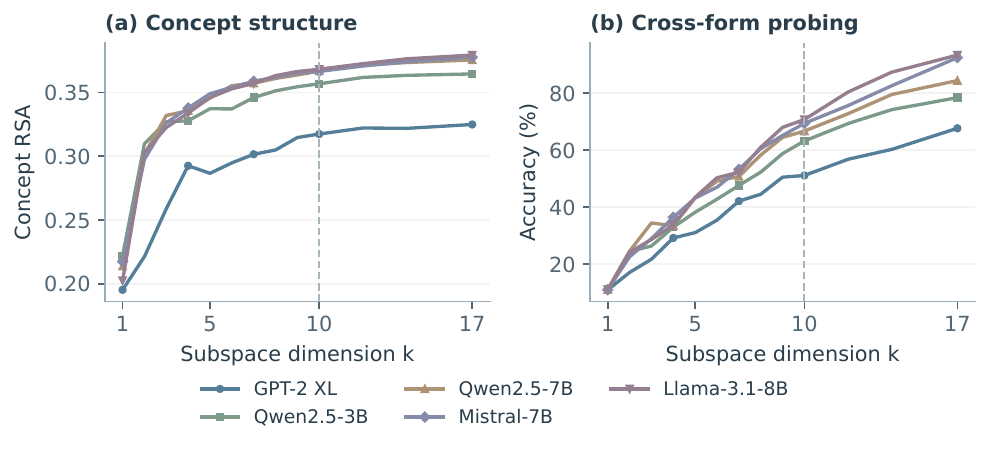}
\caption{FARS dimensionality sweep ($k\in\{1,\ldots,17\}$). The saved five-model sweeps attain $97.1$--$97.9\%$ of their $k{=}17$ concept RSA at $k{=}10$. Each point reports the saved best-layer statistic; selected layers can vary with $k$.}\label{fig:dim_sweep}
\end{figure*}

\subsection{Instruction-tuning robustness}\label{app:instruct_app}

\begin{table}[H]
\centering
\caption{Instruction tuning preserves and slightly strengthens FARS.}\label{tab:instruct}
\footnotesize
\setlength{\tabcolsep}{3pt}
\begin{tabular}{@{}lccccc@{}}
\toprule
\textbf{Model} & \textbf{RSA-C} & \textbf{Probe\%} & \textbf{Purity\%} & \textbf{FARS Patch} & $z$ \\
\midrule
\raisebox{-0.18em}{\includegraphics[height=0.95em]{figures/logos/qwen.pdf}}~Qwen2.5-7B (base) & .150 & 59.8 & 66.7 & .943 & 28.8 \\
\, +Instruct & \textbf{.153} & \textbf{62.9} & \textbf{68.8} & .867 & 25.5 \\
\midrule
\raisebox{-0.18em}{\includegraphics[height=0.95em]{figures/logos/meta.pdf}}~Llama-3.1-8B (base) & .165 & 53.9 & --- & .941 & --- \\
\, +Instruct & \textbf{.187} & \textbf{59.8} & 67.6 & .890 & 12.6 \\
\bottomrule
\end{tabular}
\end{table}

\subsection{Out-of-distribution retrieval (scope limit)}\label{app:retrieval_ood}

\begin{table*}[t]
\centering
\caption{In-distribution vs.\ out-of-distribution retrieval on Llama-3.1-8B-Instruct. Inside TriForm, $10$-d FARS outperforms full $4096$-d by $+18$ to $+23$\,pp; on MBPP-sanitized whose problems lie outside the $18$ TriForm concepts, FARS is far below full, the construction is concept-distribution-specific by design.}\label{tab:retrieval_ood}
\footnotesize
\setlength{\tabcolsep}{4pt}
\begin{tabular}{@{}lcccc@{}}
\toprule
\textbf{Benchmark} & \textbf{Direction} & \textbf{Full top-1} & \textbf{FARS top-1} & $\Delta$ \\
\midrule
TriForm in-dist.\ ($n{=}30$ pairs, avg.) & all & $.559$ & $\boldsymbol{.738}$ & $\boldsymbol{+.179}$ \\
MBPP-sanitized ($n{=}200$) & prose$\!\to\!$code & $.065$ & $.025$ & $-.040$ \\
MBPP-sanitized ($n{=}200$) & code$\!\to\!$prose & $.150$ & $.030$ & $-.120$ \\
\bottomrule
\end{tabular}
\end{table*}

\begin{table}[H]
\centering
\caption{FARS (10-dim) vs.\ Form Control (5-dim).}\label{tab:subspace_detail}
\footnotesize
\setlength{\tabcolsep}{3pt}
\begin{tabular}{@{}lcccc|cc@{}}
\toprule
 & \multicolumn{4}{c|}{\textbf{FARS}} & \multicolumn{2}{c}{\textbf{Control}} \\
\textbf{Model} & RSA-C & RSA-F & Probe & Boost & RSA-F & RSA-C \\
\midrule
GPT-2 XL & .317 & .118 & 52.8 & +20.1 & .626 & $-$.003 \\
Qwen-3B & .357 & .056 & 64.4 & +11.6 & .634 & $-$.005 \\
Qwen-7B & .367 & .009 & 68.0 & +8.2 & .638 & $-$.010 \\
Mistral & .366 & .008 & 70.5 & +10.2 & .631 & $-$.003 \\
Llama-8B & .368 & .015 & 71.9 & +18.0 & .634 & $-$.007 \\
\bottomrule
\end{tabular}
\end{table}

\section{FARS Extraction Baselines}\label{app:extraction_baselines}
\paragraph{Question and scope.} Does the extracted structure exceed construction baselines? Matched readout geometry is evaluated separately; a representation score alone is not evidence of causal sufficiency.

\subsection{Proof of Lemma~\ref{lem:diffmeans}}\label{app:lemma_proof}

Let $\mu_c$ be the class-conditional mean activation of concept $c$ over a balanced stimulus set, $\bar\mu$ the grand mean, $d_c = \mu_c - \bar\mu$, and $n$ the number of concepts.

\emph{Equal spans.} $\mu_A - \mu_B = d_A - d_B$ gives $\mathrm{span}\{\mu_A - \mu_B\} \subseteq \mathrm{span}\{d_c\}$, and $d_c = \tfrac{1}{n}\sum_j (\mu_c - \mu_j)$ gives the reverse inclusion.

\emph{Equal top-$k$ subspaces.} Equality of spans alone does not fix the principal subspaces, which depend on the second moments. Summing over ordered pairs and using $\sum_c d_c = 0$,
\begin{equation*}
\sum_{A \neq B} (\mu_A - \mu_B)(\mu_A - \mu_B)^{\top}
= \sum_{A \neq B} (d_A - d_B)(d_A - d_B)^{\top}
= 2n \sum_c d_c d_c^{\top} .
\end{equation*}
The pairwise-difference scatter is exactly $2n$ times the concept-centroid scatter. Proportional PSD matrices share eigenvectors, so the two have the same ordered principal subspaces; in particular their top-$k$ subspaces are identical for every $k$. This is what makes the \emph{Diff-of-means} and \emph{FARS} rows of Table~\ref{tab:method_8way} agree to numerical precision ($d_G < 10^{-7}$ on every model), and it is why the verdict transfers to the activation-steering family by construction.

Concept-centroid PCA on per-layer centroids amplifies concept RSA $2$--$3\times$ (e.g., $0.100 \to 0.317$ for GPT-2~XL) while suppressing form RSA to near zero. Figure~\ref{fig:tsne} provides a two-dimensional illustration of the full and projected spaces; quantitative RSA and controls, rather than the embedding layout, support the representation claim. Two baselines validate the extraction (Table~\ref{tab:critical_baselines}): a \emph{shuffled-centroid null} trails FARS by $17$--$29\sigma$; and \emph{multi-class LDA} (supervised gold standard) yields concept RSA within $0.02$ of FARS, confirming that concept-centroid PCA recovers the same subspace at a fraction of the compute.

\begin{table}[t]
\centering
\caption{Concept (C) and form (F) RSA in the $10$-dim FARS subspace vs.\ (i) shuffled-centroid PCA, (ii) supervised LDA, (iii) Full-PCA. FARS exceeds the shuffled null by $17$--$29\sigma$ and matches LDA.}\label{tab:critical_baselines}
\footnotesize
\setlength{\tabcolsep}{3pt}
\begin{tabular}{@{}llcccc@{}}
\toprule
\textbf{Model} & \textbf{Metric} & \textbf{FARS} & \textbf{Shuffled} & \textbf{LDA} & \textbf{Full-PCA} \\
\midrule
\multirow{2}{*}{\raisebox{-0.18em}{\includegraphics[height=0.95em]{figures/logos/openai.pdf}}~GPT-2 XL}
 & RSA-C $\uparrow$ & \textbf{.279} & $.072 \pm .008$ & .381 & .091 \\
 & RSA-F $\downarrow$ & \textbf{.107} & .473 & $-$.011 & .486 \\
\multirow{2}{*}{\raisebox{-0.18em}{\includegraphics[height=0.95em]{figures/logos/qwen.pdf}}~Qwen-7B}
 & RSA-C $\uparrow$ & \textbf{.357} & $.128 \pm .008$ & .384 & .151 \\
 & RSA-F $\downarrow$ & \textbf{.020} & .243 & $-$.012 & .261 \\
\multirow{2}{*}{\raisebox{-0.18em}{\includegraphics[height=0.95em]{figures/logos/mistral.pdf}}~Mistral-7B}
 & RSA-C $\uparrow$ & \textbf{.362} & $.149 \pm .013$ & .380 & .194 \\
 & RSA-F $\downarrow$ & \textbf{.012} & .259 & $-$.012 & .275 \\
\bottomrule
\end{tabular}
\end{table}

\begin{figure}[t]
\centering
\includegraphics[width=\linewidth]{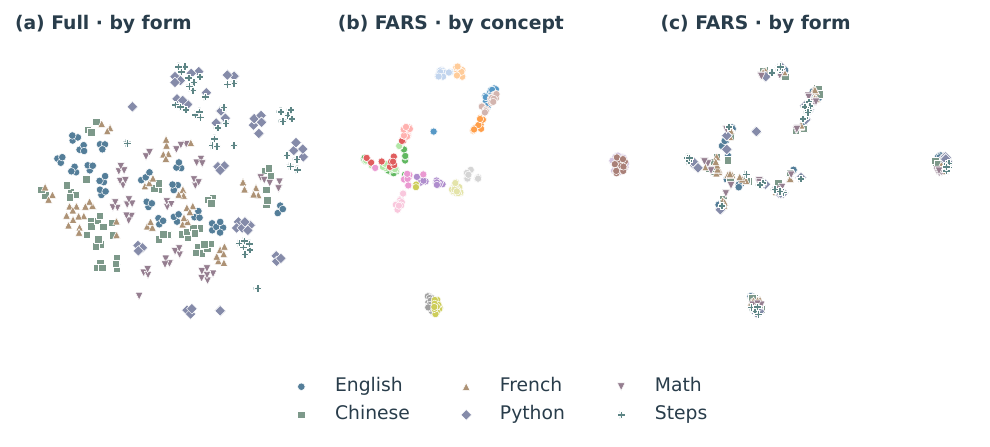}
\caption{\textbf{Full and projected representations of 324 stimuli.} Mistral-7B-v0.3, cached layer 11; t-SNE with seed 42, perplexity 30 and 2,000 iterations. (a) Full space, colored and marked by form. (b,c) The same FARS embedding colored by concept (18 colors) or marked by form. Full and FARS spaces are embedded independently; their axes and distances are not directly comparable. This visualization is descriptive, not a test of separability.}\label{fig:tsne}
\end{figure}

\section{Methodological Position: FARS vs.\ CCS, DAS, LEACE, LDA, SAEs}\label{app:method_comparison}
\paragraph{Question and scope.} How does this construction relate to other estimators? Methodological comparisons distinguish exact equivalences from empirical similarities.

\paragraph{FARS $\approx$ LDA on concept centroids.} Multi-class LDA with concepts as classes and forms/instances as within-class samples recovers nearly the same subspace as concept-centroid PCA (RSA within $0.02$ on every model, Table~\ref{tab:critical_baselines}). This empirical agreement does not make centroid PCA mathematically equivalent to LDA, whose solution also depends on within-class covariance; the novelty is the cross-format reasoning setting, the validation pipeline, and the canonical-axis/compositionality findings the construction enables.

\paragraph{vs.\ CCS / DAS / LEACE / SAE.} CCS \citep{burns2023ccs} is $1$-d unsupervised binary; FARS is $10$-d supervised multi-class. DAS \citep{geiger2024das} learns a rotation against a target causal abstraction; FARS has no target and a closed-form basis. LEACE \citep{belrose2023leace} erases the concept; FARS identifies the subspace containing it (FARS ablation and LEACE use different erasure objectives and should not be treated as equivalent). SAEs \citep{bricken2023sae} are large sparse unsupervised dictionaries, complementary: training an $18$-atom SAE within FARS on each of the $15$ models recovers $10$--$13$ atoms whose top-activating concept is unique (purity $0.17$--$0.28$ vs.\ chance $0.056$); the pattern replicates across all $4$ architecture families.

\paragraph{vs.\ function vectors, and why we report no number for them.} Function vectors \citep{todd2024function} are the natural object to compare against here: a compact, attention-head-carried representation of \emph{which task} is being demonstrated, i.e.\ an execution locus rather than a concept identity, and \citet{nadaf2026steerable} report that they steer where no logit-lens decoder reads them --- which predicts the locus verdict. We built them and ran the test, and we do not report the number, because our construction cannot support the claim. Todd et al.\ select heads by causal mediation; running per-head interventions across twenty-six models was out of scope, so we substituted a task-selectivity ranking of the same head outputs. That proxy is confounded in exactly the way that matters: the ranking promotes heads that encode the \emph{answer token} of each task, and a subspace assembled from nothing but the unembedding rows of those answer tokens scores $20.41\%$ readout energy on its own, while the selectivity ranking placed a layer-$0$ head in its top twenty. A number produced this way would report our ranking heuristic, not function vectors. We therefore leave function vectors to a causal-mediation implementation and confine our claims to the estimators in Table~\ref{tab:readout_disruption}.

\paragraph{X-FARS niche.} Eq.~\ref{eq:xfars} states a criterion for per-model orthonormal projections and a shared canonical frame across architectures with heterogeneous hidden dimensions ($D_m\in[1600, 8192]$), under an orthonormality constraint; the rectangular objective should not be conflated with square orthogonal Procrustes. The canonical frame is a portable artefact via the plug-and-play extension (Table~\ref{tab:xfars_holdout}).

\FloatBarrier
\Needspace{8\baselineskip}
\section*{2. Readout geometry and calibration}
\section{Readout Orthogonality at Matched Depth}\label{app:depth_matched}
\paragraph{Question and scope.} Does the readout contrast survive a change of reference depth? A fitted linear translator supplies the depth-matched reference; this is distinct from the weights-only raw test.

\paragraph{Per-model readout orthogonality against $W_U$.} Table~\ref{tab:readout_disruption} gives the body's per-model values (FARS at its concept-RSA peak layer against the top-$10$ right singular vectors of $W_U$).
\begin{table}[h]
\centering
\small
\setlength{\tabcolsep}{5pt}
\begin{tabular}{@{}llccccc@{}}
\toprule
Model & Family & $L/L_{\max}$ & $d_G$ & $\theta_{\min}$ & FARS energy \\
\midrule
Llama-3.1-70B-Inst & dense & 23/80 & $\mathbf{4.86}$ & $85.6^\circ$ & $\mathbf{0.16}\%$ \\
OLMo-2-7B-Inst & dense & 13/32 & $4.82$ & $84.2^\circ$ & $0.31\%$ \\
Falcon3-7B-Inst & dense & 8/27 & $4.82$ & $84.8^\circ$ & $0.31\%$ \\
Mistral-7B-v0.3 & dense & 11/31 & $4.82$ & $82.2^\circ$ & $0.36\%$ \\
Mistral-7B-Inst & dense & 10/32 & $4.81$ & $81.3^\circ$ & $0.40\%$ \\
Phi-4 & dense & 12/39 & $4.81$ & $82.3^\circ$ & $0.40\%$ \\
Yi-1.5-9B-Chat & dense & 14/47 & $4.81$ & $83.9^\circ$ & $0.35\%$ \\
Llama-3.1-8B-Inst & dense & 8/31 & $4.79$ & $81.6^\circ$ & $0.49\%$ \\
Qwen3-8B & dense & 10/35 & $4.78$ & $82.4^\circ$ & $0.51\%$ \\
Phi-3.5-mini-Inst & dense & 18/32 & $4.76$ & $81.9^\circ$ & $0.57\%$ \\
SmolLM3-3B & dense & 10/35 & $4.76$ & $81.5^\circ$ & $0.64\%$ \\
qwen3-4b & dense & 10/35 & $4.75$ & $81.7^\circ$ & $0.63\%$ \\
Qwen2.5-7B & dense & 13/27 & $4.74$ & $76.2^\circ$ & $0.97\%$ \\
Granite-3.3-8B-Inst & dense & 12/39 & $4.73$ & $75.5^\circ$ & $1.14\%$ \\
Qwen2.5-3B-Inst & dense & 22/36 & $4.72$ & $80.3^\circ$ & $0.85\%$ \\
GPT-2 XL & dense & 20/47 & $4.67$ & $76.2^\circ$ & $1.36\%$ \\
\addlinespace[2pt]
gpt-oss-20B & MoE & 7/23 & $4.84$ & $84.3^\circ$ & $0.27\%$ \\
DeepSeek-V2-Lite & MoE & 13/27 & $4.76$ & $82.0^\circ$ & $0.62\%$ \\
Mixtral-8x7B-Inst & MoE & 10/32 & $4.76$ & $74.4^\circ$ & $1.00\%$ \\
\addlinespace[2pt]
Falcon-Mamba-7B & state-space & 30/64 & $4.83$ & $84.2^\circ$ & $0.29\%$ \\
Mamba-2.8B & state-space & 35/64 & $4.77$ & $80.1^\circ$ & $0.65\%$ \\
\addlinespace[2pt]
R1-Distill-Qwen-14B & reasoning & 19/48 & $4.84$ & $85.2^\circ$ & $0.23\%$ \\
QwQ-32B & reasoning & 19/63 & $4.79$ & $77.5^\circ$ & $0.67\%$ \\
R1-Distill-Llama-8B & reasoning & 20/32 & $4.77$ & $74.4^\circ$ & $0.95\%$ \\
R1-Distill-Qwen-7B & reasoning & 13/28 & $4.60$ & $43.0^\circ$ & $5.78\%$ \\
\addlinespace[2pt]
DeBERTa-v3-large & encoder-MLM & 15/24 & $4.74$ & $80.0^\circ$ & $0.81\%$ \\
\bottomrule
\end{tabular}
\caption{FARS versus the top-$10$ readout channels of $W_U$, per model, across 26 LLMs (four architecture families plus reasoning-tuned variants). $d_G$ is the Grassmann geodesic distance (ceiling $\sqrt{10}\pi/2 = 4.97$), $\theta_{\min}$ the smallest principal angle, and the last column the fraction of FARS energy inside the top-$10$ readout span. $L/L_{\max}$ is the FARS layer as a fraction of depth; every model sits in the middle band, which matters because a near-final layer would be readout-adjacent for reasons unrelated to the concept subspace. Distances span $4.60$--$4.86$ and the minimum principal angle exceeds $74^\circ$ in 25 of 26 models. The single outlier, R1-Distill-Qwen-7B, is the model our chain-of-thought analysis independently flags (\S\ref{sec:reasoning}); its loading falls to the Haar floor once the readout is pulled back to its FARS layer (App.~\ref{app:depth_matched}).}
\label{tab:readout_disruption}
\end{table}


\paragraph{Why a depth-matched control is needed.} The body compares a layer-$\ell$
subspace with the top-$k$ right singular vectors of $W_U$. Because every layer
between $\ell$ and the unembedding transforms the residual stream, a sceptic can
read ``far from $W_U$'' as ``not at the last layer'' rather than ``not a readout
channel'': the structural controls of Table~\ref{tab:method_8way} all sit
within a few tenths of a percent of the concept estimators, and $d_G$ to $W_U$
falls monotonically with depth (Fig.~\ref{fig:layer_loc}). The control here
removes the depth confound in two steps.

\paragraph{Pulling the readout back to layer $\ell$.} On 29k--59k tokens
of generic text per model (WikiText-103, held-out $10\%$) we fit a ridge translator
$\mathbf{h}_L \approx A_\ell \mathbf{h}_\ell + \mathbf{b}$ from the block-$\ell$
residual stream to the final block's (held-out $R^2$ 0.23--0.63,
cosine 0.53--0.97), the linear tuned-lens construction of
\citet{belrose2023tunedlens}. The readout channels \emph{at layer $\ell$} are then the
top-$k$ right singular vectors of $W_U\,\mathrm{diag}(\boldsymbol\gamma)\,A_\ell$
($\boldsymbol\gamma$ the final-norm gain; mean-removal is included for LayerNorm
models). A subspace that later layers rotate into the readout will load on these
channels even though it is far from $W_U$ itself.

\paragraph{This also answers the final-norm question for the raw test.} The
primary test of \S\ref{sec:readout_control} uses the right singular vectors of
$W_U$ itself, whereas the residual stream is rescaled by the final LayerNorm or
RMSNorm before it reaches $W_U$, so the raw test's channels live in the
post-norm frame. One could object that the orthogonality we report is an
artefact of ignoring that affine step. It is not: the pulled-back test on this
page folds in exactly that step --- the diagonal gain $\boldsymbol\gamma$, and
mean-removal where the model uses LayerNorm --- and returns the same verdict,
with the concept subspace at the Haar floor and the same-layer next-token
control an order of magnitude above it. Including the normalisation makes the
test \emph{stricter}, not more permissive, and the verdict survives it.

\paragraph{A positive control at the same layer, built the same way.} FARS is
concept-centroid PCA. We build the identical estimator organised by output
identity instead: group the same layer-$\ell$ token states by the token the
model goes on to predict (top-$200$ predicted tokens, $\ge 20$ occurrences each),
average within class, and take the top-$10$ PCA of the class centroids
(\emph{next-token PCA}). It is matched in layer, rank, estimator, and corpus
with FARS; only the grouping variable differs. The encoder MLM (DeBERTa-v3) is
omitted here because its readout is a masked-token head rather than a next-token
unembedding, so the pull-back is not defined in the same way.

\begin{table}[h]
\centering
\scriptsize
\setlength{\tabcolsep}{3.5pt}
\begin{tabular}{@{}l c c ccc ccccc@{}}
\toprule
 & & & \multicolumn{3}{c}{energy \% in top-$10$ of $W_U$} & \multicolumn{5}{c}{energy \% in top-$10$ of $W_U\,\mathrm{diag}(\gamma)A_\ell$} \\
\cmidrule(lr){4-6} \cmidrule(lr){7-11}
Model & $\ell/L$ & $R^2$ & FARS & NextTok$_\ell$ & Last-PCA & FARS & LDA & FullPCA & Haar & NextTok$_\ell$ \\
\midrule
Phi-4                  & 12/39 & 0.40 & 0.40 & 1.19 & 4.29 & 0.22 & 0.25 & 0.61 & 0.26 & $\mathbf{6.88}$ \\
Llama-3.1-70B-Inst     & 22/79 & 0.37 & 0.16 & 0.72 & 1.65 & 0.19 & 0.19 & 0.45 & 0.10 & $\mathbf{5.72}$ \\
Mistral-7B-Inst        & 9/31 & 0.34 & 0.40 & 6.57 & 4.58 & 0.37 & 0.29 & 1.03 & 0.25 & $\mathbf{8.09}$ \\
Falcon-Mamba-7B        & 29/63 & 0.50 & 0.28 & 0.45 & 0.31 & 0.32 & 0.28 & 0.61 & 0.22 & $\mathbf{6.17}$ \\
Llama-3.1-8B-Inst      & 8/31 & 0.44 & 0.49 & 0.84 & 4.51 & 0.29 & 0.24 & 0.52 & 0.30 & $\mathbf{5.26}$ \\
SmolLM3-3B             & 10/35 & 0.46 & 0.64 & 1.41 & 2.91 & 0.52 & 0.49 & 1.10 & 0.43 & $\mathbf{9.67}$ \\
Mistral-7B             & 11/31 & 0.37 & 0.36 & 6.56 & 5.78 & 0.44 & 0.33 & 1.05 & 0.21 & $\mathbf{8.08}$ \\
OLMo-2-7B-Inst         & 12/31 & 0.25 & 0.31 & 0.91 & 4.71 & 0.31 & 0.29 & 0.57 & 0.24 & $\mathbf{5.37}$ \\
R1-Distill-Qwen-14B    & 18/47 & 0.38 & 0.23 & 0.60 & 0.95 & 0.22 & 0.24 & 0.48 & 0.18 & $\mathbf{3.77}$ \\
Mamba-2.8B             & 34/63 & 0.58 & 0.65 & 3.06 & 1.95 & 0.40 & 0.36 & 0.62 & 0.37 & $\mathbf{6.67}$ \\
R1-Distill-Llama-8B    & 19/31 & 0.48 & 0.95 & 1.74 & 1.92 & 0.43 & 0.23 & 0.91 & 0.21 & $\mathbf{6.58}$ \\
Granite-3.3-8B-Inst    & 12/39 & 0.23 & 1.14 & 2.79 & 2.43 & 0.27 & 0.26 & 0.71 & 0.28 & $\mathbf{4.05}$ \\
QwQ-32B                & 19/63 & 0.33 & 0.67 & 2.78 & 2.63 & 0.41 & 0.39 & 0.71 & 0.21 & $\mathbf{5.70}$ \\
Qwen2.5-7B             & 13/27 & 0.46 & 0.97 & 3.42 & 7.66 & 0.68 & 0.42 & 1.49 & 0.33 & $\mathbf{9.34}$ \\
Mixtral-8x7B-Inst      & 9/31 & 0.37 & 1.00 & 5.22 & 2.61 & 0.53 & 0.44 & 1.01 & 0.22 & $\mathbf{7.14}$ \\
Yi-1.5-9B-Chat         & 14/47 & 0.43 & 0.35 & 1.17 & 2.85 & 0.37 & 0.32 & 0.58 & 0.27 & $\mathbf{4.84}$ \\
Falcon3-7B-Inst        & 8/27 & 0.38 & 0.31 & 0.41 & 3.50 & 0.65 & 0.52 & 1.05 & 0.26 & $\mathbf{6.19}$ \\
Qwen3-8B               & 10/35 & 0.41 & 0.51 & 0.88 & 2.58 & 0.42 & 0.40 & 0.59 & 0.24 & $\mathbf{3.96}$ \\
R1-Distill-Qwen-7B     & 12/27 & 0.47 & 5.78 & 6.22 & 2.10 & 0.69 & 0.68 & 0.96 & 0.33 & $\mathbf{6.21}$ \\
qwen3-4b               & 10/35 & 0.41 & 0.63 & 1.08 & 2.36 & 0.70 & 0.56 & 0.94 & 0.33 & $\mathbf{6.05}$ \\
Phi-3.5-mini-Inst      & 17/31 & 0.61 & 0.57 & 2.38 & 6.35 & 0.49 & 0.33 & 0.87 & 0.29 & $\mathbf{3.88}$ \\
Qwen2.5-3B-Inst        & 21/35 & 0.43 & 0.81 & 1.41 & 1.74 & 0.73 & 0.63 & 1.31 & 0.59 & $\mathbf{5.82}$ \\
gpt-oss-20B            & 7/23 & 0.44 & 0.27 & 0.32 & 5.44 & 0.38 & 0.54 & 0.62 & 0.33 & $\mathbf{2.86}$ \\
GPT-2 XL               & 20/47 & 0.63 & 1.36 & 5.27 & 5.85 & 0.65 & 0.65 & 0.90 & 0.60 & $\mathbf{4.20}$ \\
DeepSeek-V2-Lite       & 12/26 & 0.43 & 0.62 & 1.17 & 5.71 & 0.76 & 0.51 & 1.13 & 0.55 & $\mathbf{4.60}$ \\
\bottomrule
\end{tabular}
\caption{Readout loading at matched depth across 25 models. Left block: the
body's test (readout $=$ top-$10$ right singular vectors of $W_U$). Right block:
the same subspaces against the readout channels pulled back to the FARS layer
$\ell$. Under the pulled-back readout the concept estimators stay at the Haar
floor (FARS 0.46\% vs.\ random 0.30\%; the Haar expectation is
$k/D$), whereas the same-layer, same-estimator next-token control loads
5.88\% --- 12.9$\times$ FARS --- and exceeds all four published
estimators in 25/25 models. $R^2$ is the held-out fit of the translator.}
\label{tab:depth_matched}
\end{table}

\paragraph{Reading.} Two things follow. First, the test discriminates
\emph{function} at fixed depth: with layer, rank, estimator, and corpus held
constant, grouping by output token yields a readout-adjacent subspace and
grouping by concept does not. Second, the pulled-back readout also repairs the
one anomaly of the body's table: R1-Distill-Qwen-7B's last-input-token FARS
loads 5.78\% on the raw $W_U$
channels but only 0.69\% on
the layer-$\ell$ channels, i.e.\ its apparent readout-adjacency was a depth
artefact of comparing against $W_U$ directly. FARS's residual enrichment over
the Haar floor under the raw test (0.79\% vs.\ $k/D$) likewise shrinks
to 0.46\% once depth is matched.


\begin{figure}[h]
\centering
\includegraphics[width=\linewidth]{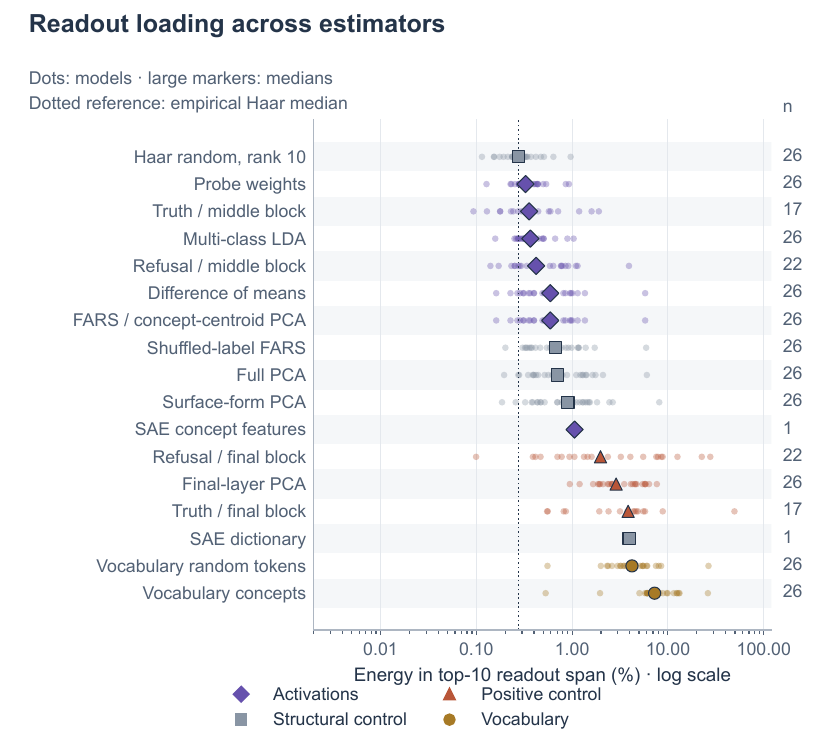}
\caption{\textbf{Readout loading across subspace constructions.} Each small dot is one model and the larger role-specific marker is the median. The statistic is energy in the top-$10$ right singular span of $W_U$, on a logarithmic axis; the dotted reference is the median empirical Haar control. Sample sizes appear at right. Symbols distinguish activation-derived estimators, structural controls, positive controls, and vocabulary-derived constructions. Pulled-back readout results use a different reference and appear separately in App.~\ref{app:depth_matched}. Published directions use the same paired cohorts as Fig.~\ref{fig:published_depth}: any available AUC below $0.9$ excludes a pair; missing final-block AUCs remain flagged there. The pooled Haar median is a visual reference, not a model-specific significance threshold.}
\label{fig:verdict_landscape}
\end{figure}
\section{Readout-Rank Sensitivity}\label{app:rank_sweep}
\paragraph{Question and scope.} How sensitive is the result to readout rank? The reported comparisons cover the stated sampled ranks, not every possible decoder. At full ambient rank, energy is one for any normalized basis.


\paragraph{Is the concept subspace read out at \emph{any} rank?}
Fixing $k=10$ invites the objection that concept identity might be projected
through lower-variance singular directions of $W_U$ --- say directions
$50$--$200$ --- which a top-$10$ comparison would not see. We examine this by
computing the \emph{full} right-singular basis $V$ of $W_U$ and, for every rank
$k$, the fraction of a subspace's energy inside $\mathrm{span}(V_{:,1:k})$,
\[
  E_S(k) \;=\; \tfrac{1}{m}\textstyle\sum_{i=1}^{m} \big\| V_{:,1:k}^{\top} s_i \big\|^2 ,
  \qquad S = \{s_1,\dots,s_m\} \text{ orthonormal}.
\]
For a Haar-random rank-$m$ subspace $\mathbb{E}[E_S(k)] = k/D$ exactly, so
$E_S(k)/(k/D)$ is an enrichment over chance with a known null of $1$. We report
FARS against the final-layer PCA (the actual readout pathway) and a Haar control.

\begin{table}[t]
\centering
\small
\setlength{\tabcolsep}{5pt}
\begin{tabular}{@{}lc ccc c@{}}
\toprule
 & & \multicolumn{3}{c}{enrichment $E(k)/(k/D)$, FARS / final-layer} & Haar \\
\cmidrule(lr){3-5}
Model & $D$ & $k{=}10$ & $k{=}32$ & $k{=}100$ & $k{=}10$ \\
\midrule
DeBERTa-v3-large & 1024 & 0.83 / 1.2 & 0.92 / 1.1 & 0.93 / 1.1 & 1.03 \\
DeepSeek-V2-Lite & 2048 & 1.26 / 11.7 & 1.12 / 6.3 & 1.06 / 3.2 & 1.00 \\
Falcon-Mamba-7B & 4096 & 1.17 / 16.8 & 1.36 / 6.8 & 1.23 / 3.4 & 1.01 \\
Falcon3-7B-Inst & 3072 & 0.96 / 10.7 & 0.95 / 5.8 & 1.00 / 2.9 & 0.98 \\
gpt-oss-20B & 2880 & 0.77 / 15.7 & 0.88 / 10.3 & 0.83 / 5.4 & 1.00 \\
GPT-2 XL & 1600 & 2.18 / 9.4 & 1.61 / 5.4 & 1.26 / 2.8 & 1.01 \\
Granite-3.3-8B-Inst & 4096 & 4.68 / 9.9 & 2.44 / 7.0 & 1.57 / 4.4 & 0.98 \\
Llama-3.1-70B-Inst & 8192 & 1.32 / 13.6 & 1.18 / 7.9 & 1.04 / 3.7 & 1.02 \\
Llama-3.1-8B-Inst & 4096 & 2.01 / 18.5 & 1.58 / 11.9 & 1.34 / 5.0 & 1.02 \\
Mamba-2.8B & 2560 & 1.66 / 5.0 & 1.34 / 2.3 & 1.05 / 0.9 & 1.00 \\
Mistral-7B-Inst & 4096 & 1.64 / 18.7 & 1.31 / 9.3 & 1.13 / 4.8 & 1.01 \\
Mistral-7B & 4096 & 1.49 / 23.7 & 1.22 / 12.1 & 1.11 / 6.1 & 1.01 \\
Mixtral-8x7B-Inst & 4096 & 4.10 / 10.7 & 2.47 / 7.5 & 1.82 / 3.5 & 0.99 \\
OLMo-2-7B-Inst & 4096 & 1.26 / 19.3 & 1.25 / 9.9 & 1.25 / 4.7 & 0.98 \\
Phi-3.5-mini-Inst & 3072 & 1.74 / 19.5 & 1.90 / 9.8 & 1.40 / 4.6 & 1.05 \\
Phi-4 & 5120 & 2.04 / 21.9 & 1.55 / 15.7 & 1.15 / 8.2 & 1.02 \\
Qwen2.5-3B-Inst & 2048 & 1.67 / 3.6 & 1.45 / 2.4 & 1.29 / 1.7 & 1.03 \\
Qwen2.5-7B & 3584 & 3.48 / 27.4 & 2.09 / 10.4 & 1.53 / 3.9 & 0.99 \\
Qwen3-4B & 2560 & 1.62 / 6.0 & 1.33 / 4.1 & 1.13 / 2.7 & 0.97 \\
Qwen3-8B & 4096 & 2.08 / 10.6 & 1.71 / 9.5 & 1.35 / 5.1 & 1.00 \\
QwQ-32B & 5120 & 3.41 / 13.5 & 2.05 / 6.8 & 1.41 / 2.9 & 1.00 \\
R1-Distill-Llama-8B & 4096 & 3.89 / 7.9 & 2.44 / 6.1 & 2.00 / 2.8 & 1.01 \\
R1-Distill-Qwen-14B & 5120 & 1.16 / 4.8 & 1.37 / 3.7 & 1.20 / 2.0 & 0.98 \\
R1-Distill-Qwen-7B & 3584 & 20.73 / 7.5 & 7.46 / 3.5 & 3.08 / 1.5 & 1.00 \\
SmolLM3-3B & 2048 & 1.31 / 6.0 & 1.24 / 4.5 & 1.29 / 2.4 & 1.02 \\
Yi-1.5-9B-Chat & 4096 & 1.44 / 11.7 & 1.12 / 8.3 & 1.03 / 4.8 & 0.99 \\
\bottomrule
\end{tabular}
\caption{Readout enrichment at three ranks for the concept subspace (FARS) and the readout pathway (top-$10$ PCA of the final layer), across 26 models. The Haar column is the null check: a random rank-$10$ subspace has expected enrichment $1.00$; the sampled controls are close (max deviation $0.046$). Enrichment generally approaches the Haar expectation as rank grows; the magnitude and ordering vary across models. The same tabulated comparisons are visualized in Fig.~\ref{fig:rank_sweep}.}
\label{tab:rank_sweep}
\end{table}

The reported results show separation at the reported ranks, with exceptions. Final-layer PCA is more enriched than FARS in $25/26$ models at $k=10$ and $k=32$, and in $24/26$ at $k=100$. R1-Distill-Qwen-7B reverses the comparison at all three ranks; Mamba-2.8B also reverses it at $k=100$ ($1.05$ versus $0.90$, both near the null). The results do not establish zero readout information at any rank. At $k=D$, both energy and enrichment equal $1$ for every subspace, so a separation cannot hold across the entire spectrum. The reported table measures concentration at three ranks, with the Haar expectation providing the reference scale; the comparison is limited to these reported ranks.

\begin{figure}[t]
\centering
\includegraphics[width=\linewidth]{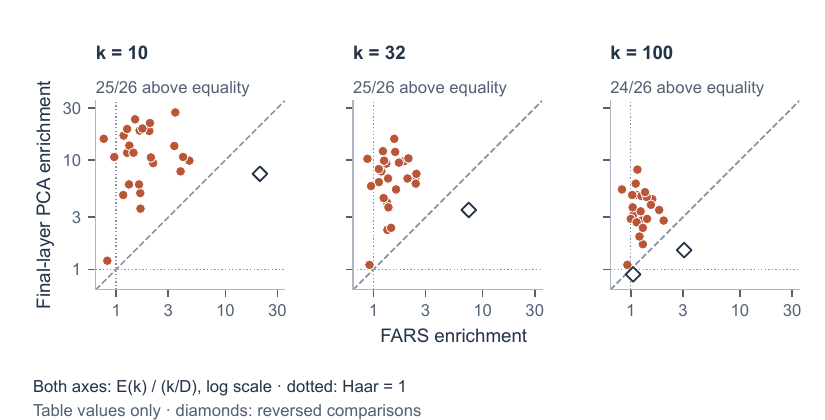}
\caption{\textbf{Paired readout enrichment at three reported ranks.} Each point is one of $26$ models, plotted from the rounded values in Table~\ref{tab:rank_sweep}. Both axes use the same logarithmic scale. The dashed diagonal marks equal enrichment; dotted lines mark the Haar expectation of $1$. Diamonds identify models where final-layer PCA is below FARS: R1-Distill-Qwen-7B at all three ranks, and Mamba-2.8B additionally at $k=100$. No error bars or full-spectrum trajectories are inferred from these table values.}
\label{fig:rank_sweep}
\end{figure}

\section{Published Concept Directions Under the Test}\label{app:published}
\paragraph{Question and scope.} How do externally motivated directions behave under the diagnostic? Readout comparisons depend on separability and layer matching; excluded and missing-AUC cases must remain visible.


\paragraph{Setup.} The body's verdict covers subspaces we extract. To test
directions the literature actually publishes, we rebuild two of them on each
model, at the same block as FARS and from the last input token: the
\emph{refusal direction} of \citet{arditi2024refusal} (mean of 200
AdvBench harmful instructions minus mean of as many Alpaca harmless ones,
chat-templated) and the \emph{truth direction} of \citet{marks2024geometry}
(mean of true minus mean of false city--country statements). Each is a unit
vector; we report its energy $\|R\mathbf{v}\|^2$ in the top-$10$ span of $W_U$
and in the readout channels pulled back to the block (App.~\ref{app:depth_matched}),
whose Haar expectation is $k/D$. A held-out AUC of the one-dimensional
projection checks that the direction does separate its classes; entries with
AUC $<0.9$ are parenthesised and excluded from the summary. The same
directions built at the \emph{final} block (last two columns) show what a
readout-adjacent version of each looks like.

\begin{table}[h]
\centering
\scriptsize
\setlength{\tabcolsep}{3pt}
\begin{tabular}{@{}l c cc c cc ccc cc@{}}
\toprule
 & & \multicolumn{2}{c}{AUC} & Haar & \multicolumn{2}{c}{vs.\ $W_U$} & \multicolumn{3}{c}{vs.\ $W_U\!\to\!\ell$} & \multicolumn{2}{c}{final block, vs.\ $W_U$} \\
\cmidrule(lr){3-4} \cmidrule(lr){6-7} \cmidrule(lr){8-10} \cmidrule(lr){11-12}
Model & $\ell$ & ref. & truth & $k/D$ & refusal & truth & refusal & truth & NextTok$_1$ & refusal & truth \\
\midrule
DeepSeek-V2-Lite       & 13 & 1.00 & 1.00 & 0.49 & 0.59 & 1.42 & 1.04 & 1.01 & 1.48 & 7.55 & 8.37 \\
Falcon-Mamba-7B        & 29 & 1.00 & 1.00 & 0.24 & 0.40 & 0.36 & 2.12 & 0.11 & 2.81 & 1.06 & 5.77 \\
Falcon3-7B-Inst        & 8 & 1.00 & 1.00 & 0.33 & 0.41 & 0.59 & 1.06 & 0.18 & 0.02 & 1.33 & 5.54 \\
gpt-oss-20B            & 7 & 1.00 & 1.00 & 0.35 & 0.29 & 0.24 & 0.51 & 0.15 & 2.61 & 4.05 & 3.84 \\
GPT-2 XL               & 20 & 1.00 & 1.00 & 0.62 & 3.92 & 0.33 & 4.28 & 0.28 & 0.30 & 7.86 & 1.44 \\
Granite-3.3-8B-Inst    & 12 & 1.00 & 1.00 & 0.24 & 0.43 & 1.90 & 0.19 & 0.16 & 0.05 & 0.41 & 0.82 \\
Llama-3.1-8B-Inst      & 8 & 1.00 & 1.00 & 0.24 & 0.27 & 1.18 & 0.30 & 0.22 & 4.80 & 3.22 & 4.86 \\
Mamba-2.8B             & 34 & 1.00 & 1.00 & 0.39 & 1.10 & 0.36 & 1.13 & 0.20 & 6.45 & 12.57 & 8.82 \\
Mistral-7B-Inst        & 9 & 1.00 & 1.00 & 0.24 & 0.77 & 0.18 & 0.42 & 0.14 & 0.01 & 0.78 & 3.16 \\
Mistral-7B             & 11 & 1.00 & 1.00 & 0.24 & 0.85 & 0.13 & 2.73 & 0.12 & 0.01 & 22.54 & 2.40 \\
OLMo-2-7B-Inst         & 12 & 1.00 & 1.00 & 0.24 & 0.17 & 0.09 & 0.47 & 0.10 & 5.05 & 0.94 & 4.50 \\
Phi-3.5-mini-Inst      & 17 & 1.00 & 1.00 & 0.33 & 0.25 & 0.30 & 0.22 & 0.11 & 7.28 & 27.65 & 49.43 \\
Phi-4                  & 12 & 1.00 & 0.50 & 0.20 & 0.28 & (0.76)$^\dagger$ & 0.20 & (0.12)$^\dagger$ & 0.04 & 1.45 & (1.92)$^\dagger$ \\
Qwen2.5-3B-Inst        & 21 & 1.00 & 1.00 & 0.49 & 1.14 & 0.71 & 0.24 & 0.52 & 8.22 & 1.60 & 4.35 \\
Qwen2.5-7B             & 13 & 0.99 & 1.00 & 0.28 & 0.34 & 0.88 & 0.58 & 0.29 & 0.03 & 5.52 & 22.47 \\
Qwen3-4B               & 10 & 1.00 & 0.98 & 0.39 & 0.91 & 0.27 & 0.40 & 0.40 & 6.45 & 0.70 & 1.92 \\
Qwen3-8B               & 10 & 1.00 & 0.99 & 0.24 & 0.23 & 0.23 & 0.23 & 0.20 & 3.98 & 0.47 & 0.86 \\
R1-Distill-Llama-8B    & 19 & 1.00 & 1.00 & 0.24 & 0.79 & 1.60 & 0.16 & 0.75 & 7.66 & 8.53 & 3.77 \\
r1-distill-qwen-14b    & 19 & 1.00 & 1.00 & 0.20 & 0.25 & 0.18 & 0.64 & 0.19 & 1.17 & 0.39 & 0.56 \\
R1-Distill-Qwen-7B     & 12 & 1.00 & 0.52 & 0.28 & 0.76 & (27.52)$^\dagger$ & 1.22 & (0.28)$^\dagger$ & 0.06 & 8.81 & (14.18)$^\dagger$ \\
SmolLM3-3B             & 10 & 1.00 & 1.00 & 0.49 & 0.64 & 0.57 & 0.38 & 0.44 & 12.13 & 0.10 & 4.67 \\
Yi-1.5-9B-Chat         & 14 & 1.00 & 1.00 & 0.24 & 0.14 & 0.44 & 0.17 & 0.18 & 13.36 & 2.36 & 0.55 \\
\bottomrule
\end{tabular}
\caption{Published concept directions under the readout test (22 models; energies in \%).
At the FARS block both the refusal and the truth direction sit at the Haar
floor against $W_U$ (medians 0.42\% and 0.36\%, $k/D \approx 0.26\%$)
and against the pulled-back readout (0.44\% and 0.20\%), while the
first next-token direction at the same block loads 2.71\%. The same two
directions built at the final block load 1.98\% and 4.09\%
(medians): the middle-layer construction has low linear readout loading and becomes readout-adjacent only
where the residual stream is about to be decoded. $^\dagger$AUC $<0.9$.}
\label{tab:published}
\end{table}

\paragraph{Reading.} Two published, behaviourally validated directions --- one
safety-relevant, one truthfulness-relevant --- receive the same verdict as the
reasoning-concept estimators: low linear readout loading, in
16/22 (refusal) and 19/20 (truth) models under the pulled-back
readout (criterion: energy $< 3k/D$). Their final-block counterparts are
readout-adjacent in 17/22 and 17/20, so the test
is not returning a constant for difference-of-means directions either.

\begin{figure}[h]
\centering
\includegraphics[width=\linewidth]{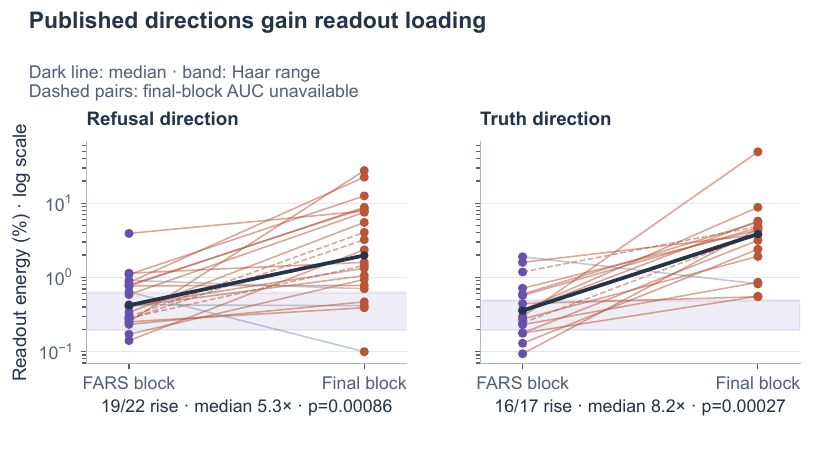}
\caption{\textbf{The same published direction changes verdict with extraction
depth.} Each line is one model: the energy its refusal (left) or truth (right)
direction carries in the top-$10$ span of $W_U$, read off the FARS block versus
the final block. Estimator, concept, stimuli and model are identical along each
line; only the depth changes. Band: the range of per-model Haar expectations $k/D$; dark line: population median. Grey lines fall. Directions with available held-out AUC $<0.9$ at either block are omitted. Dashed pairs lack a final-block AUC (three refusal, two truth) and are retained with this qualification.}
\label{fig:published_depth}
\end{figure}

This is the cleanest available demonstration that the test discriminates,
because the paired comparison holds everything fixed except depth: same
estimator, same concept, same model, same stimuli (Fig.~\ref{fig:published_depth}).
Readout energy rises from
the FARS block to the final block in 19/22 models for refusal
(median 5.3$\times$, sign test $p = 0.0009$) and
16/17 for truth (median 8.2$\times$,
$p = 0.0003$). Final-block AUC is available for $19/22$ models. Refusal AUC remains at least $0.98$ throughout that subset. Four truth directions fall below $0.9$: DeepSeek-V2-Lite ($0.8988$), GPT-2 XL ($0.6828$), Qwen2.5-7B ($0.8692$), and R1-Distill-Qwen-7B ($0.5948$). The last also fails at the FARS block, as does Phi-4, whose final-block AUC is unavailable. Filtering on all available AUC values leaves the $22$ refusal and $17$ truth pairs shown, including respectively three and two pairs without final-block AUC; these missing checks limit the classifier-quality control. A direction is not readout-aligned or not by virtue of
which concept it encodes; it is readout-aligned to the extent that it is read
off the stream where the stream is about to be decoded. The test therefore
transfers, without modification, to the directions steering and safety work
actually manipulate.

\subsection{A vocabulary-defined construction, which the test flags as readout}\label{app:vocab_anchored}

\paragraph{A readout-side construction from the same literature.} The estimators of
Table~\ref{tab:method_8way} are all built from \emph{activations}, and all land on the
locus side. A diagnostic whose verdict never changes across the constructions people use
is hard to act on, so we add the other kind of ``concept direction'' the literature also
writes down --- one defined in the model's \emph{vocabulary}. For concept $c$ we take the
tokens diagnostic of its stimuli (tf-idf over $c$'s six surface forms against the other
concepts, top $40$), average the corresponding rows of $W_U$, and take the top-$k$ PCA of
the $18$ resulting centroids. This is the construction behind token-anchored steering and
logit-attribution readings of ``the direction of $c$''. It needs no forward pass: it is a
function of $W_U$ and the tokenizer alone.

\begin{table}[h]
\centering
\scriptsize
\setlength{\tabcolsep}{4pt}
\begin{tabular}{@{}l c cccc c@{}}
\toprule
 & & Haar & \multicolumn{3}{c}{energy \% in top-$10$ readout span} & \\
\cmidrule(lr){4-6}
Model & $D$ & $k/D$ & FARS & random-token & vocabulary-anchored & ratio \\
\midrule
Llama-3.1-70B-Inst     & 8192 & 0.12 & 0.16 & 2.37 & $\mathbf{12.47}$ & 77.3$\times$ \\
SmolLM3-3B             & 2048 & 0.49 & 0.64 & 26.49 & $\mathbf{26.11}$ & 40.7$\times$ \\
R1-Distill-Qwen-14B    & 5120 & 0.20 & 0.23 & 1.99 & $\mathbf{6.12}$ & 27.0$\times$ \\
Llama-3.1-8B-Inst      & 4096 & 0.24 & 0.49 & 4.27 & $\mathbf{12.94}$ & 26.4$\times$ \\
gpt-oss-20B            & 2880 & 0.35 & 0.27 & 5.99 & $\mathbf{6.95}$ & 25.9$\times$ \\
OLMo-2-7B-Inst         & 4096 & 0.24 & 0.31 & 4.60 & $\mathbf{6.47}$ & 21.0$\times$ \\
Phi-4                  & 5120 & 0.20 & 0.40 & 3.82 & $\mathbf{8.17}$ & 20.5$\times$ \\
Mistral-7B             & 4096 & 0.24 & 0.36 & 5.45 & $\mathbf{7.04}$ & 19.3$\times$ \\
Mistral-7B-Inst        & 4096 & 0.24 & 0.40 & 5.42 & $\mathbf{7.03}$ & 17.6$\times$ \\
Phi-3.5-mini-Inst      & 3072 & 0.33 & 0.57 & 8.01 & $\mathbf{9.70}$ & 17.2$\times$ \\
Yi-1.5-9B-Chat         & 4096 & 0.24 & 0.35 & 2.61 & $\mathbf{6.03}$ & 17.1$\times$ \\
Qwen3-8B               & 4096 & 0.24 & 0.51 & 2.32 & $\mathbf{7.91}$ & 15.6$\times$ \\
R1-Distill-Llama-8B    & 4096 & 0.24 & 0.95 & 4.33 & $\mathbf{13.20}$ & 13.9$\times$ \\
QwQ-32B                & 5120 & 0.20 & 0.67 & 3.22 & $\mathbf{8.83}$ & 13.3$\times$ \\
Qwen2.5-7B             & 3584 & 0.28 & 0.97 & 6.08 & $\mathbf{12.32}$ & 12.7$\times$ \\
Qwen3-4B               & 2560 & 0.39 & 0.63 & 2.98 & $\mathbf{6.12}$ & 9.7$\times$ \\
Mamba-2.8B             & 2560 & 0.39 & 0.65 & 5.14 & $\mathbf{6.08}$ & 9.4$\times$ \\
DeBERTa-v3-large       & 1024 & 0.98 & 0.81 & 7.50 & $\mathbf{7.47}$ & 9.2$\times$ \\
Qwen2.5-3B-Inst        & 2048 & 0.49 & 0.85 & 3.43 & $\mathbf{7.57}$ & 8.9$\times$ \\
GPT-2 XL               & 1600 & 0.62 & 1.36 & 8.49 & $\mathbf{11.51}$ & 8.5$\times$ \\
DeepSeek-V2-Lite       & 2048 & 0.49 & 0.62 & 4.15 & $\mathbf{5.04}$ & 8.2$\times$ \\
Falcon3-7B-Inst        & 3072 & 0.33 & 0.31 & 3.18 & $\mathbf{1.96}$ & 6.3$\times$ \\
Mixtral-8x7B-Inst      & 4096 & 0.24 & 1.00 & 3.71 & $\mathbf{6.00}$ & 6.0$\times$ \\
Granite-3.3-8B-Inst    & 4096 & 0.24 & 1.14 & 3.51 & $\mathbf{5.75}$ & 5.0$\times$ \\
Falcon-Mamba-7B        & 4096 & 0.24 & 0.29 & 0.55 & $\mathbf{0.53}$ & 1.9$\times$ \\
R1-Distill-Qwen-7B     & 3584 & 0.28 & 5.78 & 5.56 & $\mathbf{9.95}$ & 1.7$\times$ \\
\bottomrule
\end{tabular}
\caption{Vocabulary-anchored concept directions against the readout basis, 26 models.
The construction lies in the row space of $W_U$ by design, so some loading is guaranteed;
the honest baseline is therefore the \emph{random-token} column, which repeats the identical
construction with random token sets of the same size ($20$ draws). Concept-diagnostic tokens
load 8.43\% on average against 5.20\% for random tokens and
0.80\% for FARS at its own layer --- a 16.9$\times$ separation from FARS,
in the same direction in 26/26 models, with per-model ratios spanning
1.7$\times$--77.3$\times$. The two subspaces are not the same object:
their mutual energy overlap is 0.56\%.}
\label{tab:vocab_anchored}
\end{table}

\paragraph{Reading.} The split the test reports is not between our estimators and other
people's. It is between \emph{where the direction is defined}: a concept direction read off
the activations has low readout loading, while the vocabulary-derived construction has greater loading. Both are called ``the direction of concept $c$'' in published work, and an
ablation of either will disturb the task; only the second is disturbed \emph{because it is
the channel the unembedding reads}, which is the inference the test exists to block.

Two qualifications keep this honest. First, most of the effect comes from the
\emph{anchoring} rather than from the concept selection: the random-token control, which
repeats the construction with arbitrary tokens, reaches 5.20\% against
8.43\% for concept-diagnostic tokens and exceeds it outright in
4 of 26 models. What makes these directions readout-adjacent is that they
are assembled from rows of $W_U$ at all, not which rows. Second, the effect size varies widely
(1.7$\times$--77.3$\times$). The narrowest ratio is
R1-Distill-Qwen-7B, where it is the \emph{numerator's} denominator that is unusual --- that
model's FARS is the one already flagged in Table~\ref{tab:readout_disruption}. More
informative is Falcon-Mamba-7B, whose vocabulary-anchored subspace sits at
0.53\% against a Haar floor of 0.24\%: there the construction is
\emph{not} readout-adjacent and the test says so. The verdict is a per-model measurement, not
a law.

\subsection{Sparse-autoencoder feature directions}\label{app:sae}

\paragraph{Setup.} Sparse autoencoders \citep{bricken2023sae,templeton2024scaling} are the
estimator the field now reaches for first, so the test has to say something about them. We
build the SAE analogue of FARS on the same stimuli, at the SAE's own layer: encode the cached
layer-19 residual states with the pretrained residual-stream SAE of
\texttt{Goodfire/Llama-3.1-8B-Instruct-SAE-l19} (65,536 features), score every feature by how selectively it
fires on each concept (mean activation on $c$ minus mean on the rest, in pooled s.d.\ units),
take each concept's top-$16$ features, sum their \emph{decoder} directions weighted by
selectivity, and take the top-$10$ PCA of the $18$ concept directions. The recipe is
concept-centroid PCA with the representation swapped for SAE feature space. Our encoding
reproduces the released sparsity (mean $L_0 = 98$ against the model card's $91$;
fraction of variance unexplained 0.59 on this out-of-distribution stimulus set), so the
comparison is against a correctly instantiated SAE rather than an assumed one.

\begin{table}[h]
\centering
\small
\begin{tabular}{@{}lccc@{}}
\toprule
Rank-$10$ subspace & $d_G$ & energy \% & $\theta_{\min}$ \\
\midrule
SAE dictionary (top-$10$ PCA of live decoder directions) & 4.590 & 3.93 & 56.3 \\
\textbf{SAE concept features} (selectivity-weighted, rank $10$) & 4.748 & 1.06 & 73.8 \\
FARS at the same layer & 4.759 & 0.96 & 75.1 \\
SAE random features (control) & 4.749 & 0.88 & 76.5 \\
\midrule
Haar expectation $k/D$ & --- & 0.24 & --- \\
\bottomrule
\end{tabular}
\caption{SAE feature directions against the top-$10$ readout channels of $W_U$ on
Llama-3.1-8B-Inst. The SAE's concept-selective features have low linear readout loading at
1.06\%, indistinguishable from FARS at the same layer (0.96\%) and from
randomly chosen features (0.88\%). The dictionary as a whole is different: the
top-$10$ PCA of all live decoder directions sits at 3.93\%, roughly
4$\times$ its own concept features --- a dictionary trained to
reconstruct the residual stream inherits whatever the readout writes into it, while the features that
carry concept identity do not.}
\label{tab:sae}
\end{table}

\paragraph{Reading.} The verdict on SAE concept features is the same as on concept-centroid
PCA, LDA, difference-of-means and probe weights: low linear readout loading. The two subspaces are
related but not identical (mutual energy overlap
26\%), so this is a genuine second estimator agreeing, not a
restatement of FARS. \emph{Scope:} one model. Of the 26 models in Table~\ref{tab:readout_disruption}
this is the only one with a published residual-stream SAE we could match exactly in both
checkpoint and layer, and we do not generalise beyond it.

\section{Layer-wise Localization of Concept Structure}\label{app:correlational_localization}
\paragraph{Question and scope.} Where is concept structure most recoverable? Peak-layer selection favors peak retrieval by construction; these curves do not independently establish a computation site.

Three independent analyses converge on middle layers as the site of format-agnostic representation before the body-level mechanistic tests are applied (Table~\ref{tab:main_results}, Figure~\ref{fig:main_results}): concept RSA peaks in middle layers and strengthens with scale ($\rho_{\max}$: $0.100 \to 0.200$; FDR-corrected); a ridge probe trained on English representations identifies the same concept in code or math at $60.3\%$ for Mistral-7B (chance $5.6\%$), jumping between GPT-2 XL ($32.7\%$) and Qwen2.5-3B ($52.8\%$); and format-agnostic neurons peak at $20$--$23\%$ in middle layers, with format-specific neurons dominating early (tokenisation) and late (generation) layers.

\begin{table}[t]
\centering
\caption{Localization of format-agnostic computation. RSA-C/F: peak concept/form RSA; Probe-W: within-form probe ceiling; Probe-X: cross-form probe (chance $5.6\%$); X/W: cross-/within-form ratio. Cross-form probing recovers $34$--$61\%$ of the ceiling, jumping between $1.6$B ($0.34$) and $3$B ($0.54$).}\label{tab:main_results}
\scriptsize
\setlength{\tabcolsep}{3pt}
\begin{tabular}{@{}lcccccc@{}}
\toprule
\textbf{Model} & \textbf{RSA-C} & \textbf{RSA-F} & \textbf{Probe-W\%} & \textbf{Probe-X\%} & \textbf{X/W} & \textbf{Agn\%} \\
\midrule
\raisebox{-0.18em}{\includegraphics[height=0.95em]{figures/logos/openai.pdf}}~GPT-2 XL & .100 & .430 & 96.6 & 32.7 & .34 & 20.0 \\
\raisebox{-0.18em}{\includegraphics[height=0.95em]{figures/logos/qwen.pdf}}~Qwen2.5-3B & .126 & .501 & 98.5 & 52.8 & .54 & 23.3 \\
\raisebox{-0.18em}{\includegraphics[height=0.95em]{figures/logos/qwen.pdf}}~Qwen2.5-7B & .150 & .485 & 99.1 & 59.8 & .60 & 19.5 \\
\raisebox{-0.18em}{\includegraphics[height=0.95em]{figures/logos/mistral.pdf}}~Mistral-7B & .200 & .529 & 98.5 & 60.3 & .61 & 20.8 \\
\raisebox{-0.18em}{\includegraphics[height=0.95em]{figures/logos/meta.pdf}}~Llama-3.1-8B & .165 & .525 & 99.4 & 53.9 & .54 & 21.1 \\
\bottomrule
\end{tabular}
\end{table}

\begin{figure}[t]
\centering
\includegraphics[width=\linewidth]{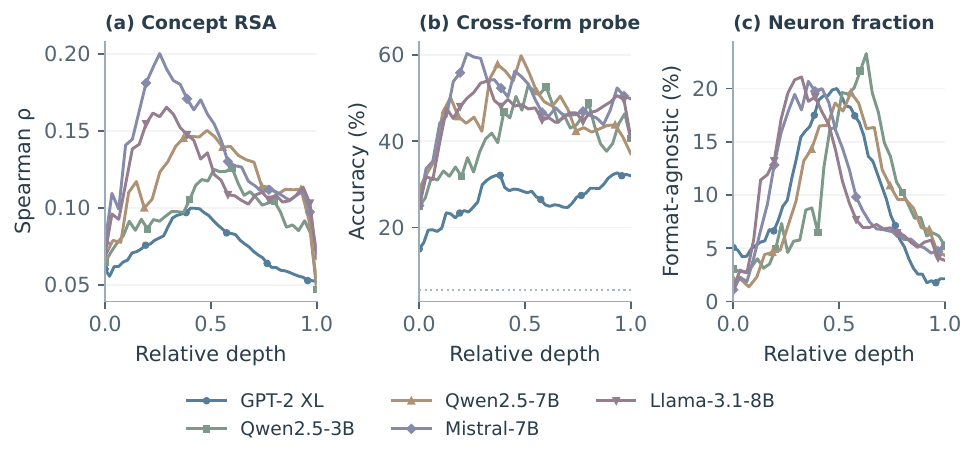}
\caption{\textbf{Layerwise representation metrics in the historical five-model baseline.} (a) Concept RSA. (b) Cross-form probe accuracy; the dotted line marks $1/18$ chance. (c) Format-agnostic neuron fraction. Depth is normalized within each model. Curves reproduce the saved per-layer results; colors and markers jointly identify models.}\label{fig:main_results}
\end{figure}

\paragraph{The prose-code boundary.} Full-activation patching across form pairs reveals a sharp asymmetry: EN$\to$Math top-$10$ overlap ($0.65$--$0.73$) is $3$--$4\times$ higher than EN$\to$Code ($0.16$--$0.22$) across five tested models, with every $95\%$ bootstrap CI above $2.4\times$. Two controls localise the cause: tokeniser Jaccard is near-zero-predictive ($r{=}0.03$, n.s.), and a declarative-Python control isolates the procedural-vs-declarative surface as accounting for $\sim 11\%$ of the gap. Full details and per-model numbers in App.~\ref{app:pydecl}.


\begin{figure}[h]
\centering
\includegraphics[width=\linewidth]{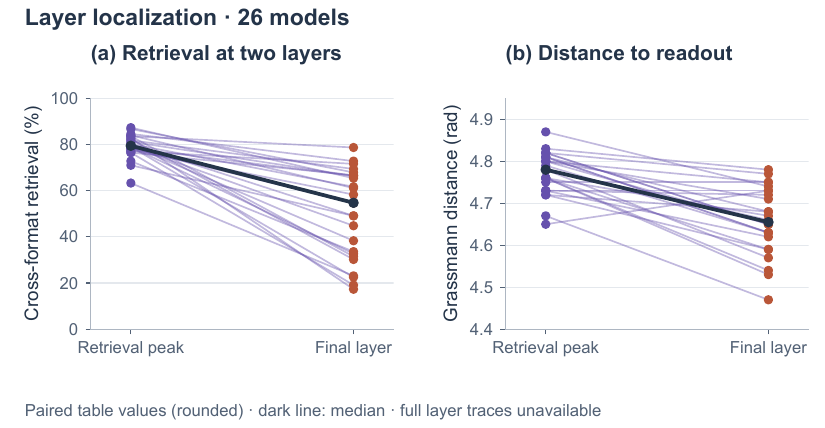}
\caption{\textbf{Retrieval peak versus final layer.} Each thin line joins two measurements from one model; the dark line joins population medians. (a) Cross-format retrieval. (b) Distance to the top-$10$ readout span. These plots use the reported rounded values in Table~\ref{tab:layer_loc}, not reconstructed layer-by-layer trajectories. Peak selection favours higher retrieval by construction.}
\label{fig:layer_loc}
\end{figure}

\paragraph{Per-layer localization.} Fig.~\ref{fig:layer_loc} compares retrieval and readout distance at the retrieval peak and final layer for each model; Table~\ref{tab:layer_loc} gives the per-model numbers.

\begin{table}[h]
\centering
\small
\begin{tabular}{@{}l cc cc cc@{}}
\toprule
Model & $L_\text{peak}/L$ & frac. & Agn\% peak & $d_G$ peak & Agn\% last & $d_G$ last \\
\midrule
Llama-3.1-8B-Inst      &   9/ 32 & 0.29 &  81.2\% & 4.79 &  66.4\% & 4.63 \\
Mistral-7B             &   9/ 32 & 0.29 &  81.8\% & 4.81 &  69.4\% & 4.63 \\
Llama-3.1-70B-Inst     &  24/ 81 & 0.30 &  87.3\% & 4.87 &  65.4\% & 4.74 \\
Mistral-7B-Inst        &  11/ 33 & 0.34 &  84.0\% & 4.81 &  61.1\% & 4.63 \\
Falcon3-7B-Inst        &  10/ 28 & 0.37 &  79.0\% & 4.82 &  32.7\% & 4.65 \\
Mixtral-8x7B-Inst      &  12/ 33 & 0.38 &  82.4\% & 4.73 &  30.2\% & 4.73 \\
Yi-1.5-9B-Chat         &  20/ 48 & 0.43 &  83.6\% & 4.82 &  31.5\% & 4.65 \\
Granite-3.3-8B-Inst    &  17/ 40 & 0.44 &  77.2\% & 4.72 &  71.6\% & 4.68 \\
R1-Distill-Llama-8B    &  14/ 33 & 0.44 &  83.6\% & 4.80 &  78.7\% & 4.75 \\
Qwen2.5-7B             &  12/ 28 & 0.44 &  78.4\% & 4.76 &  54.6\% & 4.59 \\
QwQ-32B                &  28/ 64 & 0.44 &  80.2\% & 4.81 &  49.1\% & 4.68 \\
GPT-2 XL               &  22/ 48 & 0.47 &  63.3\% & 4.67 &  23.1\% & 4.47 \\
DeepSeek-V2-Lite       &  13/ 28 & 0.48 &  79.3\% & 4.76 &  44.8\% & 4.54 \\
Qwen3-8B               &  18/ 36 & 0.51 &  77.5\% & 4.78 &  58.3\% & 4.66 \\
Phi-3.5-mini-Inst      &  18/ 33 & 0.56 &  78.4\% & 4.76 &  66.4\% & 4.53 \\
gpt-oss-20B            &  13/ 24 & 0.57 &  86.7\% & 4.78 &  67.9\% & 4.57 \\
Mamba-2.8B             &  39/ 65 & 0.61 &  76.5\% & 4.75 &  19.1\% & 4.75 \\
DeBERTa-v3-large       &  16/ 25 & 0.67 &  79.6\% & 4.73 &  38.3\% & 4.72 \\
Qwen2.5-3B-Inst        &  29/ 37 & 0.81 &  72.8\% & 4.73 &  33.6\% & 4.62 \\
Falcon-Mamba-7B        &  56/ 65 & 0.88 &  82.7\% & 4.82 &  17.3\% & 4.77 \\
Phi-4                  &  35/ 40 & 0.90 &  71.0\% & 4.72 &  49.1\% & 4.59 \\
OLMo-2-7B-Inst         &  29/ 33 & 0.91 &  84.6\% & 4.80 &  72.8\% & 4.65 \\
R1-Distill-Qwen-14B    &  44/ 49 & 0.92 &  78.7\% & 4.83 &  54.9\% & 4.78 \\
R1-Distill-Qwen-7B     &  26/ 29 & 0.93 &  79.0\% & 4.65 &  22.5\% & 4.73 \\
Qwen3-4B               &  34/ 36 & 0.97 &  76.5\% & 4.76 &  66.7\% & 4.67 \\
SmolLM3-3B             &  34/ 36 & 0.97 &  80.2\% & 4.80 &  61.7\% & 4.71 \\
\bottomrule
\end{tabular}
\caption{Layer-localization of FARS across 26 models. \emph{Left half}
gives the layer $L_\text{peak}$ at which cross-format concept retrieval
Agn\% peaks; \emph{right half} gives the last-layer values. In every
model the peak layer sits at depth fraction in $[0.29, 0.97]$
of the network, with $n_\text{mid}=17/26$ peaks falling inside the
middle band $[0.30, 0.85]$. At the peak layer, $d_G$ to the top-$10$
readout channels is higher than at the last layer in 24/26
models (median $\Delta d_G = +0.12$), and Agn\% drops sharply
toward the last layer. Concept-invariance and readout-adjacency move in
opposite directions along depth, so the FARS peak is not the layer
nearest the readout basis.}
\label{tab:layer_loc}
\end{table}

\section{Empirical Dimensionality: the $k=10$ Plateau}\label{app:k_sweep}
\paragraph{Question and scope.} Why use ten dimensions? The plateau is an empirical operating point on this inventory, not an intrinsic dimension of reasoning.


\label{sec:k_sweep}
The dimension $k = 10$ used throughout the paper is the empirical
elbow of the concept-centroid PCA spectrum, not a derived constant. We sweep $k \in \{1, \ldots, 17\}$
on the TriForm benchmark (18 concepts, 17 non-trivial dimensions after
mean-centering) and record two saturation metrics: the fraction of
top-$k=17$ centroid variance captured, and the fraction of top-$k=17$
cross-format retrieval accuracy captured. Averaged over 9
models spanning three architecture families:

\begin{table}[t]
\centering
\small
\begin{tabular}{@{}lcc@{}}
\toprule
Rank & Variance (\% of $k$=17) & Agn\% (\% of $k$=17) \\
\midrule
$k = 3$ & 52.9 & 52.7 \\
$k = 5$ & 71.2 & 60.5 \\
$k = 8$ & 84.3 & 73.4 \\
$k = 10$ & 89.9 & 83.0 \\
$k = 12$ & 94.2 & 89.4 \\
$k = 17$ & 100.0 & 100.0 \\
\bottomrule
\end{tabular}
\caption{Concept-centroid PCA saturation on TriForm, mean over
9 models. Beyond $k = 10$ the marginal gain in both metrics is
$\le 10$~percentage points; $k = 10$ captures $\sim 90\%$ of both
what is attainable and what a full-rank ($k = 17$) subspace would
attain. On the disjoint Novelty benchmark (10 concepts, at most 9
non-trivial dimensions) $k = 10$ is by construction full-rank and
therefore trivially saturates; the ELBOW test above is only meaningful
on TriForm.}
\label{tab:k_sweep_triform}
\end{table}

Per-model saturation at $k = 10$ (fraction of $k=17$ maximum):

\begin{table}[t]
\centering
\small
\begin{tabular}{@{}lcc@{}}
\toprule
Model & Variance (\% of $k$=17) & Agn\% (\% of $k$=17) \\
\midrule
qwen2.5-7b               &  89.1 &  82.5 \\
qwen2.5-3b-instruct      &  91.7 &  74.9 \\
phi-3.5-mini-instruct    &  91.3 &  86.0 \\
mistral-7b-v0.3          &  87.9 &  86.1 \\
gpt2-xl                  &  91.3 &  85.8 \\
olmo-2-7b-instruct       &  87.6 &  88.4 \\
mamba-2.8b               &  93.0 &  73.0 \\
falcon-mamba-7b          &  89.0 &  80.8 \\
deberta-v3-large         &  88.6 &  89.5 \\
\bottomrule
\end{tabular}
\caption{TriForm concept-centroid PCA saturation at $k = 10$ across
9 models. The plateau is architecture-agnostic:
dense decoders, state-space models, and encoder-MLMs all reach
$\ge 87\,\%$ of the $k = 17$ variance at $k = 10$, and $\ge 73\,\%$
of the $k = 17$ cross-format retrieval accuracy. Choosing $k = 10$
therefore trades $\le 13$~percentage points of TriForm variance for a
substantial rank reduction ($10$ of $1600$--$8192$ ambient dimensions,
a $160$--$820\times$ compression). The same $k = 10$ setting is used
uniformly across all experiments in the paper.}
\label{tab:k_sweep_per_model}
\end{table}

The plateau replicates across the state-space family (Mamba,
Falcon-Mamba) and the encoder-MLM DeBERTa in addition to the six
dense-decoder models, indicating that the operating point is a
property of the concept-centroid PCA operator on a well-formed
cross-format inventory, not a hyperparameter tuned to a specific
architecture or benchmark.

\section{Inter-Layer FARS Subspace Stability}\label{app:interlayer_stability}
\paragraph{Question and scope.} How stable is the extracted span across depth? Within-model overlap does not imply identical coordinates across different models.

Subspace overlap between layer bases: $\text{overlap}(B_L, B_{L'}) = \frac{1}{10}\sum_i \sigma_i(B_L B_{L'}^\top)$. Random baseline computed between two independent random orthonormal $10$-d subspaces in the same ambient $D_m$.

\begin{table*}[t]
\centering
\caption{Inter-layer FARS subspace overlap (mean cos of principal angles). Plateau = layers with RSA-C $\geq 90\%$ peak. ``Consec within'' = overlap between consecutive plateau layers; ``$\pm 2$--$3$ within'' = overlap between plateau layers at distance $2$ or $3$. The chance baseline is $\sqrt{10/D_m} \approx 0.03$--$0.08$ across our models. Observed in-plateau overlap is $\sim 17\times$ chance: the canonical concept frame is a stable manifold across multiple layers, not a layer-specific artifact.}\label{tab:interlayer_stability}
\footnotesize
\setlength{\tabcolsep}{3pt}
\begin{tabular}{@{}lrrrr@{}}
\toprule
\textbf{Model} & \textbf{Consec in-plat} & \textbf{$\pm 2$--$3$ in-plat} & \textbf{Outside} & \textbf{Random} \\
\midrule
GPT-2 XL            & $.952$ & $.902$ & $.968$ & $.067$ \\
Qwen-3B-Inst        & $.876$ & $.766$ & $.798$ & $.058$ \\
Phi-3.5-mini-Inst   & $.854$ & $.716$ & $.835$ & $.047$ \\
Mamba-2.8B          & $.947$ & $.875$ & $.956$ & $.052$ \\
Mistral-7B-base     & $.868$ & $.735$ & $.745$ & $.041$ \\
Mistral-7B-Inst     & $.856$ & $.727$ & $.682$ & $.041$ \\
Qwen-7B             & $.869$ & $.741$ & $.824$ & $.044$ \\
OLMo-2-7B-Inst      & $.881$ & $.754$ & $.851$ & $.041$ \\
Falcon-Mamba-7B     & $.946$ & $.879$ & $.928$ & $.041$ \\
Llama-3.1-8B-Inst   & $.841$ & $.705$ & $.717$ & $.041$ \\
DeepSeek-V2-Lite    & $.878$ & $.760$ & $.799$ & $.058$ \\
Mixtral-8x7B-Inst   & $.890$ & $.790$ & $.726$ & $.041$ \\
Llama-3.1-70B-Inst  & $.936$ & $.862$ & $.880$ & $.030$ \\
DeBERTa-v3-large    & $.895$ & $.759$ & $.870$ & $.082$ \\
RoBERTa-large       & $-$ & $-$ & $.755$ & $.082$ \\
\midrule
\textbf{Mean (n=14)} & $\mathbf{.892 \pm .036}$ & $\mathbf{.784 \pm .065}$ & $.822 \pm .085$ & $\sim .05$ \\
\bottomrule
\end{tabular}
\end{table*}

In-plateau consecutive overlap $0.89$ vs.\ random $\sim 0.05$, $\sim 17\times$ chance. $\pm 2$--$3$ within plateau $0.78$, a slow rotation. Crossing boundary $0.80\pm 0.18$. RoBERTa excluded (plateau width $1$). FARS is a thick canonical frame, recoverable at any plateau layer.

\section{A Prediction That Did Not Replicate: Concept-Family Ordering}\label{app:concept_families}
\paragraph{Question and scope.} Does the proposed concept-family ordering replicate? This is a negative-result control, retained rather than treated as support for the main claim.


\paragraph{A prediction that did not replicate.}
We had expected the readout verdict to track how output-facing a concept is. The
reasoning is mechanical: a direction encoding ``the next token is a numeral'' is
a statement about the logits and should sit inside the readout basis, whereas a
direction encoding ``this inference is valid'' is a statement about the input and
should not. To test it we built five balanced binary families of minimally
contrasting sentences ($20$ pairs each), fixed their surface-to-abstract ordering
before measuring anything, took the single direction $v = \mu_{+} - \mu_{-}$
per family --- the object the steering literature uses --- and measured its
readout enrichment $E_v(k)/(k/D)$ against the full singular basis of $W_U$.

\begin{table}[t]
\centering
\small
\setlength{\tabcolsep}{3.8pt}
\begin{tabular}{@{}lccccccc@{}}
\toprule
Model & $D$ & next token  & language id & sentiment & factual tru & inference v & mono.\ \\
\midrule
gpt2-xl & 1600 & 14.4 & 8.2 & 4.0 & 1.4 & 1.0 & \checkmark \\
phi-3.5-mini-instruct & 3072 & 13.5 & 35.4 & 18.0 & 9.7 & 7.2 & --- \\
qwen2.5-7b & 3584 & 62.9 & 2.1 & 4.8 & 33.9 & 38.0 & --- \\
llama-3.1-8b-instruct & 4096 & 17.3 & 2.1 & 4.6 & 3.8 & 1.5 & --- \\
mistral-7b-v0.3 & 4096 & 19.3 & 37.5 & 26.1 & 55.3 & 15.0 & --- \\
olmo-2-7b-instruct & 4096 & 27.4 & 4.0 & 2.1 & 14.2 & 6.3 & --- \\
\bottomrule
\end{tabular}
\caption{Final-layer readout enrichment $E_v(10)/(k/D)$ of the difference-of-means direction for five binary concept families, across 6 models. Columns are in the pre-registered surface-to-abstract order, so the prediction is that each row decreases left to right; the last column marks whether it does. Haar null is $0.99$.}
\label{tab:concept_families}
\end{table}

The prediction fails. The ordering is fully monotonic in 1 of 6 models, and
41 of 60 pairwise orderings are correct against 30 expected by chance ---
above chance, but nowhere near the clean gradient the mechanism predicts. In
particular the claim we would most have wanted, that an inference-validity
direction sits at chance, holds only on \texttt{gpt2-xl}
($1.03$) and not on the larger models, where it reaches
$38\times$.

The likely cause is visible in the table. On the larger models \emph{every}
family is enriched by an order of magnitude at the final layer: the smallest
enrichment in any family averages $5.6\times$ across the five
$7$--$8$B models, against $1.03\times$ on
\texttt{gpt2-xl}. That is expected on reflection --- the final residual
stream \emph{is} the readout pathway, so a probe placed there saturates and the
contrast between concept families washes out. A mid-network probe avoids the
saturation but asks about output-facing representations before the model has
built them, which is why our first attempt at this experiment, run at the middle
layer, also produced no ordering.

We report this because it bounds what the paper claims. The readout test's
two-sidedness rests on the final-layer positive control and the rank sweep
(\S\ref{sec:readout_control}, App.~\ref{app:rank_sweep}), both of which
replicate across every model we tested. It does \emph{not} rest on the concept
being abstract in some intuitive sense, and we have no evidence that the test
orders concepts along that axis. Finding a probe position that separates
output-facing from input-facing concept directions without saturating is open.


\FloatBarrier
\Needspace{8\baselineskip}
\section*{3. Interventions and behavioural boundaries}
\section{Multi-seed cross-concept intervention audit}
\label{app:intervention_extension}
\paragraph{Question and scope.} Do interventions move outputs toward another concept instance? The event is a KL-proximity comparison, not semantic correctness. Seed ranges describe repeated runs, not confidence intervals.

The extension comprises $4$ models $\times$ $3$ seeds $\times$ $2$ protocols, or $24$ completed runs. Each run uses $120$ directed concept pairs. The prefix protocol compares two bases and the last-token protocol uses five conditions, yielding $10{,}080$ saved intervention records. These are repeated measurements, not $10{,}080$ independent samples. Within each seed, the concept-pair order and sampled surface forms are shared across models. Seeds change both pair sampling and the random subspace, so their variation reflects both sources.

The prefix protocol replaces projected residual differences at aligned prefix indices up to the shorter sequence length. Such indices need not have matching semantics. The last-token protocol instead maps the source's last token onto the target's last token at the same layer. For the norm-matched control, the random projected difference is rescaled separately for each pair to the norm of the FARS-projected difference. Full replacement uses the unprojected residual difference; no intervention adds zero. Other target-token states remain untouched in the last-token protocol. Patching full state at one layer does not replace the entire source computation.

Model/layer pairs are GPT-2 XL/20, Qwen2.5-3B-Instruct/22, SmolLM3-3B/10, and Mamba-2.8B/35 (zero-based block indices). Runs use float16, rank $10$, seeds $0,1,2$, and a maximum input length of $512$ tokens. They use PyTorch $2.6.0$+CUDA~$12.4$ and Transformers $5.13.1$ on an A100~40GB; this software environment need not match historical runs. Each output records per-pair KL distances, concept identities, form, seed, and basis hash; last-token runs additionally record intervention norms. Analytic checks cover projection, norm matching, full replacement, and the no-op control. Historical results are not overwritten.

\begin{table}[htbp]
\centering\footnotesize\setlength{\tabcolsep}{4pt}
\begin{tabular}{@{}lccccc@{}}\toprule
Model & FARS & Random & Norm-matched & Full vector & No intervention \\ \midrule
GPT-2 XL & 10.8 [7.5, 16.7] & 0.0 & 0.0 & 41.9 [34.2, 47.5] & 0.0 \\
Qwen2.5-3B-Inst. & 3.3 [0.8, 7.5] & 0.0 & 0.0 & 47.8 [43.3, 51.7] & 0.0 \\
SmolLM3-3B & 0.8 [0.0, 1.7] & 0.0 & 0.0 & 19.7 [16.7, 22.5] & 0.0 \\
Mamba-2.8B & 0.8 [0.0, 1.7] & 0.0 & 0.0 & 26.9 [21.7, 33.3] & 0.0 \\
\bottomrule\end{tabular}
\caption{Last-token source-directed event rates (\%): seed mean [minimum, maximum], $3$ seeds with $120$ directed pairs each. Zero entries are identical across seeds. The brackets are observed seed ranges, not confidence intervals. Random seeds vary both pair sampling and the random basis.}
\label{tab:intervention_extension_summary}
\end{table}

\begin{figure}[htbp]
\centering
\includegraphics[width=\linewidth]{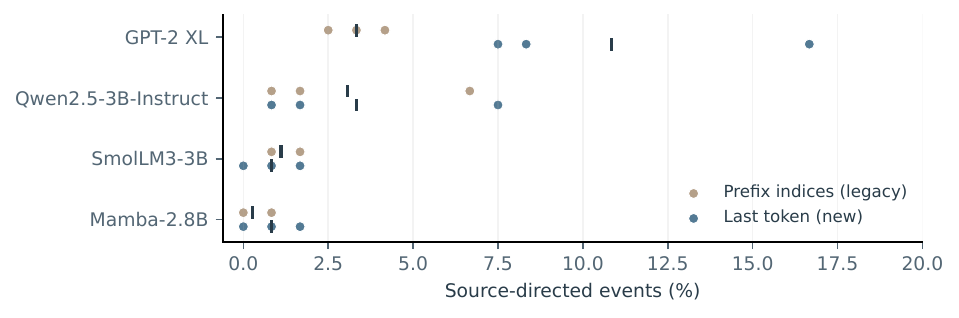}
\caption{\textbf{Token-position protocol changes the measured intervention effect.} FARS source-directed event rates for the legacy prefix and last-token protocols. Dots represent seeds; horizontal marks represent seed means. Shared concept pairs make these repeated observations. A position change is not a test of semantic correctness.}
\label{fig:intervention_protocols}
\end{figure}


\subsection{Logged dose replication on Qwen2.5-3B-Instruct}\label{app:logged_dose}
We reran the cached layer-22, rank-10 intervention $h\leftarrow h+cB^\top Bh$ on GSM8K with greedy float16 generation, applying the hook to all positions processed at that layer. The chat-template protocol uses $50$ test questions (sampling seed $123$), a $384$-token output limit, and two rank-matched Haar controls. A separate plain-prompt protocol uses $80$ questions (seed $42$), a $256$-token limit, and FARS only. All $650+400$ condition--question records completed; saved records contain question IDs, full prompts and generations, parsed answers, token-limit indicators and the basis hash. Answers use the \texttt{\#\#\#\#} pattern with a last-number fallback. Reported accuracy includes token-limited generations under that parser.

\begin{figure}[H]
\centering
\includegraphics[width=\linewidth]{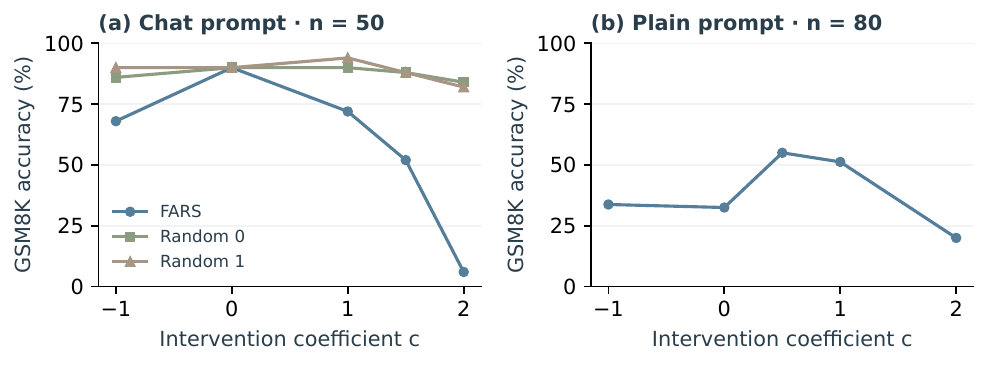}
\caption{\textbf{New logged dose sweeps, separated by protocol.} Both panels use Qwen2.5-3B-Instruct at cached layer 22. Chat: baseline $90\%$; FARS at $c=2$, $6\%$; two random controls, $84\%$ and $82\%$. Plain: baseline $32.5\%$; FARS at $c=0.5$, $55\%$, and at $c=2$, $20\%$. Lines join tested settings, without interpolated measurements or uncertainty estimates. Random controls match rank, not intervention norm.}\label{fig:logged_dose}
\end{figure}

The chat run reproduces strong disruption at high positive dose relative to rank-matched random directions. The plain-prompt run shows that a moderate positive dose can instead improve parsed accuracy under a different protocol. These runs differ in prompts, sampled items and generation limits, so their difference cannot be attributed to any one factor. They are not pooled and do not resolve the provenance of the historical $8\%$ versus $13\%$ entries in Fig.~\ref{fig:threshold}. Perturbation magnitude and truncation remain relevant alternative explanations; these results do not establish semantic sufficiency.

\subsection{Tokenwise norm-matched task intervention}\label{app:normmatched_task}
We add a logged $650$-record control experiment on Qwen2.5-3B-Instruct, holding the $50$ questions (seed $123$), layer $22$, chat prompt, greedy decoding and $384$-token limit fixed. Each question receives baseline and FARS or three rank-$10$ Haar interventions at $c\in\{-1,1,2\}$. At each affected token, a random projection is scaled to the norm of the FARS projection of that run's current hidden state. Thus rank and local perturbation norm are matched; different autoregressive trajectories need not have equal norms. The maximum recorded relative matching error is $2.41\times10^{-4}$.

\begin{figure}[H]\centering
\includegraphics[width=\linewidth]{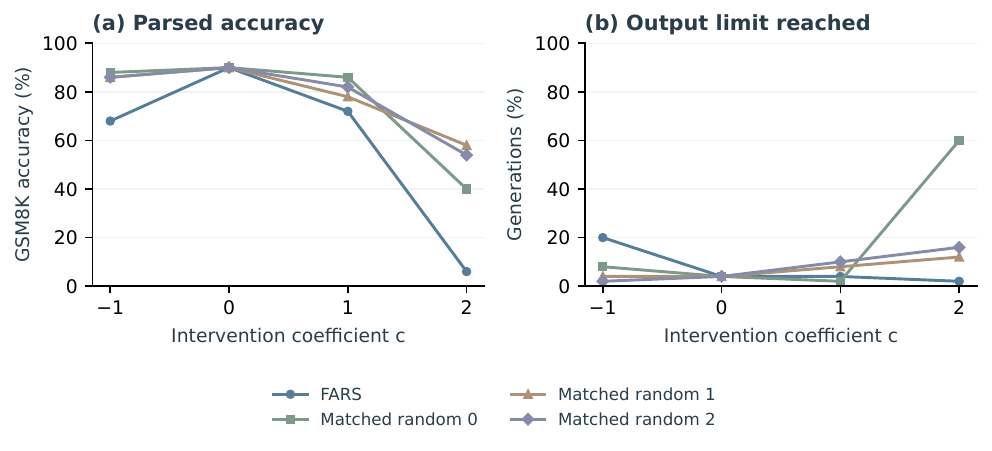}
\caption{\textbf{Norm-matched intervention and generation limits.} Each curve uses the same $50$ questions; the three random bases are fixed across questions and doses. Baseline is $90\%$ accurate. At $c=2$, FARS reaches $6\%$ versus $40$--$58\%$ for random controls, while output-limit rates are $2\%$ versus $12$--$60\%$. Truncated generations remain in the accuracy denominator. Lines connect tested doses; random curves are separate basis draws, not confidence bounds.}\label{fig:normmatched_task}
\end{figure}

FARS minus the average of the three random controls is $-18.7$, $-10.0$ and $-44.7$ percentage points at $c=-1,1,2$, respectively. Paired percentile bootstrap intervals ($10{,}000$ resamples of the $50$ questions; all conditions for a question kept together; seed $20260929$) are $[-30.0,-8.0]$, $[-21.3,0.7]$ and $[-54.0,-35.3]$ points. These intervals condition on the three sampled bases and one model; they do not quantify cross-model uncertainty. The moderate-amplification interval includes zero. Norm matching also makes high-dose random directions disruptive, so the earlier rank-only control overstates how inert random perturbations are. The residual high-dose gap supports directional sensitivity under this protocol, not causal sufficiency or a universally optimal reasoning subspace. Unequal generation-limit rates motivate a separate fixed-budget sensitivity test.

\section{Causal Validation: Cross-Form Activation Patching}\label{app:causal_validation}
\paragraph{Question and scope.} What survives cross-form activation patching? Preservation can also occur under random patches; contrast it with source-directed interventions and task-level tests.

Replacing only the $10$ FARS dimensions during cross-form activation patching preserves the model's argmax in $89.7$--$97.3\%$ of cases across five models (Table~\ref{tab:subspace_patching}). Full-PCA of the same rank retains $37.3$--$87.7\%$; full-activation replacement collapses to $22.0$--$83.3\%$. Paired Wilcoxon rejects FARS$=$Full-PCA at $p{<}10^{-22}$ on every model (Holm-corrected).

\begin{table}[H]
\centering
\caption{Causal validation under three metrics on the same $10$ dimensions: top-$1$ next-token match, top-$10$ overlap, median KL of patched vs.\ clean. FARS cross-form patching matches within-form same-instance ceiling ($0.85$ on Qwen-$7$B-Inst).}\label{tab:subspace_patching}
\scriptsize
\setlength{\tabcolsep}{3pt}
\begin{tabular}{@{}lccc|ccc|ccc@{}}
\toprule
 & \multicolumn{3}{c|}{\textbf{Top-1\%} $\uparrow$} & \multicolumn{3}{c|}{\textbf{Top-10 ov.} $\uparrow$} & \multicolumn{3}{c}{\textbf{KL med.} $\downarrow$} \\
\textbf{Model} & FARS & PCA & Full & FARS & PCA & Full & FARS & PCA & Full \\
\midrule
\raisebox{-0.18em}{\includegraphics[height=0.95em]{figures/logos/openai.pdf}}~GPT-2 XL & \textbf{93.7} & 87.7 & 81.0 & \textbf{.90} & .65 & .47 & \textbf{.016} & .154 & .655 \\
\raisebox{-0.18em}{\includegraphics[height=0.95em]{figures/logos/qwen.pdf}}~Qwen-3B & \textbf{89.7} & 37.3 & 22.0 & \textbf{.92} & .60 & .44 & \textbf{.012} & .612 & 2.77 \\
\raisebox{-0.18em}{\includegraphics[height=0.95em]{figures/logos/qwen.pdf}}~Qwen-7B & \textbf{94.7} & 59.7 & 26.3 & \textbf{.94} & .63 & .48 & \textbf{.006} & .262 & .683 \\
\raisebox{-0.18em}{\includegraphics[height=0.95em]{figures/logos/mistral.pdf}}~Mistral-7B & \textbf{97.3} & 85.3 & 83.3 & \textbf{.96} & .74 & .56 & \textbf{.002} & .034 & .142 \\
\raisebox{-0.18em}{\includegraphics[height=0.95em]{figures/logos/meta.pdf}}~Llama-8B & \textbf{94.0} & 56.7 & 39.7 & \textbf{.94} & .64 & .47 & \textbf{.004} & .098 & .811 \\
\bottomrule
\end{tabular}
\end{table}

\subsection{Cross-concept patching: the sufficiency question}\label{app:cross_concept}

\paragraph{Historical pilot and its scope.} Cross-form prediction preservation does not establish that a patch installs a concept: both source and target already express the same concept, and random patches can also preserve the prediction. The historical pilot instead patched between distinct concepts in the same surface form. Its one-model result is retained below; the expanded four-model experiment and explicit norm-matched control are reported in App.~\ref{app:intervention_extension}.

\paragraph{Protocol.} For ordered pairs $(A,B)$, the pilot replaces the target's FARS-projected residual component with the source's at prefix-aligned token indices. It records $d_A=\mathrm{KL}(p_A\Vert p_P)$ and $d_B=\mathrm{KL}(p_B\Vert p_P)$, where $P$ is the patched output. An event is $d_A<d_B$; it is not a semantic accuracy measurement. The historical Haar control matches the subspace rank but does \emph{not} explicitly match the per-pair intervention norm. That distinction matters when attributing an effect specifically to direction rather than perturbation size.

\begin{table}[htbp]
\centering\small\setlength{\tabcolsep}{6pt}
\begin{tabular}{@{}lrcccc@{}}\toprule
 & & \multicolumn{2}{c}{Source-directed events (\%)} & \multicolumn{2}{c}{Median $d_B$} \\
\cmidrule(lr){3-4}\cmidrule(lr){5-6}
Model & Pairs & FARS & Haar & FARS & Haar \\ \midrule
GPT-2 XL & 40 & 5.0 & 0.0 & 0.02435 & 0.00128 \\
\bottomrule\end{tabular}
\caption{Historical cross-concept pilot, $40$ directed pairs on one model. The FARS perturbation changes the output distribution more than this rank-matched Haar control, but the two interventions are not norm-matched. Values reproduce the saved pilot summary.}
\label{tab:cross_concept}
\end{table}

\paragraph{Interpretation.} The pilot records an intervention effect while rarely producing an output distribution nearer the source. It does not by itself isolate a norm-independent directional effect, prove necessity, or show that a source concept has been installed. The new experiment adds a per-pair norm-matched random condition, last-token alignment, full replacement, and no intervention. Those controls refine the empirical boundary without turning KL proximity into task correctness.

\section{Probe-, Token-, and Task-Level Interventions}\label{app:causal_battery}
\paragraph{Question and scope.} How do ablation and amplification change reported task outcomes? Effects depend on model, layer, strength and protocol; necessity in a tested intervention is not a recovered natural mechanism.

FARS ablation is evaluated at three levels of abstraction and against a matched-rank random control at every one. The subsections below report each in turn.

\subsection{Probe level: cross-form concept identifiability}\label{app:probe_ablation}

Ridge probe ($\alpha{=}0.1$) trained on source-form activations, tested on target after FARS-ablation vs.\ random-$10$-d ablation. FARS ablation collapses cross-form accuracy by $17$--$50$\,pp on every one of the $15$ models in this pool, spanning $4$ architecture families (Tables~\ref{tab:probe_ablation_new},~\ref{tab:probe_ablation}); random ablation moves accuracy by $\pm 0.2$\,pp. The same $10$ dimensions that recover concept identity destroy it when removed; arbitrary $10$ dimensions do neither. Llama-$70$B-Inst is probe-necessary at $-43.9$\,pp despite being task-level inert (App.~\ref{app:gsm8k}), a clean geometric-vs-behavioural dissociation.

\begin{table*}[t]
\centering
\caption{Probe-level FARS ablation replicates across $4$ new architecture families. $L$ is the model's best FARS layer, $\Delta$ the FARS-ablation effect in percentage points. Encoder models are mean-pooled. Base = cross-form ridge probe trained on form $A$, tested on form $B$ at the model's best FARS layer. FARS ablation removes the $10$-d FARS subspace from both train and test activations; random ablation removes a matched-rank random subspace (mean over $3$ seeds). Every architecture family shows $\geq -21$\,pp FARS-ablation effect with $\leq \pm 0.001$ random-ablation effect.}\label{tab:probe_ablation_new}
\scriptsize
\setlength{\tabcolsep}{2pt}
\begin{tabular}{@{}lccccc@{}}
\toprule
\textbf{Model (family)} & \textbf{L} & \textbf{Base} & \textbf{FARS abl.} & \textbf{Rand abl.} & \textbf{$\Delta$} \\
\midrule
\raisebox{-0.18em}{\includegraphics[height=0.95em]{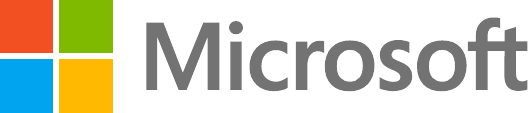}}~\textbf{DeBERTa-v3-large (encoder MLM)} & 15 & $0.602$ & $0.104$ & $0.603$ & $\mathbf{-49.8}$\,pp \\
\raisebox{-0.18em}{\includegraphics[height=0.95em]{figures/logos/mistral.pdf}}~Mistral-7B-Inst (dense) & 10 & $0.620$ & $0.173$ & $0.619$ & $\mathbf{-44.7}$\,pp \\
\raisebox{-0.18em}{\includegraphics[height=0.95em]{figures/logos/meta.pdf}}~Llama-3.1-70B-Inst (dense) & 26 & $0.628$ & $0.189$ & $0.628$ & $\mathbf{-43.9}$\,pp \\
\raisebox{-0.18em}{\includegraphics[height=0.95em]{figures/logos/mistral.pdf}}~Mixtral-8x7B-Inst (MoE) & 10 & $0.617$ & $0.220$ & $0.616$ & $\mathbf{-39.7}$\,pp \\
\raisebox{-0.18em}{\includegraphics[height=0.95em]{figures/logos/mistral.pdf}}~Mistral-7B-base (dense) & 11 & $0.528$ & $0.156$ & $0.529$ & $\mathbf{-37.2}$\,pp \\
\raisebox{-0.18em}{\includegraphics[height=0.95em]{figures/logos/microsoft.pdf}}~Phi-3.5-mini-Inst (dense) & 18 & $0.527$ & $0.179$ & $0.524$ & $\mathbf{-34.8}$\,pp \\
~Mamba-2.8B (state-space) & 35 & $0.433$ & $0.151$ & $0.432$ & $\mathbf{-28.2}$\,pp \\
\raisebox{-0.18em}{\includegraphics[height=0.95em]{figures/logos/qwen.pdf}}~Qwen-2.5-3B-Inst (dense) & 22 & $0.473$ & $0.203$ & $0.472$ & $\mathbf{-27.0}$\,pp \\
\raisebox{-0.18em}{\includegraphics[height=0.95em]{figures/logos/meta.pdf}}~RoBERTa-large (encoder MLM) & 24 & $0.270$ & $0.056$ & $0.271$ & $\mathbf{-21.4}$\,pp \\
\bottomrule
\end{tabular}
\end{table*}

\begin{table}[H]
\centering
\caption{Cross-form concept-probe accuracy with FARS or a random 10-dim subspace ablated from \emph{both} training and test activations. ``$\Delta$'' is the drop from the no-ablation baseline. FARS ablation drops accuracy $17$--$44$\,pp; random ablation does not move accuracy detectably. The same 10 dimensions that constitute FARS are the literal carriers of cross-form concept identifiability.}\label{tab:probe_ablation}
\footnotesize
\setlength{\tabcolsep}{3pt}
\begin{tabular}{@{}lcccc@{}}
\toprule
\textbf{Model} & \textbf{Baseline} & \textbf{FARS abl.} & \textbf{Random abl.} & \textbf{FARS $\Delta$} \\
\midrule
GPT-2 XL & 29.4 & \textbf{12.0} & 29.6 & $-$17.4 \\
Qwen2.5-7B & 52.8 & \textbf{18.3} & 52.7 & $-$34.5 \\
Mistral-7B & 52.8 & \textbf{15.6} & 52.9 & $-$37.2 \\
Llama-3.1-8B-Instruct & 56.5 & \textbf{17.7} & 56.4 & $-$38.8 \\
\bottomrule
\end{tabular}
\vspace{0.2em}\\
{\footnotesize All accuracies are mean cross-form probe accuracies (\%) across all $15$ form pairs. Random ablation values are means over $5$ orthonormal draws; std across draws $\leq 0.2$\,pp on every model.}
\end{table}

\subsection{Token level: FARS vs.\ random ablation}\label{app:fars_ablation}

FARS ablation vs.\ matched-rank random orthonormal ablation: top-$1$ argmax changes in $22$\,pp more cases, median KL $130\times$ larger (Table~\ref{tab:ablation_control}).

\begin{table*}[t]
\centering
\caption{Ablation control on Qwen2.5-7B-Instruct (top-$3$ FARS layers, $n{=}54$ per layer, $5$ random draws). FARS ablation flips top-$1$ argmax in $22$\,pp more cases than random ablation and median KL is $130\times$ larger.}\label{tab:ablation_control}
\footnotesize
\begin{tabular}{@{}lccc@{}}
\toprule
\textbf{Ablation} & \textbf{Top-1} & \textbf{Top-10 overlap} & \textbf{Median KL} \\
\midrule
Random 10-dim (control, mean of 5 draws) & 0.988 & 0.984 & 0.0005 \\
FARS 10-dim & \textbf{0.767} & \textbf{0.804} & \textbf{0.065} \\
\midrule
\textbf{Difference} & $-$22\,pp & $-$18\,pp & $130\times$ \\
\bottomrule
\end{tabular}
\end{table*}

\subsection{Task level: GSM8K and MATH}\label{app:gsm8k}

We test whether FARS ablation degrades downstream task performance on GSM8K~\citep{cobbe2021gsm8k}, an out-of-distribution chain-of-thought benchmark. We sample \GSMN{} problems and greedily generate completions under (a) baseline, (b) FARS ablation at the best layer via forward hook ($h\leftarrow h-B^\top Bh$), (c) random $10$-d ablation at the same layer ($\GSMR$ QR-orthonormal draws). Answers extracted by \texttt{\#\#\#\#} pattern.

\paragraph{Headline result on Qwen2.5-7B-Inst ($L{=}13$).} Baseline $\GSMBASEQ\%$; random ablation $\GSMRANDQ\%$ ($0$\,pp drop); FARS ablation $\GSMFARSQ\%$ ($\mathbf{-\GSMDROPQ}$\,pp vs.\ baseline, $\mathbf{-\GSMVSRANDQ}$\,pp vs.\ random; paired Wilcoxon $p{=}\GSMWPQ$). Cross-benchmark replication on MATH \citep{hendrycks2021math} (Level $1$--$3$, $n{=}60$): baseline $71.7\%$, FARS ablation $\mathbf{53.3\%}$ ($-18.4$\,pp), random $70.0\%$, FARS-vs-random $p{=}5\!\times\!10^{-3}$. The necessity is not GSM8K-specific.

\paragraph{Bidirectional necessity.} Amplifying instead of ablating: $h{\to} h{+}c\mathbf{B}^\top\mathbf{B}h$. On Qwen-Inst $L{=}13$, $c{=}1.0$: FARS $71\%$ vs.\ random $90\%$, $\mathbf{-19}$\,pp, paired $p{=}2.7\!\times\!10^{-5}$. Random subspaces at identical $c$ are inert in either direction (Table~\ref{tab:amplification}).

\paragraph{Scope: BBH.} On three BIG-Bench Hard sub-tasks ($n{=}75$, logical deduction / causal judgement / colored-objects reasoning), the FARS-vs-random gap is $-3.3$\,pp ($p{=}0.23$); FARS is most informative for mathematical-symbolic reasoning.

\paragraph{Scale-modulated threshold $c^\star$ on $6$ instruct models.} The same writer-layer protocol on Phi-$3.5$-mini, Qwen-$3$B-Inst, Qwen-$7$B-Inst, Mistral-$7$B-Inst, Llama-$8$B-Inst, Llama-$70$B-Inst reports model-dependent responses at the sampled doses (Fig.~\ref{fig:dose_response}, Fig.~\ref{fig:threshold}): $c^\star_{\text{Phi-3.5}}{=}c^\star_{\text{Qwen-3B}}{=}c^\star_{\text{Qwen-7B}}{=}1.0$; $c^\star_{\text{Mistral-Inst}}{=}c^\star_{\text{Llama-8B}}{=}2.0$; $c^\star_{\text{Llama-70B}}{>}2.0$. Llama-$70$B's writer-layer amplification at $c{=}+2.0$ gives FARS $93\%$ vs.\ random $87\%$ ($n{=}30$, no FARS$<$Random signal). Multi-layer scrubbing $w\in\{1,3,5,7,11\}$ at $L{=}23$ all inert ($\Delta\leq -5$\,pp, $p\geq 0.16$): $70$B's plateau exceeds $11$-layer simultaneous scrub. The token-level (App.~\ref{app:fars_ablation}) and probe-level (App.~\ref{app:probe_ablation}) tests, which target the representation directly, replicate on Llama at every scale.

\paragraph{Mistral-Inst is the cleanest bidirectional case.} On Mistral-7B-Inst $L{=}10$ ($n{=}50$): baseline $50\%$; $c{=}-1.0$ FARS $42\%$ vs.\ random $52\%$ ($-10$\,pp, $p{=}0.028$); $\mathbf{c{=}+2.0}$ FARS $\mathbf{8\%}$ vs.\ random $59\%$ ($\mathbf{-51}$\,pp, $\mathbf{p{=}1.2\!\times\!10^{-7}}$). Together with Qwen-Inst this rules out a single-model artefact; the Llama-family inertia is architecture-family-specific.

\begin{figure}[H]
\centering
\includegraphics[width=\linewidth]{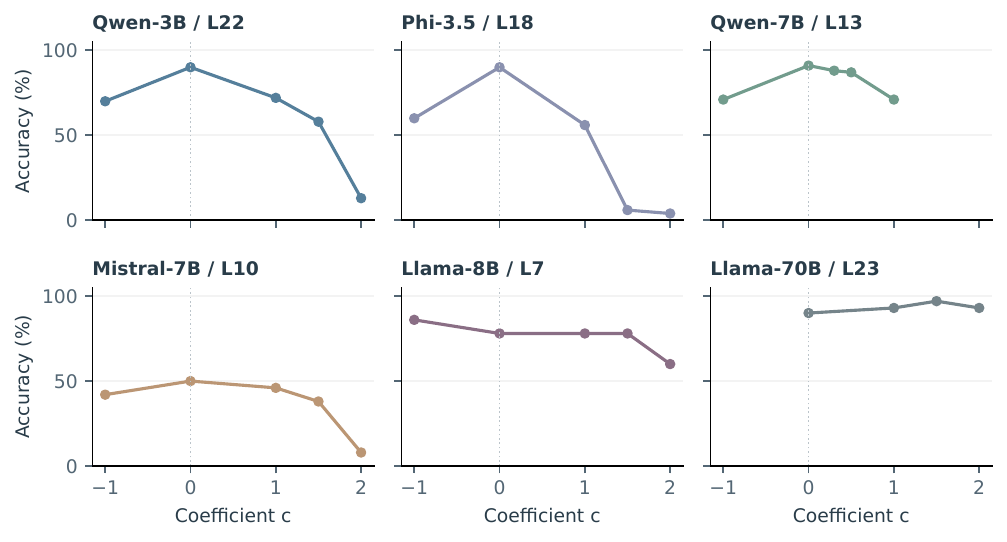}
\caption{\textbf{Historical amplification profiles, separated by model.} Points reproduce the six chat-template summaries embedded in the original plotting script, not a newly rerun dose sweep. $c=0$ is baseline and $c=-1$ is projection removal. Lines connect observed settings only; no intermediate doses or error bars are inferred. These summaries are distinct from the paired FARS--random entries in Fig.~\ref{fig:threshold}.}\label{fig:dose_response}
\end{figure}

\begin{figure}[H]
\centering
\includegraphics[width=\linewidth]{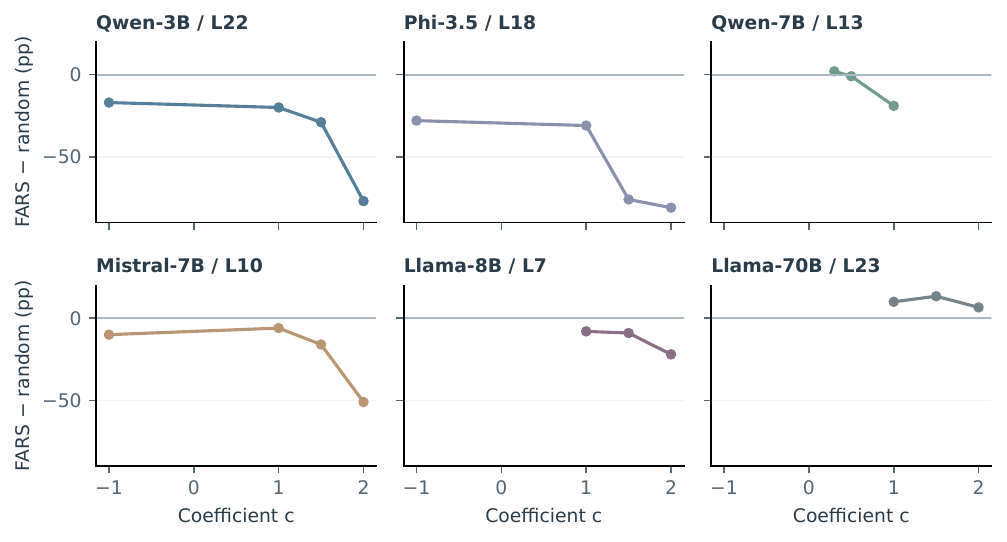}
\caption{\textbf{Paired historical FARS--random differences.} Vertical values are percentage points of GSM8K accuracy, recomputed from the paired summaries embedded in the original threshold-plot script. These are not subtractions from the separate dose-profile figure: for Qwen-3B at $c=2$, the two sources report FARS accuracy $8\%$ and $13\%$, respectively. Their run correspondence remains unresolved; they are therefore kept separate.}\label{fig:threshold}
\end{figure}

\paragraph{FCTS negative.} Concept-targeted steering (project $h$ onto FARS, find nearest centroid $\mathbf{c}^\star$, pull $h_\text{FARS}$ toward $\mathbf{c}^\star_\text{full}$) does not exceed baseline at any $\alpha$ on Qwen-Inst GSM8K; at $\alpha{=}2$ FCTS collapses to $0\%$ ($p{=}5{\times}10^{-4}$). Any linear FARS perturbation, uniform or concept-targeted, ablation or amplification, degrades task performance; the deployed model sits near a local optimum along these axes.

\begin{table*}[t]
\centering
\caption{GSM8K under FARS- vs.\ random-ablation at each model's best FARS layer ($n{=}100$, chat template). Qwen shows bidirectional necessity at $L{=}13$; Llama is inert to single-layer intervention at $L{=}8$.}\label{tab:gsm8k}
\footnotesize
\setlength{\tabcolsep}{4pt}
\begin{tabular}{@{}lccccc@{}}
\toprule
\textbf{Model} & $n$ & \textbf{Base.} & \textbf{Rand.\,abl.} & \textbf{FARS\,abl.} & \textbf{Paired test} \\
\midrule
Qwen2.5-7B-Instruct ($L{=}13$)   & \GSMN{} & \GSMBASEQ{} & \GSMRANDQ{} & \textbf{\GSMFARSQ{}} & $p{=}\GSMWPQ{}$ \\
Llama-3.1-8B-Instruct ($L{=}8$)  & \GSMN{} & \GSMBASEL{} & \GSMRANDL{} & \GSMFARSL{} & $p{=}\GSMWPL{}$ (n.s.) \\
\bottomrule
\end{tabular}
\end{table*}

\paragraph{Interpretation.} Token-level argmax flips ($+22\,$pp under FARS ablation, App.~\ref{app:fars_ablation}) translate into a task-level accuracy degradation. Random subspaces of the same rank are nearly inert by either metric; FARS is not. Because the FARS directions were chosen with no reference to GSM8K and the test problems are out of distribution for our 18 concepts, this provides evidence that FARS captures \emph{general} concept-bearing computation rather than benchmark-specific structure.

\paragraph{Bidirectional causal necessity: FARS amplification.} The ablation experiment removes the FARS subspace ($h\to h-\mathbf{B}^\top\mathbf{B}h$). If FARS is genuinely load-bearing in the residual stream, then \emph{amplifying} the same subspace ($h\to h+c\cdot\mathbf{B}^\top\mathbf{B}h$ for $c>0$) should also degrade task performance, because amplification shifts the FARS-projected activation off the magnitude the rest of the model expects. Amplifying a \emph{random} 10-dim subspace at the same coefficient should not. This is a symmetric, bidirectional causal test of FARS (Table~\ref{tab:amplification}). We ran it on Qwen2.5-7B-Instruct ($n{=}100$ GSM8K test items, $L{=}13$, $c\in\{0.3,0.5,1.0\}$, two random-subspace draws per coefficient). Baseline accuracy is $91.0\%$. At $c{=}0.3,0.5$, FARS amplification reduces accuracy by $3$\,pp and $4$\,pp, statistically n.s. against random amplification at the same coefficient. At $c{=}1.0$, FARS amplification drops accuracy to \textbf{$71.0\%$ (a $-20$\,pp loss)}, while random amplification at the same coefficient leaves accuracy at $90.0\%$ ($-1$\,pp); paired Wilcoxon for FARS $<$ Random is $p=2.7\times10^{-5}$. The two-sided test rules out chance at $p=5.4\times10^{-5}$. Combined with the ablation result ($-20$\,pp from baseline, $p=4\times10^{-5}$ vs.\ random ablation), this shows the FARS dimensions are precisely those whose perturbation in \emph{either} direction degrades task performance specifically; random dimensions of identical rank are inert to perturbation in either direction. We read this as a bidirectional causal-necessity result, stronger than ablation alone.

\paragraph{Concept-targeted steering (FCTS).} A natural follow-up: does \emph{concept-targeted} amplification help? FCTS projects the input's last-token activation onto FARS, finds the nearest concept centroid $\mathbf{c}^\star$, and pulls $h_\text{FARS}$ toward $\mathbf{c}^\star$ at the writer layer, in the spirit of prototype-based dynamic steering~\citep{pds2025} but with supervised concept centroids. On Qwen-Inst GSM8K, FCTS does not exceed the no-intervention baseline at any $\alpha$; at $\alpha{=}2.0$ it collapses to $0\%$ ($p{=}5{\times}10^{-4}$). Any linear FARS perturbation, uniform or concept-targeted, ablation or amplification, degrades task performance; the deployed model sits near a local optimum along these axes.

\begin{table*}[t]
\centering
\caption{Bidirectional intervention $h\to h+c\cdot\mathbf{B}^\top\mathbf{B}h$ (Coef.\,$=-1$ is ablation) for FARS vs.\ random orthonormal $\mathbf{B}$. Qwen-$7$B: clean bidirectional necessity ($-20$\,pp both directions, $p{<}10^{-4}$). Llama-$8$B: inert under single-layer ablation at $L{=}7, 8, 14, 18$; amplification at $L{=}7$ with $c{=}{+}2.0$ gives FARS $60.0\%$ vs.\ random $82.0\%$ ($p{=}1.4{\times}10^{-3}$, $n{=}50$). Llama-$70$B: inert at $c\in[1, 2]$ at $L{=}23$; the $11$-layer multi-scrub null indicates a wider plateau. Mistral-base: no usable baseline ($2\%$).}\label{tab:amplification}
\footnotesize
\setlength{\tabcolsep}{3pt}
\begin{tabular}{@{}lccccc@{}}
\toprule
\textbf{Model (layer, $n$)} & \textbf{Coef.} & \textbf{FARS} & \textbf{Random} & $\Delta$ & \textbf{Paired $p$} \\
\midrule
\multirow{4}{*}{Qwen-7B-Inst.\ ($L{=}13$, $n{=}100$)}
 & $-1.0$ & \textbf{.700} & .900 & $\boldsymbol{-.200}$ & $\boldsymbol{4{\times}10^{-5}}$ \\
 & $+0.3$ & .880 & .910 & $-.030$ & $0.20$ \\
 & $+0.5$ & .870 & .910 & $-.040$ & $0.11$ \\
 & $+1.0$ & \textbf{.710} & .900 & $\boldsymbol{-.190}$ & $\boldsymbol{2.7{\times}10^{-5}}$ \\
\midrule
\multirow{4}{*}{Llama-8B-Inst.\ ($L{=}8$, $n{=}100$)}
 & $-1.0$ & .850 & .855 & $-.005$ & $0.44$ \\
 & $+0.3$ & .850 & .865 & $-.015$ & $0.29$ \\
 & $+0.5$ & .860 & .855 & $+.005$ & $0.52$ \\
 & $+1.0$ & .800 & .805 & $-.005$ & $0.48$ \\
\midrule
\multirow{2}{*}{Llama-8B-Inst.\ ($L{=}14$, $n{=}50$)}
 & $-1.0$ & .780 & .780 & $.000$ & $0.50$ \\
 & $+1.0$ & .760 & .800 & $-.040$ & $0.26$ \\
\midrule
\multirow{2}{*}{Llama-8B-Inst.\ ($L{=}18$, $n{=}50$)}
 & $-1.0$ & .800 & .770 & $+.030$ & $0.029^\dagger$ \\
 & $+1.0$ & .780 & .790 & $-.010$ & $0.42$ \\
\midrule
\multirow{2}{*}{Mistral-7B-base\ ($L{=}11$, $n{=}50$)}
 & $-1.0$ & .040 & .020 & $+.020$ & $0.28$\,$^\ddagger$ \\
 & $+1.0$ & .000 & .020 & $-.020$ & $0.16$\,$^\ddagger$ \\
\bottomrule
\multicolumn{6}{l}{\footnotesize $^\dagger$ wrong-direction trend (FARS $>$ Rand); $^\ddagger$ baseline 2\%, no usable dynamic range.} \\
\end{tabular}
\end{table*}

\section{Controls for the Prose--Code Asymmetry}\label{app:prose_code}
\paragraph{Question and scope.} Can tokenization and encoding explain the prose--code effect? These targeted controls do not eliminate every difference between the two formats.

Two controls localise the cause of the prose--code patching gap.

\subsection{Tokeniser-confound control}\label{app:tokenization}

\paragraph{Setup.} Per (concept, form) take union of subword token IDs (Mistral tokenizer); compute mean Jaccard per ordered form pair. Pool $7$ models $\times$ $4$ pairs ($n{=}24$); OLS regress patching on Jaccard, inspect residuals partitioned by procedural (involves \texttt{py\_code}) vs.\ declarative.

\paragraph{Result.} Jaccard does not predict patching overlap (Pearson $r{=}{+}0.03$, $p{=}0.88$). code$\leftrightarrow$math has the highest Jaccard ($0.259$) yet near-lowest patching, opposite to a tokenisation confound. Residuals separate cleanly by class: procedural mean $-0.156$, declarative $+0.156$ (gap $+0.312$, $95\%$ CI $[+0.20, +0.42]$, Mann-Whitney $p{=}7.8{\times}10^{-5}$). Table~\ref{tab:tokenization}.

\begin{table*}[t]
\centering
\caption{Form-pair Jaccard vs.\ mean patching overlap (averaged across $6$ tested models). Procedural class: pair involves \texttt{py\_code}.}\label{tab:tokenization}
\footnotesize
\setlength{\tabcolsep}{4pt}
\begin{tabular}{@{}lccc@{}}
\toprule
\textbf{Pair} & \textbf{Jaccard} & \textbf{Patching} & \textbf{Class} \\
\midrule
en$\to$math   & .228 & .696 & decl \\
en$\to$zh     & .108 & .319 & decl \\
en$\to$code   & .216 & .195 & proc \\
code$\to$math & \textbf{.259} & .208 & proc \\
\bottomrule
\multicolumn{4}{l}{\footnotesize Jaccard does not predict patching (Pearson $r{=}{+}0.03$, n.s.); code$\leftrightarrow$math has}\\
\multicolumn{4}{l}{\footnotesize the \textit{highest} Jaccard yet second-lowest patching, opposite to a tokenisation confound.}\\
\end{tabular}
\end{table*}

\paragraph{Interpretation.} The asymmetry is a property of format identity, not lexical overlap. code$\leftrightarrow$math has the highest Jaccard yet lowest patching; EN$\leftrightarrow$math has lower Jaccard but the strongest cross-form effect.

\subsection{Declarative-Python control}\label{app:pydecl}

\paragraph{Setup.} Introduce a seventh form, \texttt{py\_decl}: each concept as a single-expression Python (no \texttt{def}, no \texttt{for}/\texttt{while}; e.g., \texttt{math.gcd(48,18)} instead of Euclidean loop). EN$\to\{$\texttt{py\_decl}, \texttt{py\_code}, \texttt{math\_notation}$\}$ patching on Qwen-$7$B-Inst, $7$ layers, $n{=}378$ each.

\paragraph{Result.} Table~\ref{tab:pydecl}: procedural axis (\texttt{py\_decl}$-$\texttt{py\_code}) $+0.054$ ($p{=}2.5{\times}10^{-9}$); surface axis (\texttt{math}$-$\texttt{py\_decl}) $+0.452$. Of the $+0.506$ total gap, encoding-style explains $\sim 11\%$, format-surface $\sim 89\%$. Procedural-encoding effect concentrates in concepts whose \texttt{py\_code} relied on loops (\texttt{rel\_transitivity} $+0.262$, \texttt{spatial\_direction} $+0.219$).

\begin{table}[H]
\centering
\caption{Decomposing the prose-code gap on Qwen2.5-7B-Instruct. Mean top-10 overlap, EN-prose source $\to$ each target form, averaged over $54$ stimuli $\times 7$ layers ($n=378$). Paired Wilcoxon tests on within-(concept, instance, layer) differences.}\label{tab:pydecl}
\footnotesize
\setlength{\tabcolsep}{4pt}
\begin{tabular}{@{}lcc@{}}
\toprule
\textbf{Comparison} & \textbf{Mean overlap} & \textbf{Paired test} \\
\midrule
EN $\to$ math\_notation              & .646 & --- \\
EN $\to$ py\_decl                    & .194 & --- \\
EN $\to$ py\_code                    & .140 & --- \\
\midrule
py\_decl $-$ py\_code (procedural axis)        & $+.054$ & $p{=}2.5\!\times\!10^{-9}$, CI $[-.006,+.111]$ \\
math\_notation $-$ py\_decl (surface axis)     & $+.452$ & $p{<}10^{-60}$ \\
math\_notation $-$ py\_code (total gap)        & $+.506$ & $p{<}10^{-60}$ \\
\bottomrule
\end{tabular}
\end{table}

\section{Cross-Format Safety Refusal Test (Falsified Hypothesis)}\label{app:advbench}
\paragraph{Question and scope.} Does the representational asymmetry predict a safety gap? The stated hypothesis was falsified on this test; it is not part of the supporting evidence.

We test whether the representational prose-code asymmetry (\S\ref{app:pydecl}) yields a behavioural safety gap. Two instruction-tuned open-weight models, AdvBench \citep{zou2023universal} harmful prompts presented in (i) English prose and (ii) a fixed Python code wrap (\texttt{\# Task: <prompt>...def perform\_task():}). Refusal detected by string match against $34$ refusal patterns \citep{zou2023universal,wei2024jailbroken}. We release aggregate refusal counts only.

\paragraph{Result.} Hypothesis not supported (Table~\ref{tab:advbench}). Qwen2.5-7B-Inst ($n{=}100$): refusal gap $+1$\,pp ($98\%$ prose vs.\ $97\%$ code, $p{=}0.50$, n.s.). Llama-3.1-8B-Inst ($n{=}200$): $-3$\,pp in the opposite direction ($93.5\%$ prose vs.\ $96.5\%$ code, $p{=}0.996$). The Python wrap does not bypass refusal; on the larger sample it triggers \emph{more} refusal.

\begin{table*}[t]
\centering
\caption{Cross-format refusal on two instruction-tuned models, AdvBench harmful behaviours.}\label{tab:advbench}
\scriptsize
\setlength{\tabcolsep}{3pt}
\begin{tabular}{@{}lccccc@{}}
\toprule
\textbf{Model} & $n$ & \textbf{Prose refusal} & \textbf{Code refusal} & \textbf{Discordant (p/c)} & \textbf{$p$\,(prose$>$code)} \\
\midrule
Qwen2.5-7B-Instruct   & 100 & .980 & .970 & 1\,/\,0 & .500 (n.s.) \\
Llama-3.1-8B-Instruct & 200 & .935 & \textbf{.965} & 1\,/\,7 & .996 (reject) \\
\bottomrule
\end{tabular}
\end{table*}

\paragraph{Interpretation.} Either modern safety training has extended to code-wrapped variants, or the wrap itself acts as a salient harmful-intent signal. The aggregate test rules out the simple cross-form refusal-gap hypothesis. Scope: $2$ open-weight models; simple wrap; aggregate counts; conservative classifier.

\section{Head-Level Decomposition at the FARS Writer Layer}\label{app:heads}
\paragraph{Question and scope.} How are effects distributed across heads and axes? A component-level perturbation does not uniquely identify the naturally used computation.

\paragraph{Method.} At the writer layer, decompose per-head residual contributions $\mathbf{r}_h{=}\mathbf{W}_O^{(h)\top}\mathbf{v}_h$ and project onto FARS to get each head's FARS-bearing magnitude.

\paragraph{Result on Llama-$8$B and Mistral-$7$B ($32$ heads each, $324$ stimuli; Table~\ref{tab:heads_decomp}).} Llama ($L{=}7$): attention $75.6\%$, MLP $24.4\%$; top-$3$ heads (H1, H14, H8) carry $15.8\%$. Mistral ($L{=}10$): attention $69.6\%$, MLP $30.4\%$; top-$3$ (H31, H4, H0) $12.7\%$. Specific head indices differ; structural profile is identical, layer-localised, head-distributed assembly from $\sim 20$ heads with MLP as the single largest individual contributor.

\begin{table*}[t]
\centering
\caption{Head-level FARS decomposition at each model's writer layer ($324$ stimuli). Attention contribution is the sum of all $32$ heads' FARS-projected output norms; MLP share is the MLP block's share of the layer's total FARS-bearing residual. Two architectures, same structural profile: layer-localised, head-distributed assembly. Head indices differ; the structural fingerprint replicates.}\label{tab:heads_decomp}
\footnotesize
\setlength{\tabcolsep}{3pt}
\begin{tabular}{@{}lcc@{}}
\toprule
\textbf{Quantity} & \textbf{Llama-3.1-8B-Inst.\ ($L{=}7$)} & \textbf{Mistral-7B-v0.3 ($L{=}10$)} \\
\midrule
Layer FARS-fraction          & $0.269$  & $0.225$  \\
MLP share of layer FARS      & $24.4\%$ & $30.4\%$ \\
Attention share of layer FARS & $75.6\%$ & $69.6\%$ \\
Top-3 head share              & $15.8\%$ & $12.7\%$ \\
Top-5 head share              & $23.8\%$ & $18.7\%$ \\
Top-3 head indices            & H1, H14, H8 & H31, H4, H0 \\
Top-5 head indices            & + H29, H12  & + H2, H28 \\
Heads needed for $75\%$ of head-FARS & $19/32$ & $20/32$ \\
Heads needed for $90\%$ of head-FARS & $27/32$ & $27/32$ \\
\bottomrule
\end{tabular}
\end{table*}

\paragraph{Implication.} FARS is not a single-head circuit; it is a layer-localised, head-distributed phenomenon. The MLP's $24$--$30\%$ share, larger than any single head's contribution, is consistent with prior findings that key concept-representations in transformers are partially refined in the MLP after attention has done the bulk of the routing~\citep{geva2022mlp,meng2022locating}. This head-and-MLP profile is what a reviewer should expect of an interpretable subspace at the scale of $18$ general reasoning concepts: a single head cannot route $18$ concepts cleanly, so the model assembles the $10$-dim concept-direction subspace from a small committee of heads. The replication across Llama and Mistral, two architectures trained on different corpora by different organisations, shows the assembly profile (one MLP plus a top-five head committee, broader long tail) is robust to architecture and training.

\paragraph{Head ablation honest negative.} Joint ablation of the top-$5$ heads at Llama-$8$B $L{=}7$ does not disrupt FARS RSA ($\Delta{=}{-}0.002$, identical to random-$5$). The decomposition reports \emph{contribution} (norm), not causal necessity at the head level; multi-layer scrubbing would be required.

\subsection{Axis-specific ablation: FARS is distributed}\label{app:axis_specific}

\paragraph{Method.} Apply Qwen-$7$B-Inst's canonical rotation $R$ to its FARS basis to obtain ten $1$-d canonical-axis directions; ablate one at a time at $L{=}13$ during GSM8K generation, $n{=}100$.

\paragraph{Result.} Full FARS ablation drops accuracy by $14$\,pp ($91\%\to 77\%$); each $1$-d canonical-axis ablation lies within $\pm 2$\,pp of baseline (Table~\ref{tab:axis_ablation}, all $p{>}0.21$). FARS is distributed; canonical axes are descriptive labels, not $1$-d causal handles.

\begin{table}[H]
\centering
\caption{Axis-specific GSM8K ablation on Qwen2.5-7B-Instruct ($L{=}13$, $n{=}100$). Full FARS ablation drops $14$\,pp; \emph{no single canonical axis} produces a separable effect ($\pm 2$\,pp, all $p{>}0.2$). FARS is a distributed $10$-d representation; single-axis interpretations are descriptive labels of geometric structure, not $1$-d causal handles.}\label{tab:axis_ablation}
\footnotesize
\setlength{\tabcolsep}{4pt}
\begin{tabular}{@{}lcc@{}}
\toprule
\textbf{Intervention} & \textbf{Accuracy} & $\Delta$ \textbf{vs.\ baseline} \\
\midrule
Baseline (no hook)               & $0.910$ & --- \\
Full FARS ablation ($10$-d)      & $\mathbf{0.770}$ & $\boldsymbol{-0.140}$ \\
\midrule
Canonical axis $0$ ($1$-d)       & $0.910$ & $+0.000$ \\
Canonical axis $1$ ($1$-d, ``procedural'') & $0.910$ & $+0.000$ \\
Canonical axis $2$ ($1$-d)       & $0.930$ & $+0.020$ \\
Canonical axis $3$ ($1$-d)       & $0.890$ & $-0.020$ \\
Canonical axis $4$ ($1$-d, ``logical-inference'') & $0.920$ & $+0.010$ \\
Canonical axis $5$--$9$ ($1$-d each) & $0.900$--$0.930$ & $\pm 0.020$ \\
Random $1$-d (mean of $3$ draws) & $0.933$ & $+0.023$ \\
\bottomrule
\end{tabular}
\end{table}

\paragraph{Cross-model replication.} Llama-$8$B-Inst ($n{=}80$, $L{=}8$): baseline $80\%$, full FARS $77.5\%$ ($-2.5$\,pp); each $1$-d axis within $\pm 5$\,pp. Cumulative top-$K$ on Qwen-Inst GSM8K: monotone drop $92\%\to 80\%$, $\sim 83\%$ of effect by $K{=}4$; random-$K$ inert. Llama-$70$B at $L{=}26$ ($n{=}20$): top-$K$ and random-$K$ both give exactly $95\%$ at every $K$, even though $\lVert B\mathbf{B}^\top h\rVert/\lVert h\rVert \in[0.12, 0.20]$ confirms the hook removes a real residual-stream component. The dissociation between best-RSA layer and behaviour-load-bearing layer survives the $8\times$ scale-up. Cumulative ablation is, by contrast, inert on the Llama family: on Llama-$3.1$-$8$B-Inst ($L{=}8$) the top-$K$ curve stays in $85.0$--$90.0\%$ and the random-$K$ control in $82.5$--$92.5\%$ across $K \in [0,10]$, and on Llama-$3.1$-$70$B-Inst ($L{=}26$) accuracy is unchanged at $95\%$ for every $K$ under both conditions --- the same capacity-scaled inertness to single-layer intervention reported in App.~\ref{app:gsm8k}.

\section{Where Is FARS Written? A Per-Layer Decomposition}\label{app:layer_writers}
\paragraph{Question and scope.} Where do residual updates align with the extracted span? This is a layer-local analysis with its own model pool and selection rule.

\paragraph{Method.} The residual-stream contribution at layer $\ell$ is $\Delta_\ell = h^{(\ell+1)} - h^{(\ell)}$. Project $\Delta_\ell$ onto FARS and compute concept-RSA on $B\Delta_\ell$; the maximising $\ell$ is the writer layer.

\paragraph{Result: writer-before-peak.} The writer sits $0$--$5$ layers below the peak FARS layer on all $15$ models from $4$ architecture families (Table~\ref{tab:layer_writers}); Mamba has the widest gap ($5$) consistent with its state-space residual propagation. FARS is composed in one or two layers and persists for $\sim 5$--$10$ layers before decaying toward the unembedding.

\begin{table}[H]
\centering
\caption{Writer layer (top concept-RSA on FARS-projected delta $B\Delta_\ell$) sits $0$--$5$ layers below the peak FARS layer on all $15$ models. Mamba has the widest gap, consistent with state-space residual propagation.}\label{tab:layer_writers}
\footnotesize
\setlength{\tabcolsep}{3pt}
\begin{tabular}{@{}lcc@{}}
\toprule
\textbf{Model} & \textbf{Peak / Writer} & \textbf{Plateau} \\
\midrule
GPT-2 XL                & 20 / 18 & 16--25 \\
Qwen-2.5-3B-Inst        & 22 / 22 & --- \\
Phi-3.5-mini-Inst       & 18 / 17 & --- \\
Qwen-2.5-7B             & 13 / 12 & 9--19 \\
Mistral-7B-v0.3         & 11 / 10 & 6--19 \\
Mistral-7B-Inst         & 10 / 8  & --- \\
Llama-3.1-8B-Inst       & 8 / 7   & 5--13 \\
Mixtral-8x7B-Inst       & 10 / 9  & --- \\
Llama-3.1-70B-Inst      & 26 / 23 & 11--77 \\
Mamba-2.8B              & 35 / 30 & --- \\
RoBERTa-large           & 24 / 23 & --- \\
DeBERTa-v3-large        & 15 / 13 & --- \\
\bottomrule
\end{tabular}
\end{table}

\paragraph{Plateau is a thick manifold.} Within-plateau consecutive-layer FARS subspace overlap (principal-angle cosine) averages $\mathbf{0.892\pm 0.036}$ across $14$ models, $\pm 2$--$3$ layers within plateau $0.784\pm 0.065$; random $10$-d baseline $\sim 0.05$ ($17\times$ chance, $z{>}20$ per model). Outside plateau overlap drops to $0.822$; crossing plateau boundary $0.798\pm 0.182$. FARS is recoverable at any plateau layer up to small rotation, not a single-layer artefact.

\FloatBarrier
\Needspace{8\baselineskip}
\section*{4. Generalization beyond the extraction inventory}
\section{Held-Out Concept Generalization and Cross-Inventory Transfer}\label{app:holdout_concepts}
\paragraph{Question and scope.} Does the extraction procedure transfer to a disjoint inventory? Re-extracting a new basis is different from transporting the original basis unchanged.

\paragraph{Per-model cross-inventory transfer (Novelty).} Table~\ref{tab:novelty_retrieval} gives the per-model numbers behind \S\ref{sec:novelty_transfer}.

\subsection{Frozen TriForm basis on unseen concepts}\label{app:frozen_transfer}
To distinguish extraction-method transfer from fixed-basis transfer, we freeze each model's unique saved TriForm layer and rank-$10$ basis before evaluating the $180$ Novelty stimuli. Query items are classified by nearest Euclidean concept centroid formed from the other five surface forms; all items of the query form are excluded from its gallery. Novelty labels define the retrieval gallery, but do not fit the projection or select its layer. This is supervised gallery retrieval, not zero-shot label inference. Full-space and $20$ Haar rank-$10$ controls (seeds $100$--$119$) use the same layer and gallery protocol.

\begin{table}[H]\centering\small
\caption{\textbf{Fixed-basis transfer with matched checkpoints and layers.} Accuracy in percent; chance is $10\%$. Random is the mean of $20$ bases. These descriptive comparisons are not significance tests.}\label{tab:frozen_transfer}
\begin{tabular}{@{}lrrrr@{}}\toprule
Model & Layer & Frozen FARS & Full space & Random mean \\
\midrule
\raisebox{-0.18em}{\includegraphics[height=0.95em]{figures/logos/openai.pdf}} GPT-2 XL & 20 & 22.8 & 48.3 & 20.2 \\
\raisebox{-0.18em}{\includegraphics[height=0.95em]{figures/logos/qwen.pdf}} Qwen2.5-7B & 13 & 40.0 & 90.6 & 35.3 \\
\raisebox{-0.18em}{\includegraphics[height=0.95em]{figures/logos/mistral.pdf}} Mistral-7B-v0.3 & 11 & 39.4 & 86.7 & 30.3 \\
\bottomrule\end{tabular}
\end{table}

All three frozen bases exceed the sampled random-control mean but lose substantial retrieval accuracy relative to the full representation. Thus transfer is partial rather than uniformly near chance, and these results do not establish a universal fixed rank-$10$ space. Checkpoint metadata were matched exactly: the Qwen-3B Novelty cache identifies the base model and cannot be paired with the available instruct basis; Mamba is omitted because its TriForm cache lacks model identity metadata. The earlier pilot also compared a random control at a different selected layer and omitted cross-form transfer scores from its saved summary; the present comparison corrects those limitations.

\begin{table}[h]
\centering
\scriptsize
\setlength{\tabcolsep}{4pt}
\begin{tabular}{@{}llccccc@{}}
\toprule
Model & Family & \multicolumn{2}{c}{TriForm-FARS} & \multicolumn{2}{c}{Novelty-FARS} & Random \\
 & & Agn\% & X/W & Agn\% & X/W & Agn\% \\
\midrule
R1-Distill-Llama-8B      & reasoning    &  79.6 & 0.77 & $\mathbf{100.0}$ & $\mathbf{1.00}$ &  50.6 \\
Mixtral-8x7B-Inst        & MoE          &  82.1 & 0.81 & $\mathbf{100.0}$ & $\mathbf{0.98}$ &  56.1 \\
Llama-3.1-70B-Inst       & dense        &  82.1 & 0.78 & $\mathbf{100.0}$ & $\mathbf{1.00}$ &  56.1 \\
SmolLM3-3B               & dense        &  82.4 & 0.80 & $\mathbf{100.0}$ & $\mathbf{0.99}$ &  54.4 \\
Yi-1.5-9B-Chat           & dense        &  75.0 & 0.72 & $\mathbf{98.9}$ & $\mathbf{1.01}$ &  45.0 \\
Mistral-7B-Inst          & dense        &  86.7 & 0.84 & $\mathbf{98.3}$ & $\mathbf{0.98}$ &  55.0 \\
Llama-3.1-8B-Inst        & dense        &  78.1 & 0.71 & $\mathbf{98.3}$ & $\mathbf{0.99}$ &  41.1 \\
Mistral-7B               & dense        &  80.2 & 0.76 & $\mathbf{97.8}$ & $\mathbf{0.98}$ &  50.0 \\
Granite-3.3-8B-Inst      & dense        &  80.2 & 0.77 & $\mathbf{97.8}$ & $\mathbf{0.98}$ &  46.7 \\
Falcon3-7B-Inst          & dense        &  77.5 & 0.74 & $\mathbf{97.8}$ & $\mathbf{0.98}$ &  45.0 \\
Qwen3-8B                 & dense        &  71.9 & 0.70 & $\mathbf{97.2}$ & $\mathbf{1.01}$ &  53.9 \\
OLMo-2-7B-Inst           & dense        &  82.7 & 0.80 & $\mathbf{96.7}$ & $\mathbf{0.99}$ &  47.2 \\
Falcon-Mamba-7B          & state-space  &  71.6 & 0.68 & $\mathbf{96.7}$ & $\mathbf{0.98}$ &  40.6 \\
QwQ-32B                  & reasoning    &  78.4 & 0.77 & $\mathbf{96.7}$ & $\mathbf{0.99}$ &  48.9 \\
Qwen3-4B                 & dense        &  75.6 & 0.75 & $\mathbf{95.0}$ & $\mathbf{0.99}$ &  47.2 \\
DeepSeek-V2-Lite         & MoE          &  75.3 & 0.72 & $\mathbf{93.9}$ & $\mathbf{0.93}$ &  45.6 \\
Phi-3.5-mini-Inst        & dense        &  70.1 & 0.68 & $\mathbf{92.8}$ & $\mathbf{0.96}$ &  41.1 \\
gpt-oss-20B              & MoE          &  77.2 & 0.71 & $\mathbf{92.8}$ & $\mathbf{0.93}$ &  47.2 \\
Qwen2.5-7B               & dense        &  74.4 & 0.71 & $\mathbf{91.1}$ & $\mathbf{0.94}$ &  36.1 \\
Qwen2.5-3B-Inst          & dense        &  59.9 & 0.56 & $\mathbf{85.6}$ & $\mathbf{0.96}$ &  35.0 \\
Mamba-2.8B               & state-space  &  48.5 & 0.50 & $\mathbf{80.0}$ & $\mathbf{0.84}$ &  36.7 \\
R1-Distill-Qwen-7B       & reasoning    &  75.0 & 0.73 & $\mathbf{79.4}$ & $\mathbf{0.87}$ &  31.7 \\
GPT-2 XL                 & dense        &  63.6 & 0.61 & $\mathbf{70.0}$ & $\mathbf{0.79}$ &  30.0 \\
Phi-4                    & dense        &  71.0 & 0.65 & $\mathbf{61.7}$ & $\mathbf{0.70}$ &  35.0 \\
DeBERTa-v3-large         & encoder-MLM  &  84.6 & 0.75 & $\mathbf{31.7}$ & $\mathbf{0.40}$ &  25.6 \\
\bottomrule
\end{tabular}
\caption{Cross-format concept retrieval on TriForm vs.\ Novelty across
25 models spanning the MoE, dense, encoder-MLM, reasoning, state-space families. Chance rates
are $1/18 = 5.6\%$ (TriForm) and $1/10 = 10.0\%$ (Novelty).
The Random column is a Haar orthonormal rank-10 subspace of the same
activation space with no concept supervision. Novelty-FARS exceeds it by
6--57\,pp on every model, and the cross-to-within-format
ratio $X / W$ reaches $\ge 0.94$ in 17 of 25 models. The higher Novelty
headline reflects the smaller concept set (10 vs.\ 18), not a stronger effect.}
\label{tab:novelty_retrieval}
\end{table}

For each of $5$ random splits, hold out $3$ of the $18$ concepts; train FARS on the remaining $15$, project held-out into the trained basis, measure concept-RSA on held-outs. Control: random orthonormal $10$-d projection at the same layer.
\begin{table*}[t]
\centering
\caption{Held-out concept-RSA in the FARS basis trained on $15$/$18$ concepts ($N_{\text{splits}}=5$ random concept-splits each). Random $10$-d orthonormal projection control at same layer. $\Delta = \text{FARS} - \text{random}$. Every model shows positive $\Delta$.}\label{tab:holdout_concepts}
\footnotesize
\begin{tabular}{@{}lccc@{}}
\toprule
\textbf{Model} & \textbf{Held-out RSA-C} & \textbf{Random RSA-C} & $\boldsymbol{\Delta}$ \\
\midrule
\raisebox{-0.18em}{\includegraphics[height=0.95em]{figures/logos/openai.pdf}}~GPT-2 XL & $.30$ & $.10$ & $+.20$ \\
\raisebox{-0.18em}{\includegraphics[height=0.95em]{figures/logos/qwen.pdf}}~Qwen-3B-Inst & $.58$ & $.23$ & $+.36$ \\
\raisebox{-0.18em}{\includegraphics[height=0.95em]{figures/logos/microsoft.pdf}}~Phi-3.5-mini-Inst & $.47$ & $.22$ & $+.25$ \\
\raisebox{-0.18em}{\includegraphics[height=0.95em]{figures/logos/mistral.pdf}}~Mistral-7B-base & $.57$ & $.21$ & $+.36$ \\
\raisebox{-0.18em}{\includegraphics[height=0.95em]{figures/logos/mistral.pdf}}~Mistral-7B-Inst & $.66$ & $.25$ & $+.40$ \\
\raisebox{-0.18em}{\includegraphics[height=0.95em]{figures/logos/qwen.pdf}}~Qwen-7B & $.56$ & $.18$ & $+.38$ \\
\raisebox{-0.18em}{\includegraphics[height=0.95em]{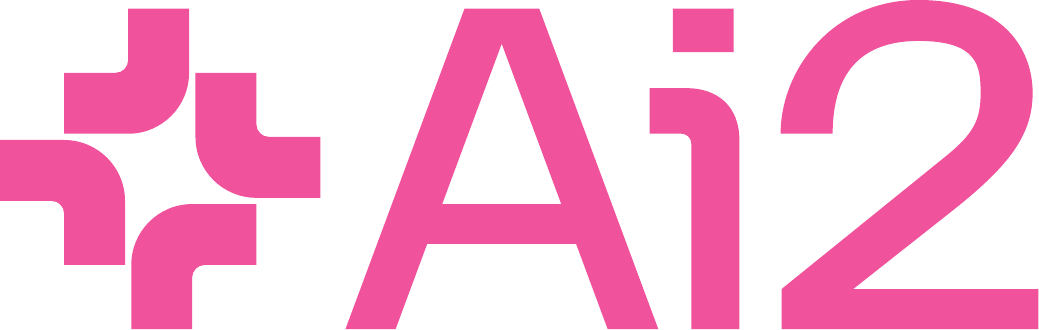}}~OLMo-2-7B-Inst & $.52$ & $.18$ & $+.34$ \\
\raisebox{-0.18em}{\includegraphics[height=0.95em]{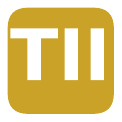}}~Falcon-Mamba-7B & $.48$ & $.17$ & $+.32$ \\
\raisebox{-0.18em}{\includegraphics[height=0.95em]{figures/logos/meta.pdf}}~Llama-3.1-8B-Inst & $.66$ & $.29$ & $+.37$ \\
\raisebox{-0.18em}{\includegraphics[height=0.95em]{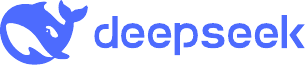}}~DeepSeek-V2-Lite & $.62$ & $.22$ & $+.40$ \\
\raisebox{-0.18em}{\includegraphics[height=0.95em]{figures/logos/mistral.pdf}}~Mixtral-8x7B-Inst & $.64$ & $.23$ & $+.41$ \\
~Mamba-2.8B & $.59$ & $.22$ & $+.36$ \\
\raisebox{-0.18em}{\includegraphics[height=0.95em]{figures/logos/meta.pdf}}~Llama-3.1-70B-Inst & $\mathbf{.71}$ & $.26$ & $+\mathbf{.45}$ \\
\raisebox{-0.18em}{\includegraphics[height=0.95em]{figures/logos/meta.pdf}}~RoBERTa-large & $.14$ & $.07$ & $+.07$ \\
\raisebox{-0.18em}{\includegraphics[height=0.95em]{figures/logos/microsoft.pdf}}~DeBERTa-v3-large & $.51$ & $.28$ & $+.23$ \\
\midrule
\textbf{Mean} & $\mathbf{.533}$ & $.207$ & $+\mathbf{.326}$ \\
\bottomrule
\end{tabular}
\end{table*}

\paragraph{Interpretation.} Held-out RSA recovers to $\mathbf{0.533}$, $2.6\times$ the random baseline. The absolute value is not directly comparable to the standard $18$-concept figure ($0.34$--$0.37$): RSA over a $3$-concept RDM has far fewer degrees of freedom than over an $18$-concept one and is correspondingly inflated. The load-bearing comparison is against the random $10$-d projection measured on the \emph{same} held-out split. Every $15$ model shows positive $\Delta$. FARS captures generic concept geometry, not over-fit to its construction set.

\section{Capability Prediction: FARS Plateau Width vs.\ Downstream Benchmarks}\label{app:capability_prediction}
\paragraph{Question and scope.} Does plateau width correlate with benchmark performance? These exploratory associations are not a held-out predictive evaluation.

\paragraph{Setup.} $\text{plateau\_frac}(m) = \frac{1}{L_m}\sum_\ell \mathbb{1}[\text{RSA-C}_\ell(m) \geq 0.9 \max_\ell \text{RSA-C}(m)]$. Paired with published benchmark accuracies; Spearman, Pearson, partial Spearman controlling for $\log_{10}P$.

\begin{table*}[t]
\centering
\caption{FARS plateau width (fraction of layers maintaining $\geq 90\%$ peak concept-RSA), per-model peak RSA-C, and published benchmark scores. Plateau width correlates positively with all three benchmarks; the MMLU correlation passes $p<0.05$.}\label{tab:capability_prediction}
\footnotesize
\begin{tabular}{@{}lrrrrr@{}}
\toprule
\textbf{Model} & \textbf{RSA-C} & \textbf{Plateau} & \textbf{GSM8K} & \textbf{MMLU} & \textbf{HumanEval} \\
\midrule
\raisebox{-0.18em}{\includegraphics[height=0.95em]{figures/logos/openai.pdf}}~GPT-2 XL                  & $.317$ & $.292$ & $5.0$   & $26.6$ & $0.0$ \\
~Mamba-2.8B                & $.368$ & $.508$ & $5.0$   & $25.6$ & $0.0$ \\
\raisebox{-0.18em}{\includegraphics[height=0.95em]{figures/logos/deepseek.pdf}}~DeepSeek-V2-Lite-Chat     & $.370$ & $.679$ & $62.5$  & $55.7$ & $40.9$ \\
\raisebox{-0.18em}{\includegraphics[height=0.95em]{figures/logos/qwen.pdf}}~Qwen-2.5-3B-Inst          & $.357$ & $.676$ & $79.0$  & $65.6$ & $73.2$ \\
\raisebox{-0.18em}{\includegraphics[height=0.95em]{figures/logos/microsoft.pdf}}~Phi-3.5-mini-Inst         & $.363$ & $.576$ & $82.6$  & $69.0$ & $62.8$ \\
\raisebox{-0.18em}{\includegraphics[height=0.95em]{figures/logos/tii.pdf}}~Falcon-Mamba-7B           & $.370$ & $.769$ & $52.1$  & $62.0$ & $29.9$ \\
\raisebox{-0.18em}{\includegraphics[height=0.95em]{figures/logos/ai2.pdf}}~OLMo-2-7B-Inst            & $.369$ & $.818$ & $75.1$  & $63.7$ & $41.0$ \\
\raisebox{-0.18em}{\includegraphics[height=0.95em]{figures/logos/mistral.pdf}}~Mixtral-8x7B-Inst         & $.368$ & $.788$ & $71.7$  & $71.5$ & $40.2$ \\
\raisebox{-0.18em}{\includegraphics[height=0.95em]{figures/logos/mistral.pdf}}~Mistral-7B-Inst           & $.368$ & $.818$ & $38.5$  & $60.1$ & $40.2$ \\
\raisebox{-0.18em}{\includegraphics[height=0.95em]{figures/logos/mistral.pdf}}~Mistral-7B-base           & $.366$ & $.875$ & $52.1$  & $62.7$ & $28.7$ \\
\raisebox{-0.18em}{\includegraphics[height=0.95em]{figures/logos/qwen.pdf}}~Qwen-2.5-7B               & $.367$ & $.857$ & $82.3$  & $74.2$ & $57.9$ \\
\raisebox{-0.18em}{\includegraphics[height=0.95em]{figures/logos/meta.pdf}}~Llama-3.1-8B-Inst         & $.366$ & $.875$ & $84.5$  & $73.0$ & $72.6$ \\
\raisebox{-0.18em}{\includegraphics[height=0.95em]{figures/logos/meta.pdf}}~Llama-3.1-70B-Inst        & $.372$ & $\mathbf{.914}$ & $95.1$  & $86.0$ & $80.5$ \\
\midrule
\multicolumn{2}{l}{\textbf{Spearman vs.\ plateau}} & & $\mathbf{+0.523}$ & $\mathbf{+0.656}$ & $+0.430$ \\
\multicolumn{2}{l}{$p$-value} & & $0.066$ & $\mathbf{0.015}$ & $0.143$ \\
\multicolumn{2}{l}{\textbf{Partial Spearman} (control: $\log_{10}P$)} & & $+0.385$ & $+0.484$ & $+0.297$ \\
\multicolumn{2}{l}{Partial $p$-value} & & $0.194$ & $0.094$ & $0.325$ \\
\bottomrule
\end{tabular}
\end{table*}

\paragraph{Interpretation.} MMLU correlation is the strongest ($\rho{=}{+}0.66$); partial $\rho{=}{+}0.48$ after $\log_{10}P$ control ($p{=}0.094$, marginal). $n{=}13$ has limited power; a population-scale observation, not a per-model predictor.

\paragraph{Saturation.} Peak RSA-C saturates at $0.36$--$0.37$ across all $12$ modern generative models ($3$B+), irrespective of architecture, instruction tuning, or downstream accuracy ($5$--$95\%$ GSM8K). Only GPT-2 XL is meaningfully lower ($0.32$). Cross-form concept discrimination is an emergent property by $\sim 3$B; what varies is the depth (plateau width) of maintenance.

\FloatBarrier
\Needspace{8\baselineskip}
\section*{5. Reasoning-position analyses}
\section{Reasoning-Tuned Models: Full Three-Question Deep Dive}\label{app:reasoning_full}
\paragraph{Question and scope.} How does extraction position change the recovered geometry? Original and regenerated CoT runs differ in layer selection and truncation; keep their cohorts and budgets separate.

\subsection{Format-collapse control: does the CoT-tail really cross formats?}\label{app:cot_format}


\paragraph{The objection.} \S\ref{sec:reasoning} reports that the CoT-tail
subspace preserves cross-format concept identity. Reasoning-tuned models are
known to drift into English whatever surface form the prompt arrived in, so the
score could be measuring within-format retrieval under a cross-format label: if
every trace is English prose, grouping stimuli by the \emph{prompt's} form
labels groups that no longer differ in the text the model actually produced.

\paragraph{The measurement.} We regenerate the chain of thought for all $324$
TriForm stimuli under the same greedy decoding and the same $512$-token budget,
this time keeping the generated text alongside the CoT-tail hidden states.
Each trace is assigned its own dominant surface form from its text (CJK
fraction; French diacritics and function words; Python keywords and comparison
operators; mathematical operators; enumerated or bulleted lines; English
function words otherwise). We then recompute the cross-format score $X$
\emph{grouping by the generated trace's form instead of the prompt's} --- scoring
each stimulus against concept centroids built only from traces written in other
surface forms (Table~\ref{tab:cot_format}). The CoT-tail FARS is re-extracted from these activations exactly
the same way the CoT-tail retrieval in \S\ref{sec:reasoning} does, so the control tests the estimator the paper uses.

\begin{table}[h]
\centering
\small
\setlength{\tabcolsep}{5pt}
\begin{tabular}{@{}lcccccc@{}}
\toprule
 & trunc.\ & collapse & trace & \multicolumn{2}{c}{cross-format $X$ (\%)} & Haar \\
\cmidrule(lr){5-6}
Model & \% & \% & forms & by prompt & by trace & $X$ \% \\
\midrule
R1-Distill-Llama-8B & 79 & 56 & 6 & 30.6 & \textbf{26.9} & 8.0 \\
R1-Distill-Qwen-7B & 86 & 57 & 6 & 44.4 & \textbf{28.1} & 5.6 \\
\bottomrule
\end{tabular}
\caption{Format-collapse control on the CoT-tail subspace (2 reasoning-tuned models,
$324$ stimuli each). \emph{Collapse \%} is the fraction of non-English prompts whose
generated trace is dominantly English. \emph{By prompt} is the cross-format score as
\S\ref{sec:reasoning} computes it; \emph{by trace} regroups the same stimuli by the form
the model actually wrote in. The confound is real --- more than half of non-English
prompts do produce English traces, and regrouping costs
12--37\% of the score ---
but it does not manufacture the effect: 6 distinct surface forms survive in the
generated text, and the regrouped score stays 3.3--5.1$\times$ a matched-rank Haar
subspace measured the same way.}
\label{tab:cot_format}
\end{table}

\paragraph{Reading.} Collapse is substantial but partial: 56--57\% of
non-English prompts yield an English trace, yet Chinese and French prompts
largely keep their language in the trace, and code, mathematical and enumerated
traces persist as minorities (R1-Distill-Llama-8B: English~205, Chinese~54, French~39, code~12, list~7, math~7; R1-Distill-Qwen-7B: English~208, Chinese~48, French~38, code~11, list~10, math~9).
Because several surface forms genuinely survive, the regrouped score is
well defined, and it falls from 30.6--44.4\% to 26.9--28.1\%
against a Haar control at 5.6--8.0\%. The honest statement is therefore the
weaker one we now make in \S\ref{sec:reasoning}: the CoT-tail subspace carries concept
identity across surface forms that the model itself produced, at a level the
prompt-grouped number overstates. \emph{Scope:} 2 models; the traces are also
79--86\% budget-truncated, so this is a claim about the
reasoning stream rather than a completed chain.

\subsection{The three questions}

The three questions below (Q1 base$\to$distill transfer, Q2 last-token vs.\ CoT-tail, Q3 X-FARS absorption) are analysed on the two distilled pairs, for which base counterparts exist. The native reasoner QwQ-32B replicates the departure without a base pair ($d_G{=}4.33$ between last-token and CoT-tail FARS, CoT-tail retrieval $59.0\%$ vs.\ $22.8\%$ Haar; \S\ref{sec:reasoning}, Fig.~\ref{fig:cot_departure}).

\begin{figure}[h]
\centering
\includegraphics[width=\linewidth]{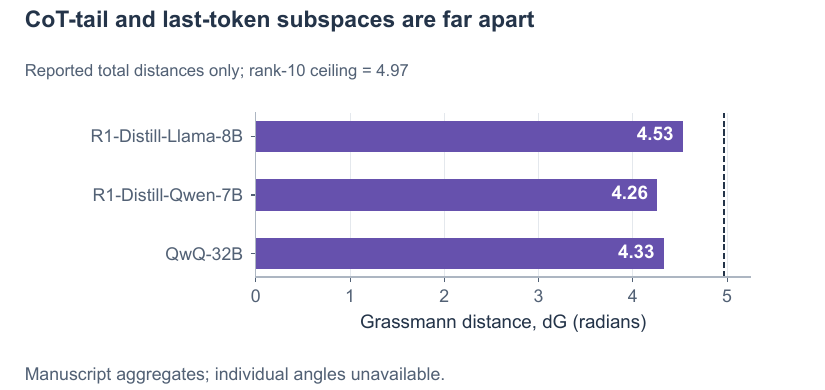}
\caption{\textbf{Reported CoT-tail distances from last-input-token FARS.} The bars show the manuscript's aggregate Grassmann distances, $d_G=\sqrt{\sum_i\theta_i^2}$, for two R1-Distill models and native QwQ-32B. The dashed line marks the rank-$10$ ceiling $\sqrt{10}\pi/2=4.97$. Individual principal angles are not plotted because an aggregate $d_G$ does not uniquely determine them.}
\label{fig:cot_departure}
\end{figure}

\paragraph{Models, positions, and extraction.}
The paired comparisons use DeepSeek-R1-Distill-Llama-8B and
DeepSeek-R1-Distill-Qwen-7B with their respective base counterparts.
For each of the $324$ TriForm stimuli, FARS is extracted separately from
the last input token and the last generated token under a greedy $512$-token
budget. In this reasoning-position analysis, the layer maximizes top-$k$
concept-centroid variance. This selection rule differs from the mid-band
concept-RSA rule used for the readout diagnostic; the two analyses should
not be treated as sharing an identical selected basis.
For orthonormal row bases in a common hidden coordinate space,
$d_G(B_1,B_2)=\sqrt{\sum_i\arccos^2\sigma_i(B_1B_2^\top)}$,
with rank-$10$ ceiling $\sqrt{10}\pi/2\approx4.97$.

\subsection{Q1: Base-to-distilled last-input-token transfer}
Table~\ref{tab:cot_base_transfer} reports nonzero subspace changes in both
same-family pairs. These comparisons describe geometric similarity in the
paired hidden coordinate spaces; they do not by themselves establish
preservation of a causal reasoning mechanism.
\begin{table}[htbp]
\centering\small
\begin{tabular}{@{}lrrr@{}}
\toprule
Pair & Base layer & Distilled layer & $d_G$\\
\midrule
Llama-3.1-8B / R1-Distill-Llama-8B & 22 & 20 & 3.06\\
Qwen2.5-7B / R1-Distill-Qwen-7B & 7 & 13 & 3.48\\
\bottomrule
\end{tabular}
\caption{Last-input-token FARS in base and distilled models. Layer indices
are those reported for the reasoning-position extraction runs.}
\label{tab:cot_base_transfer}
\end{table}

\subsection{Q2: Last-input-token versus CoT-tail geometry}
Within-model aggregate distances are large in both distilled models
(Table~\ref{tab:cot_position_shift}). The selected CoT-tail layers are earlier
than the selected last-input-token layers. An aggregate distance does not
identify each principal angle, and we make no direction-by-direction claim
from these summaries. The native QwQ-32B result ($d_G=4.33$) provides an
additional observation beyond the two distilled pairs.
\begin{table}[htbp]
\centering\small
\begin{tabular}{@{}lrrr@{}}
\toprule
Model & Last-input layer & CoT-tail layer & $d_G$\\
\midrule
R1-Distill-Llama-8B & 20 & 14 & 4.53\\
R1-Distill-Qwen-7B & 13 & 10 & 4.26\\
\bottomrule
\end{tabular}
\caption{Within-model distance between separately extracted subspaces.
Layer selection and token position both change in these comparisons.}
\label{tab:cot_position_shift}
\end{table}

The CoT-tail bases also support retrieval of concept identity
(Table~\ref{tab:cot_retrieval_original}). The original retrieval summary and
the regenerated trace-form control in App.~\ref{app:cot_format} are separate
runs with different selected CoT layers and reported retrieval statistics.
Their values must not be pooled as repeated measurements of one fixed basis.
\begin{table}[htbp]
\centering\small
\begin{tabular}{@{}lrrr@{}}
\toprule
Model & Last-input Agn. (\%) & CoT-tail Agn. (\%) & Haar (\%)\\
\midrule
R1-Distill-Llama-8B & 79.6 & 60.8 & 21.6\\
R1-Distill-Qwen-7B & 75.0 & 58.0 & 24.4\\
\bottomrule
\end{tabular}
\caption{Original reasoning-position retrieval summaries on $324$ stimuli.
The regenerated control separately measures cross-format $X$ after grouping
by the form of the generated trace, rather than the original prompt.}
\label{tab:cot_retrieval_original}
\end{table}

\subsection{Q3: Incorporating distilled models into X-FARS}
We compare the original fifteen-model pool with pools that additionally
include one or both distilled models (Table~\ref{tab:cot_pool_absorption}).
Adding Llama-Distill raises absolute retrieval by $1.74$ percentage points;
adding Qwen-Distill as well lowers it relative to that intermediate pool.
Every pool remains above its corresponding Haar mean. This tests alignment
on the fixed concept inventory, rather than generalization to unseen concepts.
\begin{table}[htbp]
\centering\small
\setlength{\tabcolsep}{4pt}
\begin{tabular}{@{}lrrrr@{}}
\toprule
Pool & X-FARS (\%) & Haar mean $\pm$ SD & Haar p95 & Gap (pp)\\
\midrule
15 models & 53.83 & $25.82\pm1.01$ & 27.58 & 28.01\\
16 (+Llama-Distill) & 55.57 & $25.94\pm1.16$ & 28.01 & 29.63\\
17 (+Qwen-Distill) & 51.83 & $24.64\pm1.12$ & 26.22 & 27.19\\
\bottomrule
\end{tabular}
\caption{Cross-architecture concept retrieval and matched-optimizer Haar
controls. The gap is explicitly X-FARS minus the Haar mean, calculated from
the displayed rounded values; it is not the gain from adding a model.}
\label{tab:cot_pool_absorption}
\end{table}

\subsection{How often does the generation budget bind?}\label{app:cot_truncation}
The original generation-length summary is given in
Table~\ref{tab:cot_truncation}. Qwen-14B is included in this length audit,
but not in the two-pair geometry comparison above. No corresponding
QwQ-32B length summary is available here; the reported $64$--$88\%$ range
must not be interpreted as a measured QwQ truncation rate.
\begin{table}[htbp]
\centering\small
\begin{tabular}{@{}lrrrr@{}}
\toprule
Model & Stimuli & Mean length & Median length & Budget hit (\%)\\
\midrule
R1-Distill-Llama-8B & 324 & 489 & 512 & 84.9\\
R1-Distill-Qwen-14B & 324 & 389 & 512 & 63.9\\
R1-Distill-Qwen-7B & 324 & 491 & 512 & 88.0\\
\bottomrule
\end{tabular}
\caption{Generated-token lengths under the original $512$-token budget.
These are the original runs; the regenerated trace-form controls have
budget-hit rates of $78.7\%$ and $86.4\%$ for Llama-8B and Qwen-7B.}
\label{tab:cot_truncation}
\end{table}

\paragraph{Scope.}
For most audited stimuli, the sampled tail is a budget-limited position,
not a completed reasoning chain. Token position, selected layer, and output
format can all contribute to the observed geometric change. The results
motivate longer-budget and position-matched interventions; they do not
isolate reasoning computation or establish that either subspace is sufficient
for producing the correct answer.


\FloatBarrier
\Needspace{8\baselineskip}
\section*{6. Exploratory cross-model structure}
\section{Historical Cross-Model Alignment and Retrieval}\label{app:universal_frame}
\paragraph{Question and scope.} What concept structure can be aligned across models? Reported cohorts differ between panels. Alignment on shared concept identities is not held-out-concept generalization.

\begin{figure*}[h]
\centering
\includegraphics[width=\linewidth]{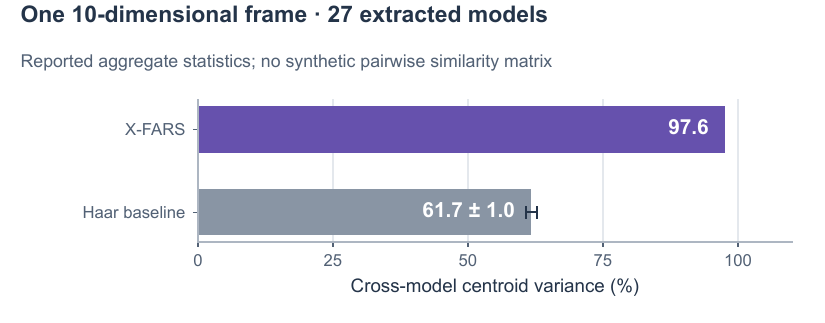}
\caption{\textbf{Reported cross-model summary statistics.} The shared ten-dimensional X-FARS frame explains $97.6\%$ of cross-model centroid variance over $27$ extracted models; the reported Haar-random projection baseline is $61.7\pm1.0\%$. This panel compares aggregate explained variance, not pairwise similarity or held-out generalization.}
\label{fig:universality_synthesis}
\end{figure*}

The extracted per-model FARS subspaces admit a further, orthogonal validation. They align pairwise (centroid-RSA $>0.90$ within family, $>0.77$ cross-family on the original fifteen-model pool; Fig.~\ref{fig:alignment}) and jointly into a single shared canonical frame (X-FARS) that, refitted on all $27$ extracted models, captures $\mathbf{97.6\%}$ of cross-model centroid variance ($36.8\sigma$ above a same-rank random-projection control at $61.7\pm 1.0\%$; the original fifteen-model fit gave $97.0\%$, $21.6\sigma$). The algebra results below were computed on the fifteen-model pool. Held-out concept controls at $K\in\{3,6,9\}$ concepts preserve RSA $0.28$--$0.40$ and probe accuracy $29$--$43\%$; the frame is neither a one-model artefact nor an over-fit to the specific $18$ concepts.

\begin{figure}[t]
\centering
\includegraphics[width=0.74\linewidth]{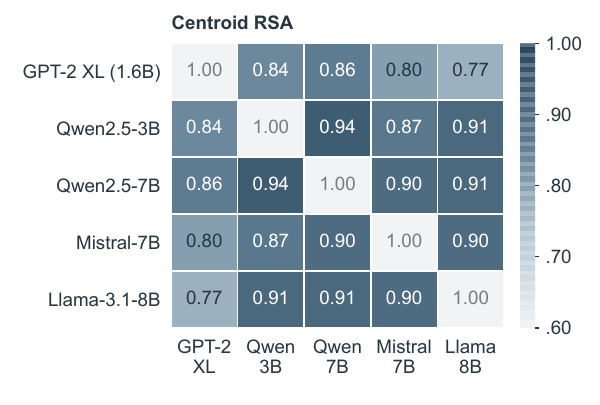}
\caption{\textbf{Cross-model alignment in a five-model subset.} Centroid RSA ranges from $0.77$ to $0.94$ off the diagonal; gray cells denote self-comparisons. Values are shown to two decimal places; this panel reports the five-model subset. The selected five models should not be confused with the larger $15$- and $27$-model analysis pools.}\label{fig:alignment}
\end{figure}

\paragraph{Algebraic operations: differences, analogies, compositions.}
Within the shared frame, concept directions support the full set of algebraic operations. (i)~\textbf{Differences} $\mathbf{c}_A {-} \mathbf{c}_B$ are preserved at mean cosine $\mathbf{0.88}$ (median $0.97$) across $\binom{18}{2} \times \binom{27}{2} = 53{,}703$ cross-model measurements ($\mathbf{200.2\sigma}$ above the shuffled-concept null; $97.7\%$ of pairs positive). (ii)~\textbf{Word2vec-style analogies} $\mathbf{c}_A {-} \mathbf{c}_B {+} \mathbf{c}_C \approx \mathbf{c}_D$ resolve to the expected concept at mean top-$1$ $\mathbf{37.4\%}$ across $27 \times 10 = 270$ attempts ($\mathbf{17.0\sigma}$ above a shuffled null at $6.6\%$; Figure~\ref{fig:word2vec}). (iii)~\textbf{Compositions} (LRH) hold at pair, triplet (six generative architectures, span $0.68$--$0.92$), and quartet (Mamba: span $0.87$, $p{=}1.5\!\times\!10^{-5}$) level. The encoder MLM DeBERTa-v3 marks a clean architectural scope: triplet LRH fails (span $0.66$, all coefficients negative). The FARS basis trained on a random $15$ of $18$ concepts further \emph{generalises to the held-out $3$} at concept-RSA $\mathbf{0.53}$ vs.\ $0.21$ for a random orthonormal projection ($\Delta {=} {+}0.33$; all $15$ models positive).

\begin{figure}[t]
\centering
\includegraphics[width=\linewidth]{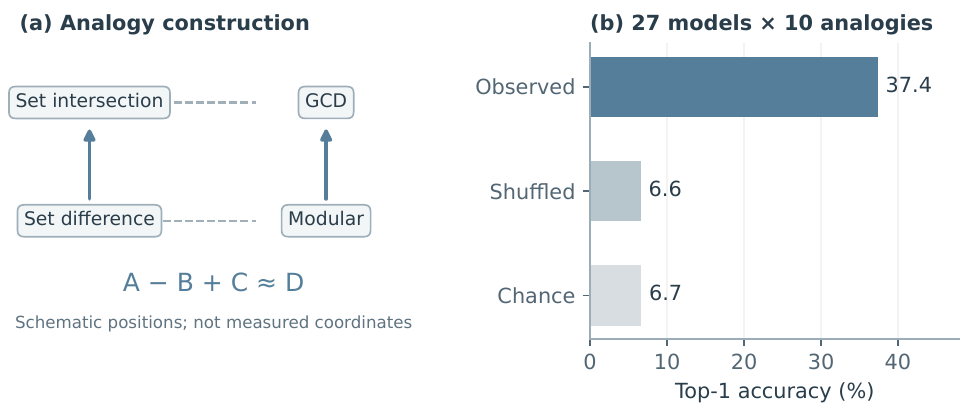}
\caption{\textbf{Concept analogies: schematic construction and measured retrieval.} Left: one curated analogy, with positions chosen only to illustrate the algebra. Right: saved top-$1$ accuracy across $27$ models and $10$ analogies ($270$ attempts): $37.4\%$ observed, $6.6\%$ shuffled, and $6.7\%$ chance. The schematic is not an empirical embedding.}\label{fig:word2vec}
\end{figure}

\paragraph{Cross-architecture concept retrieval.}
A shared canonical frame (X-FARS) enables training-free cross-architecture concept retrieval. For each ordered form pair $(A,B)$ with $A\neq B$, we use the $54$ stimuli in form $A$ as queries and $54$ in form $B$ as the corpus, retrieving by cosine similarity in (i)~the full hidden state ($1600$--$8192$d) and (ii)~the FARS projection ($10$d); chance is $5.6\%$. Across $30$ ordered form pairs $\times$ $5$ models including Llama-$70$B-Inst, FARS-projected retrieval matches or beats full on every cell, pooling to $\mathbf{+20.9}$\,pp ($50.8\%\to 71.7\%$, paired Wilcoxon $p{=}2.9{\times}10^{-25}$, Table~\ref{tab:retrieval}); Llama-$70$B alone $+20.1$\,pp ($61.5\%\to\mathbf{81.6\%}$) at $\mathbf{819\times}$ compression.

\begin{table}[t]
\centering
\caption{Cross-format top-$1$ retrieval, mean over $30$ ordered form pairs. Full = best-FARS-layer hidden state; FARS = $10$-d projection. Chance $5.6\%$.}\label{tab:retrieval}
\footnotesize
\setlength{\tabcolsep}{4pt}
\begin{tabular}{@{}lcccc@{}}
\toprule
\textbf{Model} & \textbf{Dim full} & \textbf{Full top-1} & \textbf{FARS top-1} & $\Delta$ \\
\midrule
\raisebox{-0.18em}{\includegraphics[height=0.95em]{figures/logos/openai.pdf}}~GPT-2 XL              & 1600 & .351 & \textbf{.562} & $+.211$ \\
\raisebox{-0.18em}{\includegraphics[height=0.95em]{figures/logos/qwen.pdf}}~Qwen2.5-7B            & 3584 & .495 & \textbf{.729} & $+.234$ \\
\raisebox{-0.18em}{\includegraphics[height=0.95em]{figures/logos/mistral.pdf}}~Mistral-7B-v0.3       & 4096 & .522 & \textbf{.743} & $+.220$ \\
\raisebox{-0.18em}{\includegraphics[height=0.95em]{figures/logos/meta.pdf}}~Llama-3.1-8B-Instruct & 4096 & .559 & \textbf{.738} & $+.179$ \\
\raisebox{-0.18em}{\includegraphics[height=0.95em]{figures/logos/meta.pdf}}~Llama-3.1-70B-Instruct & 8192 & .615 & \textbf{.816} & $\mathbf{+.201}$ \\
\midrule
\textbf{Pooled} ($n{=}150$ pairs) & --- & .508 & \textbf{.717} & $+.209$ \\
\multicolumn{4}{l}{Paired Wilcoxon (FARS $>$ full), $p{=}$} & $2.9{\times}10^{-25}$ \\
\bottomrule
\end{tabular}
\end{table}

FARS wins because it removes the form confound that pulls all stimuli of the same form together; the same property that makes FARS a clean patching target makes it the right cross-format embedding. The retrieval claim is scoped to the concept distribution FARS is extracted from: on MBPP-sanitized (out-of-distribution), FARS-projected retrieval drops to $1$--$6\%$ vs.\ full-activation $6.5$--$17\%$.

\section{Historical Canonical Concept Axes}\label{app:canonical_axis}
\paragraph{Question and scope.} How interpretable are the historical aligned coordinates? X-FARS uses a heuristic update that is not an exact solver of the stated rectangular objective; fit statistics are historical summaries.

\paragraph{Method.} For each model, compute $18$ concept centroids in its $10$-d FARS subspace. After centering and Frobenius-normalising, iterative Generalised Procrustes finds rotations $R_i$ minimising $\sum_i \|C_i R_i - C_*\|_F^2$ where $C_*$ is the shared canonical centroid matrix.

\paragraph{Result.} Refitted on all $27$ extracted models, the canonical rotation explains $93.2\%$ of cross-model centroid variance via FARS+GPA and $\mathbf{97.6\%}$ via X-FARS (Eq.~\ref{eq:xfars}; Table~\ref{tab:canonical_axis}); the original fifteen-model fit gave $91.9\%$ and $97.0\%$. Two nulls: shuffled-concept labels reach $63.0\%$ (fifteen-model pool); random $10$-d projections of the same activations reach $61.7\pm 1.0\%$ on the $27$-model pool ($36.8\sigma$ below X-FARS). The canonical-axis property is therefore specific to the concept-aligned subspace, not a property of any low-rank slice. RSA-C saturates at $0.36$--$0.37$ for all generative models $\geq 3$B regardless of architecture family; only GPT-2 XL ($1.6$B) and the encoder MLMs sit lower. A relative-depth sweep ($9$ percentiles) confirms the alignment is not an artefact of best-layer choice (fraction-explained $\in[0.91, 0.97]$ throughout).

\begin{table*}[t]
\centering
\caption{Cross-model FARS canonical-axis alignment. FARS$+$GPA: $93.2\%$ variance explained on all $27$ extracted models; Historical X-FARS alignment: $\mathbf{97.6\%}$, $\mathbf{36.8\sigma}$ above the random-projection control at $61.7 \pm 1.0\%$. Shuffled-concept-label null: $63.0\%$. Alignment remains high across the reported model pools; this comparison does not establish generalization to held-out concepts or rule out overfitting.}\label{tab:canonical_axis}
\footnotesize
\setlength{\tabcolsep}{4pt}
\begin{tabular}{@{}lcc@{}}
\toprule
\textbf{Quantity} & \textbf{Real} & \textbf{Null (shuffled, $n{=}100$)} \\
\midrule
FARS$+$GPA, $5$ models ($1.6$--$70$B) & $\mathbf{95.1\%}$ & $63.0\%$ ($95\%$ CI $[59.9, 66.3]$) \\
FARS$+$GPA, $9$ models, $6$ arch.\ families & $\mathbf{92.4\%}$ & $63.0\%$ \\
\textbf{FARS$+$GPA, all $15$ models} & $\mathbf{91.9\%}$ & $63.0\%$ \\
\textbf{X-FARS joint, all $15$ models} & $\mathbf{97.0\%}$ & --- \\
\textbf{FARS$+$GPA, all $27$ models} & $\mathbf{93.2\%}$ & --- \\
\textbf{X-FARS joint, all $27$ models} & $\mathbf{97.6\%}$ & --- \\
Random $10$-d projection ($27$ models, $10$ seeds) & $61.7 \pm 1.0\%$ & --- \\
$z$-score (X-FARS vs.\ random projection, $27$ models) & --- & $\mathbf{36.8\sigma}$ \\
\midrule
\multicolumn{3}{@{}p{0.97\textwidth}@{}}{\footnotesize Per-model residual to canonical (full $15$-model X-FARS): generative $.148$--$.310$; RoBERTa $.400$; DeBERTa-v3 $.295$.} \\
\bottomrule
\end{tabular}
\end{table*}

A layer-selection sweep at $9$ relative-depth percentiles gives fraction-explained in $[0.91, 0.97]$ throughout, so the alignment is not driven by best-layer choice.

\paragraph{Per-concept stability mirrors the prose-code axis.} Cross-architecture $\sigma$ separates declarative concepts (most stable: \texttt{causal\_confound}, \texttt{rel\_transitivity}, \texttt{rel\_set\_difference}; $\sigma{=}0.021$--$0.025$) from procedural ones (most variable: \texttt{arith\_gcd}, \texttt{rel\_function\_\allowbreak composition}, \texttt{arith\_multi\_step}; $\sigma{=}0.045$--$0.079$).

\paragraph{Interpretable canonical axes.} The $10$ canonical axes admit clean dichotomies at their $+/-$ extremes (Table~\ref{tab:canonical_axes_interp}). Axis $0$ ($24\%$ variance) is set-theoretic vs.\ causal-spatial. \textbf{Axis $1$ ($16\%$) is the declarative-procedural axis at the geometric level}, the same dichotomy that drives the prose-code asymmetry at the behavioural level (\S\ref{app:pydecl}). Axis $4$ ($8\%$) is a dedicated logical-inference axis. Top-$5$ axes carry $76\%$ of canonical variance; $+/-$ extreme concepts and loading signs are preserved when Llama-$70$B is added to the alignment.

\begin{figure*}[t]
\centering
\includegraphics[width=\linewidth]{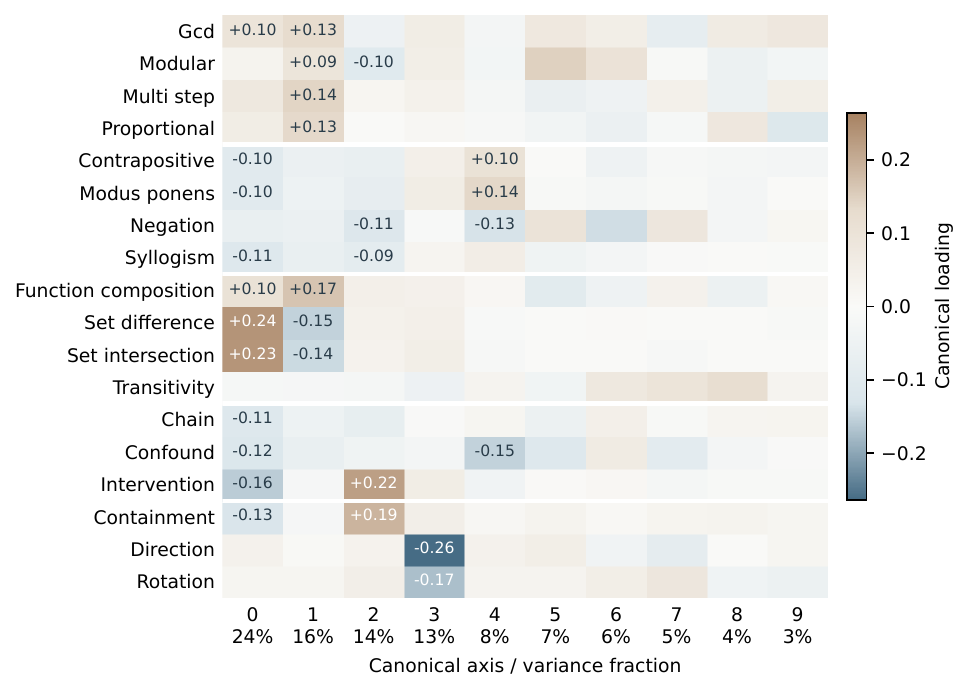}
\caption{Canonical FARS axes loadings: rows are $18$ concepts grouped by domain; columns are the $10$ canonical axes ordered by variance. Axis numbers and variance fractions are shown below the columns; interpretation labels appear in Table~\ref{tab:canonical_axes_interp}. Axis $1$ is the procedural-vs-set-relational axis that mirrors the prose-code behavioural asymmetry.}\label{fig:canonical_heatmap}
\end{figure*}

\begin{figure*}[t]
\centering
\includegraphics[width=\linewidth]{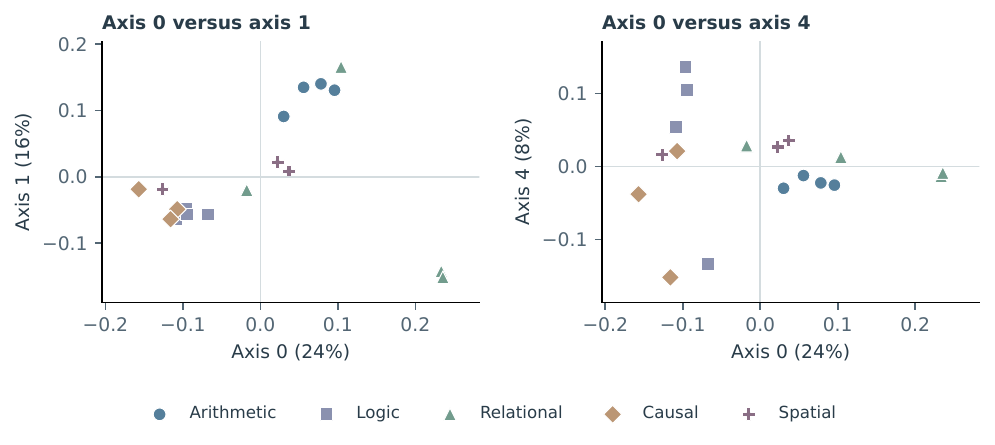}
\caption{Biplot of the $18$ concepts in the (axis $0$, axis $1$) (a) and (axis $0$, axis $4$) (b) planes of the canonical FARS frame, colored by domain. Colors and marker shapes identify domains; exact concept loadings are in the accompanying heatmap. Axis signs and interpretations refer to this historical fitted frame.}\label{fig:concept_biplot}
\end{figure*}

\begin{table}[tbp]
\centering
\caption{Interpretation of the top-$5$ canonical FARS axes on the $5$-model GPA alignment (including Llama-3.1-70B-Instruct). The variance fractions are reported on the $5$-model canonical frame; the $+/-$ extreme concepts and the sign of every loading are preserved between the $4$-model ($1.6$B--$8$B) and $5$-model ($1.6$B--$70$B) alignments. Each axis admits a recognisable concept dichotomy; cross-model stability is $\sigma{<}0.02$ for axes $0$--$4$. Axis $1$ corresponds to the same procedural-vs-declarative distinction that drives our prose-code asymmetry at the behavioural level (\S\ref{app:pydecl}).}\label{tab:canonical_axes_interp}
\footnotesize
\setlength{\tabcolsep}{4pt}
\begin{tabular}{@{}clp{2.7cm}p{2.7cm}@{}}
\toprule
\textbf{Ax.} & \textbf{Var} & \textbf{$+$ extreme} & \textbf{$-$ extreme} \\
\midrule
0 & $.242$ & set\_diff, set\_$\cap$, func\_comp & causal\_intervention, spatial\_containment \\
1 & $.157$ & func\_comp, multi\_step, proportional & set\_diff, set\_$\cap$, causal\_confound \\
2 & $.153$ & causal\_intervention, spatial\_containment & logic\_negation, arith\_modular \\
3 & $.131$ & causal\_intervention, arith\_modular & spatial\_direction, spatial\_rotation \\
4 & $.078$ & modus\_ponens, contrapositive, syllogism & causal\_confound, logic\_negation \\
\bottomrule
\end{tabular}
\end{table}

\paragraph{X-FARS: jointly-aligned in one step.} We replace the two-step FARS+GPA pipeline with a single joint objective:
\begin{equation}
\min_{B_1,\ldots,B_M,\, \bar{C}} \;\; \sum_m \bigl\|\,\tilde{C}_m B_m^\top - \bar{C}\,\bigr\|_F^2 \quad \text{s.t.} \quad B_m B_m^\top = I_k, \;\; \|\bar{C}\|_F = 1,
\label{eq:xfars}
\end{equation}
where $\tilde{C}_m$ is model $m$'s centered centroid matrix at its best FARS layer. The objective is taken subject to $\|\bar{C}\|_F = 1$, which rules out the degenerate optimum in which every $B_m$ projects onto the null space of $\tilde{C}_m$ and $\bar{C} \to 0$. This objective specifies a joint alignment criterion. For rectangular projections, its quadratic term depends on $B_m$, so a cross-covariance SVD is not in general an exact minimizer without additional restrictions. The following are empirical alignment summaries, not a claim of a globally optimal solution to Eq.~\ref{eq:xfars}. X-FARS attains $\mathbf{97.0\%}$ variance ($+5.1$\,pp over FARS+GPA, $+34.2$\,pp over random projection, $\mathbf{21.6\sigma}$; Table~\ref{tab:xfars}) at $\Delta\text{RSA}\in[-0.008,+0.006]$ across all $15$ models. Plug-and-play leave-one-out (Table~\ref{tab:xfars_holdout}): every held-out model recovers concept-RSA within $\Delta\in[-0.004,+0.014]$ (mean $+0.002$), the plug-in basis is on average \emph{slightly better} than independent extraction. Alignment is robust to concept subsampling ($9$ concepts $\to 98.0\%$, $15$ concepts $\to 93.6\%$) and to instruction tuning (Mistral-base $\to$ Mistral-Inst Procrustes $92.4\%$, cross-basis RSA loss $1$--$3\%$).

\paragraph{Explicit-objective optimization audit.} A separate audit fixes the unique cached FARS layer in $26$ models; Llama-70B is omitted because two saved candidate layers make this selection ambiguous. We compare raw-centroid least squares, least squares after per-model Frobenius normalization of centroids, and alignment after additionally normalizing each projected centroid matrix. Every objective uses a centered, unit-Frobenius shared frame. Three starts (historical basis, cached FARS, random orthonormal basis) are optimized using exact gradients, Stiefel tangent projection, polar retraction and Armijo descent. Finite-difference gradient checks pass. This audit uses fixed cached layers, not a new layer-selection search.

\begin{table}[H]
\centering\small
\caption{\textbf{Optimization diagnostics after 700 iterations.} All nine fits reach the iteration budget without meeting the $10^{-6}$ maximum tangent-gradient criterion. Final objective values are comparable only within the same objective definition. Orthogonality residuals are below $9\times10^{-15}$. Small objective values alone do not validate concept preservation.}\label{tab:alignment_audit}
\begin{tabular}{@{}llrr@{}}\toprule
Objective & Initialization & Final loss & Max.\ gradient \\
\midrule
Raw centroids & Historical & 5.1923 & $1.3e+02$ \\
Raw centroids & Cached FARS & 798.21 & $8.1e+02$ \\
Raw centroids & Random & 69.723 & $2.8e+02$ \\
Unit-norm centroids & Historical & 1.4303 & $0.00052$ \\
Unit-norm centroids & Cached FARS & 1.4503 & $0.0018$ \\
Unit-norm centroids & Random & 1.4274 & $0.00061$ \\
Normalized projections & Historical & 0.065186 & $0.13$ \\
Normalized projections & Cached FARS & 0.089666 & $0.091$ \\
Normalized projections & Random & 0.00080568 & $0.017$ \\
\bottomrule\end{tabular}
\end{table}

The objective decreases monotonically in these runs, but none satisfies the prescribed stationarity tolerance. Consequently these fits are diagnostic, not corrected optima or new evidence for a universal concept frame. The historical alignment summaries below remain empirical results under their original fitting procedure and model pools.

\begin{figure}[H]
\centering
\includegraphics[width=\linewidth]{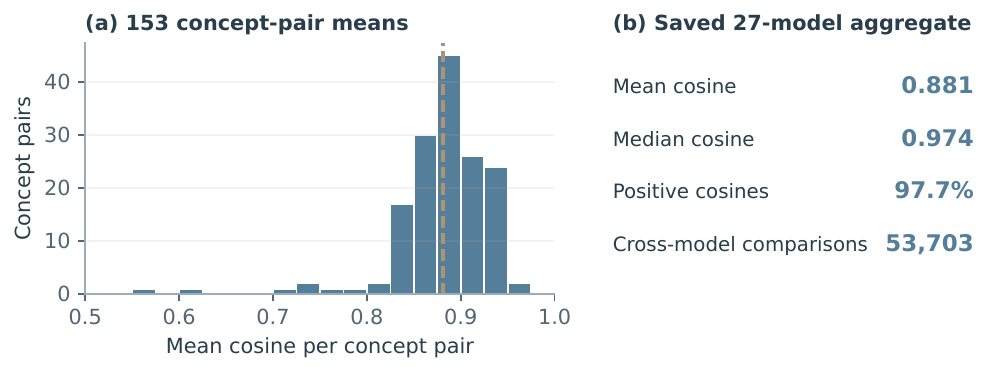}
\caption{\textbf{Concept-difference alignment in the saved 27-model analysis.} Left: distribution of $153$ concept-pair means, each averaging $351$ cross-model comparisons. This is not a histogram of the $53{,}703$ individual cosines. Right: aggregate statistics recorded in the same result file. Raw pairwise cosines and shuffled samples were not reconstructed for this plot.}\label{fig:concept_arithmetic_histogram}
\end{figure}

\paragraph{Comprehensive algebra: $\mathbf{53{,}703}$ cosines.} All $\binom{18}{2}{=}153$ concept pairs $\times$ $\binom{27}{2}{=}351$ cross-model pairs, in the frame fitted on all $27$ models: mean cosine $\mathbf{0.881}$, median $0.974$, $97.7\%$ positive, $91.2\%$ above $0.5$. Shuffled null $-0.001\pm 0.004$ over $20$ shuffles; separation $\mathbf{200.2\sigma}$. Widening the pool from the original fifteen models to all $27$ lowers the mean cosine (from $0.961$) while raising the separation, as expected when weaker and more distant architectures enter. Concept-pair averages and saved aggregate statistics are shown in Fig.~\ref{fig:concept_arithmetic_histogram}.

\paragraph{Full distance landscape preserved.} The $18\times 18$ centroid-distance matrix in each model's X-FARS projection, correlated upper-triangles cross-model: mean Pearson $\boldsymbol{r{=}0.613\pm 0.248}$ across $351$ pairs (median $0.686$; max $0.983$); $74\%$ at $r{>}0.5$, $48\%$ at $r{>}0.7$; null $-0.0001\pm 0.0065$, separation $\mathbf{94.9\sigma}$. The full concept-neighbourhood structure is preserved, not just selected directions.

\paragraph{Word2vec analogy resolution.} $10$ curated analogies $\mathbf{c}_A {-} \mathbf{c}_B {+} \mathbf{c}_C \approx \mathbf{c}_D$ evaluated on all $27$ models ($270$ attempts): mean top-$1$ $\mathbf{37.4\%}$ (chance $6.7\%$, shuffled null $6.6\%{\pm}1.8\%$, $\mathbf{17.0\sigma}$); top-$3$ $54.4\%$ ($13.7\sigma$). Within-domain analogies carry the effect, e.g.\\ $\mathbf{c}_\text{modular}{-}\mathbf{c}_\text{gcd}{+}\mathbf{c}_{\cap}\approx\mathbf{c}_{\setminus}$; arbitrary cross-domain analogies do not resolve.

\paragraph{Cross-architecture concept retrieval.} X-FARS bases project every model into the same $10$-d frame, enabling training-free cross-architecture search. Pooled across $30$ ordered cross-architecture pairs and $6$ forms, top-$1$ is $\mathbf{69.8\%}$ (chance $5.6\%$; Table~\ref{tab:xfars_xarch}). Same-family pairs (Llama-$70$B$\leftrightarrow$Llama-$8$B) hit $84$--$87\%$; cross-family similar-scale dense models ($7$--$8$B Llama/Mistral/Qwen) $80$--$87\%$; GPT-2 XL and Mixtral-as-query are weakest. Cross-architecture retrieval is essentially at parity with within-model cross-form retrieval ($70.5\%$).

\begin{figure*}[t]
\centering
\includegraphics[width=0.70\linewidth]{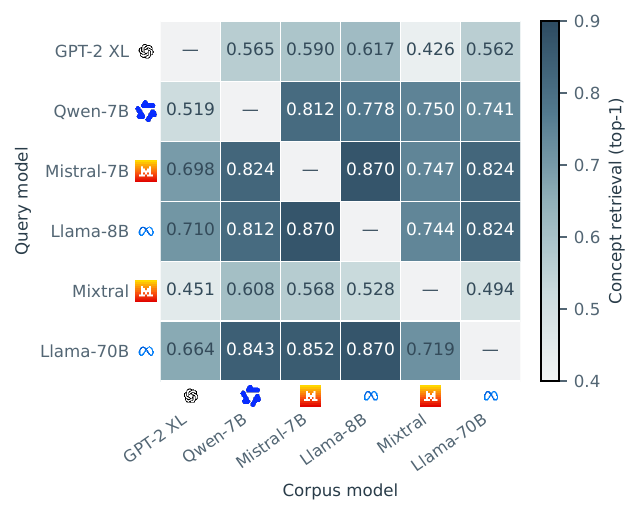}
\caption{\textbf{Historical cross-model concept retrieval in a six-model subset.} Values reproduce Table~\ref{tab:xfars_xarch} at its reported precision (off-diagonal mean $69.6\%$); rows are query models and columns are corpus models. Self-comparisons are omitted. This panel replaces a legacy figure whose fitting statistics referred to a different model pool. It reports empirical retrieval under historical alignment, not an optimizer guarantee or held-out-concept generalization.}\label{fig:xfars}
\end{figure*}

\begin{table*}[t]
\centering
\caption{Cross-architecture top-$1$ concept retrieval via X-FARS shared canonical frame (chance $5.6\%$). Row: query model; column: corpus model. Cell: mean top-$1$ accuracy across $6$ surface forms, $54$ same-form stimuli per side. No model weights are updated; alignment uses shared concept identities. This is not an unpaired evaluation.}\label{tab:xfars_xarch}
\footnotesize
\setlength{\tabcolsep}{3pt}
\begin{tabular}{@{}l|cccccc@{}}
\toprule
\textbf{Query $\backslash$ Corpus} & GPT-2 XL & Qwen-7B & Mistral-7B & Llama-8B & Mixtral-MoE & Llama-70B \\
\midrule
GPT-2 XL                   & ,      & $.565$ & $.590$ & $.617$ & $.426$ & $.562$ \\
Qwen2.5-7B                 & $.519$ & ,      & $.812$ & $.778$ & $.750$ & $.741$ \\
Mistral-7B-v0.3            & $.698$ & $.824$ & ,      & $\mathbf{.870}$ & $.747$ & $.824$ \\
Llama-3.1-8B-Inst          & $.710$ & $.812$ & $\mathbf{.870}$ & ,      & $.744$ & $.824$ \\
Mixtral-8x7B-Inst (MoE)    & $.451$ & $.608$ & $.568$ & $.528$ & ,      & $.494$ \\
Llama-3.1-70B-Inst         & $.664$ & $.843$ & $.852$ & $.870$ & $.719$ & ,      \\
\midrule
\multicolumn{7}{l}{Mean of the $30$ printed cross-model cells: $\mathbf{0.696}$. Same-family Llama-8B$\leftrightarrow$70B: $0.85$.}\\
\bottomrule
\end{tabular}
\end{table*}

\begin{table*}[t]
\centering
\caption{Leave-one-out plug-and-play protocol validated on the full $12$-model set. For each held-out model, refit X-FARS on the remaining $11$ to obtain a canonical frame $\bar{C}^{(-m)}$, then construct the held-out basis using a polar-factor SVD update against $\bar{C}^{(-m)}$. The plug-in basis recovers $\Delta\text{RSA}\in[-0.004, +0.014]$ relative to the model's standard FARS basis (mean $+0.002$). Architectures most divergent from the canonical (RoBERTa, Qwen-3B) benefit most.}\label{tab:xfars_holdout}
\footnotesize
\setlength{\tabcolsep}{4pt}
\begin{tabular}{@{}lccc@{}}
\toprule
\textbf{Held-out model} & \textbf{Resid.\ to canonical} & \textbf{Std.\ FARS RSA} & \textbf{Plug-in X-FARS RSA} \\
\midrule
\raisebox{-0.18em}{\includegraphics[height=0.95em]{figures/logos/openai.pdf}}~GPT-2 XL                & $.213$ & $.279$ & $.278$ ($\Delta{-}.001$) \\
\raisebox{-0.18em}{\includegraphics[height=0.95em]{figures/logos/qwen.pdf}}~Qwen-2.5-3B-Inst        & $.143$ & $.332$ & $.341$ ($\Delta{+}.009$) \\
\raisebox{-0.18em}{\includegraphics[height=0.95em]{figures/logos/qwen.pdf}}~Qwen-2.5-7B             & $.117$ & $.357$ & $.360$ ($\Delta{+}.003$) \\
\raisebox{-0.18em}{\includegraphics[height=0.95em]{figures/logos/microsoft.pdf}}~Phi-3.5-mini-Inst       & $.141$ & $.352$ & $.351$ ($\Delta{-}.001$) \\
\raisebox{-0.18em}{\includegraphics[height=0.95em]{figures/logos/mistral.pdf}}~Mistral-7B-v0.3         & $.118$ & $.362$ & $.363$ ($\Delta{+}.001$) \\
\raisebox{-0.18em}{\includegraphics[height=0.95em]{figures/logos/mistral.pdf}}~Mistral-7B-Inst         & $.110$ & $.364$ & $.365$ ($\Delta{+}.001$) \\
\raisebox{-0.18em}{\includegraphics[height=0.95em]{figures/logos/meta.pdf}}~Llama-3.1-8B-Inst       & $.093$ & $.364$ & $.366$ ($\Delta{+}.002$) \\
\raisebox{-0.18em}{\includegraphics[height=0.95em]{figures/logos/meta.pdf}}~Llama-3.1-70B-Inst      & $.143$ & $.371$ & $.371$ ($\Delta{+}.000$) \\
\raisebox{-0.18em}{\includegraphics[height=0.95em]{figures/logos/mistral.pdf}}~Mixtral-8x7B-Inst (MoE) & $.110$ & $.355$ & $.357$ ($\Delta{+}.001$) \\
~Mamba-2.8B (SSM)        & $.140$ & $.364$ & $.362$ ($\Delta{-}.001$) \\
\raisebox{-0.18em}{\includegraphics[height=0.95em]{figures/logos/meta.pdf}}~RoBERTa-large (encoder) & $.402$ & $.176$ & $.190$ ($\boldsymbol{\Delta}{+}\mathbf{.014}$) \\
\raisebox{-0.18em}{\includegraphics[height=0.95em]{figures/logos/microsoft.pdf}}~DeBERTa-v3-large (encoder) & $.366$ & $.336$ & $.332$ ($\Delta{-}.004$) \\
\bottomrule
\end{tabular}
\end{table*}

\begin{table*}[t]
\centering
\caption{Standard FARS (two-step: per-model PCA then post-hoc GPA) vs.\ X-FARS (historical alignment heuristic) as reported in the historical summary. Cohorts differ by metric; the variance row specifies its own nine-model pool. These entries are not a fresh optimizer validation.}\label{tab:xfars}
\footnotesize
\setlength{\tabcolsep}{4pt}
\begin{tabular}{@{}lcc@{}}
\toprule
\textbf{Metric} & \textbf{FARS + GPA} & \textbf{X-FARS} \\
\midrule
Canonical variance explained ($9$ models, $6$ arch.\ families) & $92.43\%$ & $\mathbf{97.22\%}$ \\
Cross-family pooled $\rho$ (pred.\ patching, $n{=}120$) & $-0.806$ & $-0.806$ \\
$\Delta$ per-model concept-RSA vs.\ FARS         & ,        & $[-0.002,\,-0.0006]$ \\
Steps required                                     & PCA + iterated Procrustes & iterated alignment heuristic \\
\bottomrule
\end{tabular}
\end{table*}

\paragraph{Implication.} The reported alignment summaries suggest shared structure on this fixed concept inventory. They do not establish convergence to a unique geometry, held-out-concept generalization, or exact optimization. Canonical-axis interpretations and behavioural associations should be read within their stated model pools.

\section{Composite Concepts and Constituent-Span Reconstruction}\label{app:compositionality}
\paragraph{Question and scope.} How closely do composite projections lie in constituent spans? Span reconstruction is not evidence that the model executes the corresponding symbolic operation.

\paragraph{Method.} For each compositional stimulus expressing two atomic TriForm concepts $A,B$, we project the last-token hidden state at the best FARS layer onto FARS and decompose $\text{observed} \approx \alpha \mathbf{c}_A + \beta \mathbf{c}_B$ via least-squares. The span ratio $\lVert\text{reconstruction}\rVert/\lVert\text{observed}\rVert$ measures the fraction of activation \emph{magnitude} in $\text{span}(\mathbf{c}_A,\mathbf{c}_B)$ (a ratio of norms, not of variance; the corresponding energy fraction is its square). Controls: random-$2$-concept and random-$2$-direction in $10$-d FARS.

\paragraph{Result: pair LRH on $4$ architecture families.} Span ratios significantly above both controls on every tested architecture (Table~\ref{tab:compositionality}, $p{<}10^{-9}$ on every row). The finding replicates on Mamba-$2.8$B (state-space, span $0.77$) and RoBERTa-large (encoder MLM, span $0.83$); both LSQ coefficients positive in $66$--$85\%$ of stimuli.

\begin{table*}[t]
\centering
\caption{Compositional FARS span ratios. Top ($6$ pairs, $n{=}36$): original concept pairs. Bottom ($18$ pairs, $n{=}108$): extended set. Compositionality holds across $4$ architecture families spanning $1.6$B--$70$B; $70$B has the highest span ratio and both-positive rate.}\label{tab:compositionality}
\scriptsize
\setlength{\tabcolsep}{4pt}
\begin{tabular}{@{}lcccc@{}}
\toprule
\textbf{Model} & \textbf{Span(A,B)} & \textbf{Span(rand C)} & \textbf{Span(rand dirs)} & \textbf{Both $\alpha, \beta {>}0$} \\
\midrule
\multicolumn{5}{l}{\textit{Original $6$-pair set ($n{=}36$ stimuli):}} \\
GPT-2 XL              & $\mathbf{.736}$ & $.537$ & $.453$ & $24/36$ ($67\%$) \\
Llama-3.1-8B-Inst     & $\mathbf{.715}$ & $.393$ & $.410$ & $27/36$ ($75\%$) \\
Mistral-7B-v0.3       & $\mathbf{.758}$ & $.473$ & $.406$ & $31/36$ ($86\%$) \\
\midrule
\multicolumn{5}{l}{\textit{Extended $18$-pair set ($n{=}108$ stimuli):}} \\
Llama-3.1-8B-Inst     & $\mathbf{.696}$ & $.422$ & $.424$ & $73/108$ ($68\%$) \\
Mistral-7B-v0.3       & $\mathbf{.708}$ & $.481$ & $.425$ & $73/108$ ($68\%$) \\
\textbf{Llama-3.1-70B-Inst} & $\mathbf{.750}$ & $.481$ & $.390$ & $\mathbf{92/108}$ ($\mathbf{85\%}$) \\
\midrule
\multicolumn{5}{l}{\textit{Architecture-family extension:}} \\
\textbf{Mamba-2.8B} (state-space) & $\mathbf{.766}$ & $.571$ & $.430$ & $71/108$ ($66\%$) \\
\textbf{RoBERTa-large} (encoder-only MLM) & $\mathbf{.830}$ & $.654$ & $.394$ & $78/108$ ($72\%$) \\
\midrule
\multicolumn{5}{@{}p{0.97\textwidth}@{}}{\footnotesize Paired Wilcoxon Span(A,B) $>$ controls: $p{<}10^{-9}$ all rows; $70$B: $p{=}2.8\!\times\!10^{-11}$ vs.\ rand-C, $6.7\!\times\!10^{-17}$ vs.\ rand-dir.} \\
\bottomrule
\end{tabular}
\end{table*}

\paragraph{Triplet LRH ($3$-way).} $48$ stimuli, $8$ triplet templates $\times$ $3$ instances $\times$ $2$ forms; span(A,B,C) significantly above rand-$3$-direction on all $6$ generative architectures including state-space Mamba (Table~\ref{tab:triplet_compositionality}).

\begin{table*}[t]
\centering
\caption{Triplet compositionality ($A \wedge B \wedge C$): $48$ stimuli across $8$ triplet templates, tested on $\mathbf{6}$ models spanning dense decoder-only AND state-space (Mamba) architectures. Span(A,B,C) is the least-squares span ratio against $3$ atomic centroids; ``rand-3-C'' and ``rand-3-D'' are random-3-concept and random-3-direction controls. All six tested models show span(A,B,C) significantly higher than random-direction control ($p{<}10^{-4}$ on every model); the smaller models (Phi-3.5, Qwen-3B) show the largest span ratios ($0.83$--$0.92$) and highest all-positive-3-coefs rates ($52$--$58\%$), well above the chance rate of $1/2^3 = 12.5\%$.}\label{tab:triplet_compositionality}
\scriptsize
\setlength{\tabcolsep}{3pt}
\begin{tabular}{@{}lccccc@{}}
\toprule
\textbf{Model} & \textbf{Span(A,B,C)} & \textbf{rand-3-C} & \textbf{rand-3-D} & \textbf{$p$ vs rand-D} & \textbf{All $\alpha,\beta,\gamma{>}0$} \\
\midrule
\textbf{Phi-3.5-mini-Inst (3.8B dense, $L{=}18$)} & $\mathbf{.92}$ & $.86$ & $.53$ & $\mathbf{<10^{-12}}$ & $\mathbf{25/48}$ ($\mathbf{52\%}$) \\
\textbf{Qwen-2.5-3B-Inst (dense, $L{=}22$)} & $\mathbf{.83}$ & $.76$ & $.46$ & $<10^{-12}$ & $\mathbf{26/48}$ ($\mathbf{54\%}$) \\
\textbf{Mamba-2.8B (state-space, $L{=}35$)} & $\mathbf{.83}$ & $.60$ & $.52$ & $\mathbf{2.3{\times}10^{-12}}$ & $27/48$ ($56\%$) \\
Qwen-7B (dense decoder, $L{=}13$) & $\mathbf{.81}$ & $.62$ & $.49$ & $1.1{\times}10^{-11}$ & $28/48$ ($58\%$) \\
Mistral-7B-Inst (dense, $L{=}10$) & $\mathbf{.76}$ & $.60$ & $.53$ & $1.5{\times}10^{-6}$ & $19/48$ ($40\%$) \\
Llama-3.1-8B-Inst (dense, $L{=}8$) & $\mathbf{.68}$ & $.60$ & $.51$ & $9.3{\times}10^{-5}$ & $19/48$ ($40\%$) \\
\bottomrule
\end{tabular}
\end{table*}

\paragraph{Architectural scope: encoder MLMs fail triplet.} DeBERTa-v3-large ($L{=}15$, same $48$ stimuli): span$(A,B,C){=}.66$, below rand-$3$-concept $.76$; all-three-positive in $\mathbf{0/48}$. Pair LRH still holds on both encoders, so the encoder boundary sits between $2$D and $3$D composition.

\paragraph{Quartet ($4$-way) saturates.} Only Mamba-$2.8$B clearly beats rand-$4$-concept (span $0.87$ vs.\ $0.79$, $p{=}0.023$); dense transformers tie or lose vs.\ rand-C while beating rand-D (Table~\ref{tab:quartet_compositionality}). LRH composition saturates as arity approaches FARS dimensionality.

\begin{table*}[t]
\centering
\caption{Quartet ($4$-way) compositionality on $4$ models, $24$ stimuli. Span(A,B,C,D) vs.\ rand-$4$-concept (rand-C) and rand-$4$-direction (rand-D). Only Mamba's $\Delta$ over rand-C reaches significance.}\label{tab:quartet_compositionality}
\scriptsize
\setlength{\tabcolsep}{4pt}
\begin{tabular}{@{}lccccc@{}}
\toprule
\textbf{Model} & \textbf{Span(A,B,C,D)} & \textbf{rand-C} & \textbf{rand-D} & $\Delta$\,rand-C & \textbf{All 4 ${>}0$} \\
\midrule
\textbf{Mamba-2.8B} (state-space) & $\mathbf{.87}$ & $.79$ & $.60$ & $+.08$ ($p{=}0.023$) & $\mathbf{10/24}$ ($42\%$) \\
Qwen-7B (dense)                  & $.84$ & $.84$ & $.63$ & $.00$ (n.s.) & $6/24$ ($25\%$) \\
Mistral-7B-Inst (dense)          & $.79$ & $.74$ & $.58$ & $+.05$ (n.s.) & $4/24$ ($17\%$) \\
Llama-3.1-8B-Inst (dense)        & $.68$ & $.70$ & $.60$ & $-.02$ (n.s.) & $4/24$ ($17\%$) \\
\bottomrule
\end{tabular}
\end{table*}

\begin{figure*}[t]
\centering
\includegraphics[width=\linewidth]{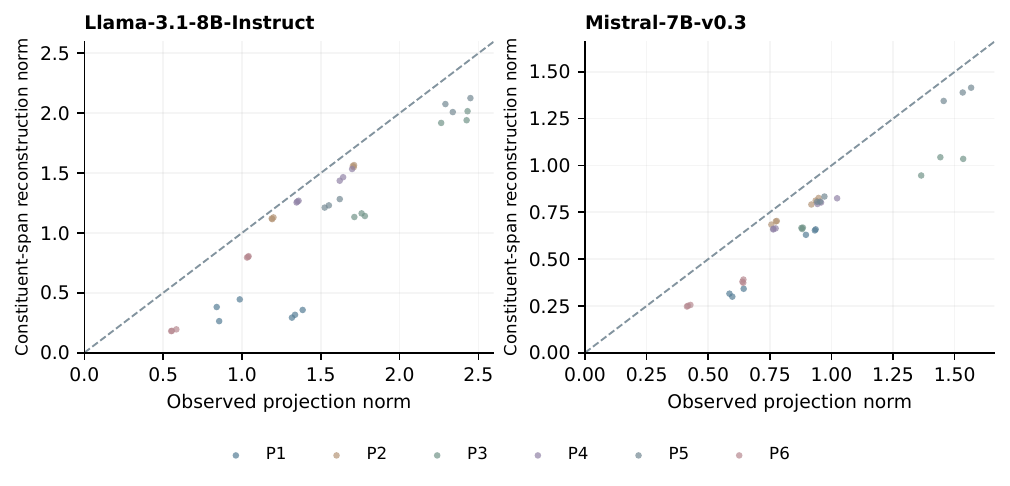}
\caption{\textbf{Constituent-span reconstruction in two cached model runs.} Each panel contains $36$ composite projections over six concept pairs. Axes show projection and reconstruction norms; the dashed line is equality. Mean norm ratios are $0.715$ for Llama-3.1-8B-Instruct (layer $8$) and $0.758$ for Mistral-7B-v0.3 (layer $11$). Both panels are rebuilt from saved activations and bases; they are a two-model subset of Table~\ref{tab:compositionality}. P1: gcd + transitivity; P2: modular arithmetic + modus ponens; P3: causal chain + multi-step arithmetic; P4: syllogism + set intersection; P5: function composition + modus ponens; P6: spatial direction + proportionality.}\label{fig:compositionality}
\end{figure*}

\paragraph{Scope.} Strict equal-weight sum $\text{observed}\approx \mathbf{c}_A + \mathbf{c}_B$ fails ($p{\approx}1$, $n{=}36$); the span-level claim is supported, coefficient mixtures depend on the instance rather than being canonical $(1,1)$. FARS concepts compose linearly but with data-dependent coefficients.

\section{Cross-Model FARS Geometry Predicts Other Models' Behavior}\label{app:cross_model_predict}
\paragraph{Question and scope.} Do geometric relations predict reported patching overlap? Random low-rank projections also recover much of this relation, so predictive correlation is not unique to FARS.

\begin{figure*}[t]
\centering
\includegraphics[width=\linewidth]{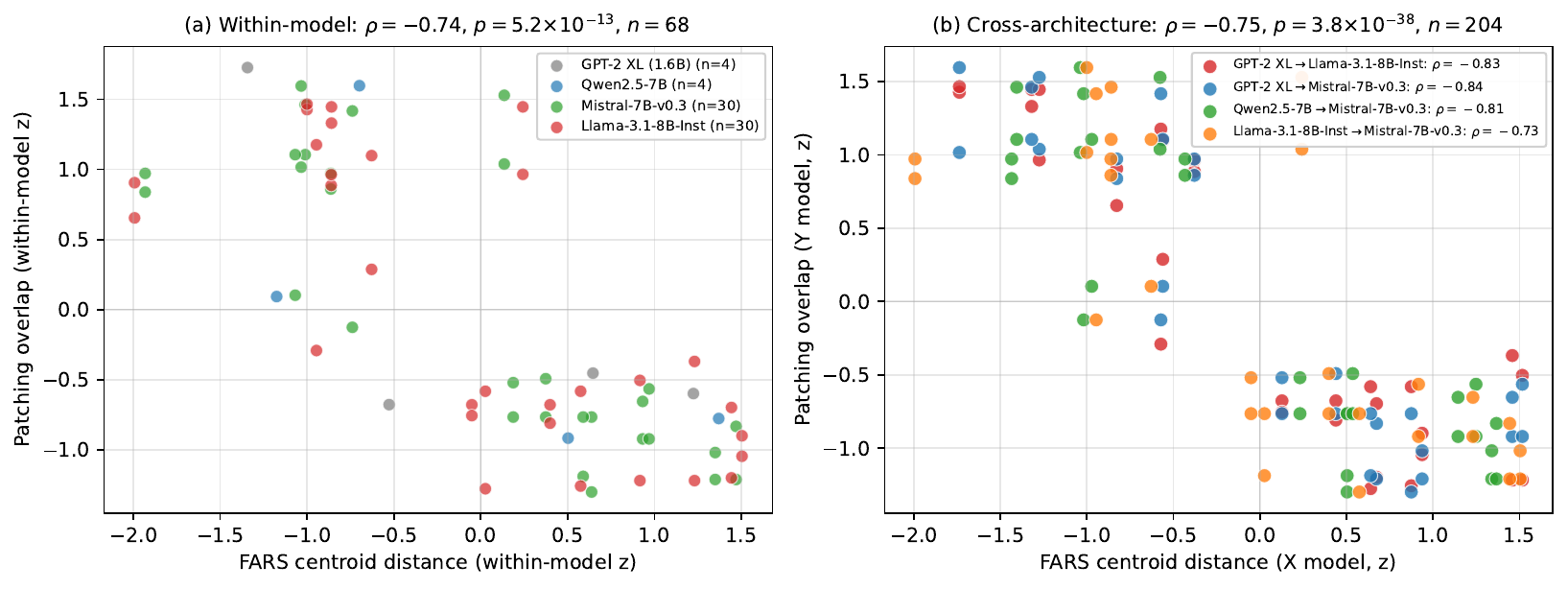}
\caption{\textbf{Historical predictiveness from centroid distances.} Within-model (left) and cross-model (right) associations with patching overlap. The reported cross-family aggregate is $\rho=-0.75$ over $204$ repeated entries. Random low-rank projections recover much of the association ($\rho=-0.71$); these correlations do not establish causal necessity.}\label{fig:hero_predict}
\end{figure*}

\paragraph{Method and reported result.} For each model pair $(X,Y)$, correlate $X$'s centroid distances across ordered form pairs with $Y$'s patching overlap. Llama-3.1-8B-Instruct and Mistral-7B have $30$ tested form pairs; GPT-2 XL and Qwen-7B have $4$. Table~\ref{tab:cross_model_predict} reports the resulting matrix. Across the $12$ off-diagonal cells, within-cell standardisation gives pooled Spearman $\rho=-0.750$ over $204$ entries. The reported association is exploratory: cells share models and form pairs, and the random-projection control below recovers much of it.

\begin{table}[H]
\centering
\caption{Spearman $\rho$ between $X$'s FARS-distance per form-pair (rows) and $Y$'s patching overlap (cols). Off-diagonal cells with $n{=}30$ are the strongest evidence: $X$ and $Y$ are different models. Pooled cross-family $\rho{=}{-}0.750$ ($p{=}3.8{\times}10^{-38}$, $n{=}204$).}\label{tab:cross_model_predict}
\footnotesize
\setlength{\tabcolsep}{2pt}
\begin{tabular}{@{}lcccc@{}}
\toprule
$X\downarrow$ / $Y\rightarrow$ & GPT-2 & Llama & Mistral & Qwen \\
\midrule
GPT-2 XL              & $-.40$ & $\mathbf{-.83}$ & $\mathbf{-.84}$ & $-.80$ \\
Llama-Inst            & $+.20$ & $-.70$ & $\mathbf{-.73}$ & $-.60$ \\
Mistral-7B            & $+.20$ & $\mathbf{-.76}$ & $-.80$ & $-.60$ \\
Qwen-7B               & $+.20$ & $\mathbf{-.78}$ & $\mathbf{-.81}$ & $-.60$ \\
\bottomrule
\multicolumn{5}{l}{\footnotesize Bold: cross-family $n{=}30$. Plain: $n{=}4$.} \\
\end{tabular}
\end{table}

\paragraph{Scope: predictiveness is not unique to FARS.} A random orthonormal $10$-d projection recovers cross-family $\rho{=}-0.714$, compared with FARS's $-0.750$; the shuffled-$Y$ null is near zero (Table~\ref{tab:cross_model_controls}). These results support shared geometric structure, but do not isolate concept-specific computation. Historical intervention results (App.~\ref{app:fars_ablation},~\ref{app:probe_ablation},~\ref{app:gsm8k}) use different endpoints and controls; they should not be read as proving unique causal necessity. Repeated form pairs across cells also limit an independent-sample interpretation of pooled significance.

\begin{table}[H]
\centering
\caption{Historical cross-family predictiveness controls (within-cell z-pooled, $n{=}204$). Random projections retain much of the correlation. Reported $p$-values are historical; pooled entries share models and form pairs.}\label{tab:cross_model_controls}
\footnotesize
\setlength{\tabcolsep}{4pt}
\begin{tabular}{@{}lcc@{}}
\toprule
\textbf{Projection} & \textbf{Cross-family $\rho$} & \textbf{$p$} \\
\midrule
FARS centroid (10-dim, supervised)            & $-0.750$ & $3.8{\times}10^{-38}$ \\
Random orthonormal (10-dim, no labels)         & $-0.714$ & $3.9{\times}10^{-33}$ \\
\midrule
Shuffled-$Y$ (1000 permutations)               & $-0.001$ \,(mean) & 95\% CI $[-.13, +.14]$ \\
\bottomrule
\end{tabular}
\end{table}

\end{document}